%% file: ICML-submission-Camera-Ready.tex
\documentclass{article}

\usepackage{microtype}
\usepackage{graphicx}

\usepackage{hyperref}

\usepackage[accepted]{icml2024}

\usepackage{booktabs}						
\usepackage{multirow}						
\usepackage{amsfonts}						
\usepackage{graphicx}						
\usepackage{duckuments}						

\usepackage[utf8]{inputenc} 
\usepackage[T1]{fontenc} 
\usepackage{url}
\usepackage{amsmath}
\usepackage{amssymb}
\usepackage{mathtools}
\usepackage{amsthm}
\usepackage{makecell}

\allowdisplaybreaks

\usepackage{booktabs}       
\usepackage{color}
\usepackage{amsfonts}       
\usepackage{nicefrac}       
\usepackage{enumerate}   
\usepackage{bm}
\usepackage{microtype}      
\usepackage{xcolor}         
\usepackage{bbm}
\usepackage{xfrac}
\usepackage{tocbibind}
\usepackage[title]{appendix}
\usepackage{tikz}
\usepackage{physics}
\usepackage{hyperref}       
\usepackage{dsfont}
\hypersetup{
    pdfnewwindow=true,     
    colorlinks=true,     
    linkcolor=cyan,     
    citecolor=cyan,     
    filecolor=magenta,     
    urlcolor=cyan     
}
\usepackage{url} 

\usepackage{caption}
\usepackage{subcaption}

\usepackage{listofitems} 
\usepackage{subfiles}
\usepackage[
  separate-uncertainty = true,
  multi-part-units = repeat
]{siunitx}

 \makeatletter
\newtheorem*{rep@theorem}{\rep@title}
\newcommand{\newreptheorem}[2]{%
\newenvironment{rep#1}[1]{%
 \def\rep@title{#2 \ref{##1}}%
 \begin{rep@theorem}}%
 {\end{rep@theorem}}}
\makeatother

\makeatletter
\newcommand{\maketitlepage}{%
    \let\thanks\@gobble
    \let\footnote\@gobble
    \@maketitle
}
\makeatother
\theoremstyle{plain}
\newtheorem{theorem}{Theorem}[section]
\newtheorem{proposition}[theorem]{Proposition}
\newtheorem{lemma}[theorem]{Lemma}
\newtheorem{corollary}[theorem]{Corollary}
\theoremstyle{definition}
\newtheorem{definition}[theorem]{Definition}
\newtheorem{assumption}[theorem]{Assumption}
\theoremstyle{remark}
\newtheorem{remark}[theorem]{Remark}

\newreptheorem{theorem}{Theorem}
\newreptheorem{lemma}{Lemma}
\newreptheorem{proposition}{Proposition}
\newreptheorem{corollary}{Corollary}
\newreptheorem{assumption}{Assumption}

\newcommand{\ols}[1]{\,\widehat{\!{#1}}} 

\newcommand{\mrd}{\mathrm{d}}
\newcommand{\mbX}{\mathbf{X}}
\newcommand{\mbx}{\mathbf{x}}
\newcommand{\mbF}{\mathbf{F}}
\newcommand{\mbf}{\mathbf{f}}
\newcommand{\bbE}{\mathbb{E}}
\newcommand{\bbR}{\mathbb{R}}
\newcommand{\mbA}{\mathbf{A}}
\newcommand{\mrm}{\mathrm{m}}

\newcommand{\KLr}{\mathrm{KL}}
\newcommand{\mbR}{\mathbb{R}}
\newcommand{\mbZn}{\mathbf{Z}_n}
\newcommand{\mun}{\mu_n}
\newcommand{\mur}{\mu_{n,(1)}}
\newcommand{\genb}{\overline{\mathrm{gen}}}
\newcommand{\gen}{\mathrm{gen}}

\newcommand{\rademacher}{ \hat{\mathfrak{R}}}

\usepackage{xspace}
\usepackage[colorinlistoftodos, color=blue!30!white 
]{todonotes}                                        
\usepackage{parskip}

\begin{document}

%

%

\newcommand{\gr}[1]{{\color{magenta} #1}}

\twocolumn[
\icmltitle{Generalization Error of Graph Neural Networks in the Mean-field Regime}

\icmlsetsymbol{equal}{*}

\begin{icmlauthorlist}
\icmlauthor{Gholamali Aminian}{equal,AlanTuring}
\icmlauthor{Yixuan He}{equal,OxStats}
\icmlauthor{Gesine Reinert}{AlanTuring,OxStats}
\icmlauthor{{\L}ukasz Szpruch}{AlanTuring,Edin}
\icmlauthor{Samuel N. Cohen}{AlanTuring,OxMaths}
\end{icmlauthorlist}

\icmlaffiliation{AlanTuring}{The Alan Turing Institute, London, United Kingdom}
\icmlaffiliation{OxStats}{Department of Statistics, University of Oxford, Oxford, United Kingdom}
\icmlaffiliation{OxMaths}{Mathematical Institute, University of Oxford, Oxford, United Kingdom}
\icmlaffiliation{Edin}{School of Mathematics, University of Edinburgh}

\icmlcorrespondingauthor{Gholamali Aminian}{gaminian@turing.ac.uk}

\icmlkeywords{Mean-field, generalization error, Graph Neural Network, Rademacher complexity}

\vskip 0.3in
]


\printAffiliationsAndNotice{\icmlEqualContribution} 

\begin{abstract}
This work provides a theoretical framework for assessing the generalization error 
of graph neural networks in the over-parameterized regime, where the number of parameters surpasses the quantity of data points. We explore two widely utilized types of graph neural networks: graph convolutional neural networks and message passing graph neural networks. Prior to this study, existing bounds on the generalization error in the over-parametrized regime were uninformative, limiting our understanding of over-parameterized network performance.  Our novel approach involves deriving upper bounds within the mean-field regime for evaluating the generalization error of these graph neural networks. We establish  
upper bounds with a convergence rate of $O(1/n)$, where $n$ is the number of graph samples. 
These upper bounds offer a theoretical assurance of the networks' performance on unseen data in the challenging over-parameterized regime and overall contribute to our understanding of their performance.
\end{abstract}

\section{Introduction}
Graph Neural Networks (GNNs) have received increasing attention due to their exceptional ability to extract meaningful representations from data structured in the form of graphs~\citep{kipfsemi,velivckovic2017graph,gori2005new,bronstein2017geometric,battaglia2018relational}. Consequently, GNNs have achieved state-of-the-art performance across various domains, including but not limited to social networks~\citep{hamilton2017inductive,fan2019graph}, recommendation systems~\citep{ying2018graph,wang2018billion} and computer vision~\citep{monti2017geometric}. Despite the success of GNN models, explaining their empirical generalization performance remains a challenge within the domain of GNN learning theory.

A central concern in GNN learning theory is to understand the efficacy of a GNN learning algorithm when applied to data that it has not been previously exposed to. This evaluation is typically carried out by investigating the generalization error, which quantifies the disparity between the algorithm's performance, as assessed through a risk function, on the training data set and its performance on previously unseen data drawn from the same underlying distribution. This paper focuses on the generalization error for GNNs in scenarios with an overabundance of model parameters, potentially surpassing the number of available training data points, with a particular focus on tasks related to classifying graphs. In this regard, our research endeavors to shed light on the generalization performance exhibited by over-parameterized GNN models in the mean-field regime~\citep{mei/montanari/nguyen:2018} for graph classification tasks.

We draw inspiration from the recent advancements in the mean-field perspective regarding the training of neural networks, as proposed in a body of literature~\citep{mei/montanari/nguyen:2018,chizat/bach:2018,mei/misiakiewicz/montanari:2019,MFLD,tzen/raginsky:2020}. These works 
propose 
to frame the process of attaining optimal weights in one-hidden-layer neural networks 
as a sampling challenge. Within the mean-field regime, the learning algorithm endeavors to discern the optimal distribution within the parameter space, rather than solely concentrating on achieving optimal parameter values. A central question driving our research is how the mean-field perspective can provide further insights into the generalization behavior of over-parameterized GNN models. Our contributions here can be summarized as follows:
\begin{itemize}
    \item We provide upper bounds on the generalization error for graph classification tasks in GNN models, including graph convolutional networks and message passing graph neural networks in the mean-field regime, via two different approaches: functional derivatives and Rademacher complexity based on symmetrized KL divergence.
   \item Using the approach based on functional derivatives, we 
   derive an upper bound with convergence rate of $O(1/n)$, where $n$ is the number of graph samples.
    \item The effects of different readout functions and aggregation functions on the generalization error of GNN models are studied.
    \item We carry out an empirical analysis on both synthetic and real-world data sets.
\end{itemize}
\paragraph{Notations:}We adopt the following convention for random variables and their distributions in the sequel. 
A random variable is denoted by an upper-case letter (e.g., $Z$), its space of possible values is denoted with the corresponding calligraphic letter (e.g., $\mathcal{Z}$), and an arbitrary value of this variable is denoted with the lower-case letter (e.g., $z$). We denote the set of integers from 1 to $N$ by $[N]\triangleq \{1,\dots,N\}$; the set of 
measures over a space $\mathcal{X}$ with finite variance
is denoted
$\mathcal{P}(\mathcal{X})$. For a matrix $\mbX\in\bbR^{k\times q}$, we denote the $i-$th row of the matrix by $\mbX[i,:]$. The Euclidean norm of a vector
$X\in\bbR^q$ 
is $\norm{X}_2:=(\sum_{j=1}^q x_j^2)^{1/2}$. For a matrix $\mathbf{Y}\in\bbR^{k\times q}$, we let
$\norm{\mathbf{Y}}_{\infty}:=\max_{1\leq j\leq k}\sum_{i=1}^q|\mathbf{Y}[j,i]|$ and $\norm{\mathbf{Y}}_{F}:=\sqrt{\sum_{j=1}^k\sum_{i=1}^q\mathbf{Y}^2[j,i]}$.  We write $\delta_z$ for a Dirac measure supported at $z$.
The KL-divergence between two
 probability distributions on $\mathbb{R}^d$ 
 with densities $p(x)$ and $q(x),$ 
 so that $q(x) > 0$ when $p(x) > 0$,  
is $\KLr(p\|q) := \int_{\mathbb{R}^d} p(x)\log(p(x)/q(x))\mrd x$ (with $0/0:=0$); the symmetrized KL divergence is $\KLr_{\mathrm{sym}}(p\|q):=\KLr(p\|q)+\KLr(q\|p)$. A comprehensive notation table is provided in the Appendix (App.)~\ref{App: preliminaries}.
\section{Our Model}\label{Sec:Problem-formulation}
\subsection{Preliminaries}\label{subsec:prelim}
We first introduce the functional linear derivative, see~\cite{cardaliaguet2019master}. 
\begin{definition}{\citep{cardaliaguet2019master}}
\label{def:flatDerivative}
A functional $U:\mathcal P(\mathbb R^n) \to \mathbb R$ 
admits a linear derivative
if there is a map $\frac{\delta U}{\delta m} : \mathcal P(\mathbb R^n) \times \mathbb R^n \to \mathbb R$ which is continuous on $\mathcal P(\mathbb R^n)$, such that $|\frac{\delta U}{\delta m}(m,a)|\leq \mathrm C(1+|a|^2)$ for a constant $C\in\mathbb{R}^{+}$ and, for all
$m, m' \in\mathcal P(\mathbb R^n)$, it holds that
\begin{align*}
&U(m') - U(m) =\!\int_0^1 \int_{\mathbb{R}^n} \frac{\delta U}{\delta m}(m_\lambda,a) \, (m'
-m)(da)\,\mrd \lambda,
\end{align*}
where $m_\lambda=m + \lambda(m' - m)$.\vspace{-0.5em}
\end{definition}   
\paragraph{Graph data samples and learning algorithm:} 
\looseness=-1 We consider graph classification for undirected graphs with $N$ nodes which have no self-loops or multiple edges.
Inputs to GNNs are graph samples, 
which are comprised of their node features and 
graph adjacency matrices. We denote the space of all adjacency matrices and node feature matrices for a graph classification task with maximum number of nodes $N_{\max}$, maximum node degree  
$d_{\max}$ and minimum node degree
$d_{\min}$ by $\mathcal{A}$ and $\mathcal{F}$, respectively. The input pair of a graph sample with $N$ nodes is denoted by $\mathbf{X}=(\mathbf{F},\mbA)$, where $\mathbf{F}\in\mathcal{F}\subset\mbR^{N\times k}$ denotes a node feature matrix with feature dimension $k$ per node and $\mbA\in\mathcal{A}\subset \{0,1\}^{N\times N}$ denotes the graph adjacency matrix. The GNN output (label) is denoted by $y\in \mathcal{Y}$ where $\mathcal{Y}=\{-1,1\}$ for binary classification.  Define $\mathcal{Z}=\mathcal{X}\times \mathcal{Y}$, where $\mathcal{X}:=\mathcal{F}\times \mathcal{A}$. Let $\mbZn=\{Z_i\}_{i=1}^n\in\mathcal{Z}$ denote the training set, where the $i-$th graph sample is $Z_i=(\mbX_i=(\mbF_i , \mbA_i),Y_i)$. We assume that $\mbZn$ are i.i.d.\,random vectors such that  $Z_i\sim\mu\in \mathcal{P}(\mathcal{Z})$. 
Its empirical measure,  $\mun:=\frac{1}{n}\sum_{i=1}^n \delta_{Z_i}$, 
is a random element with values in $\mathcal{P}(\mathcal{Z})$.
We also assume that a sample $\ols{Z}_1=(\ols{\mathbf{X}}_1,\ols{Y}_1)$ is available, and this sample is i.i.d. with respect to $\mbZn$. We set 
$\mur:=\mun+\frac{1}{n}(\delta_{\ols{Z}_1}-\delta_{Z_1}).$ 
We are interested in learning a parameterized model (or function), $\mathfrak{f}_w:\mathcal{X}\mapsto \mathcal{Y}$ for some parameters $w\in\mathcal{W}$, where $\mathcal{W}$ is 
the parameter space of the model. Inspired by \cite{aminian2023mean}, we define a learning algorithm as a map $m: {\mathcal{P}}(\mathcal{Z}) \rightarrow \mathcal{P}(\mathcal{W});  \mun \mapsto m(\mun)$, which outputs a probability distribution (measure) on $\mathcal{W}$. For example, when learning using stochastic gradient descent (SGD), the input is random samples, whereas the output is a probability measure on the space of parameters, which are used in SGD.\footnote{Due to a probability measure for the random initialization, the final output is also a probability distribution over space of parameters.}

\paragraph{Graph filters:} Graph filters are linear functions of the adjacency matrix, $\mbA$, defined by $G:\mathcal{A}\mapsto \bbR^{N \times N}$ where $N$ is the size of the input graph, see~\cite{defferrard2016convolutional}. Graph filters model the aggregation of node features in a graph neural network. For example, the symmetric normalized graph filter proposed by \cite{kipfsemi} is  $G(\mbA)=\tilde{L}:=\tilde{D}^{-1/2}\Tilde{\mbA}\tilde{D}^{-1/2}$ where $\Tilde{\mbA}=I+\mbA$, $\tilde{D}$ is the degree-diagonal matrix of $\Tilde{\mbA}$, and $I$ is the identity matrix. Another normalized filter, a.k.a. random walk graph filter~\cite{xupowerful2019}, is $G(\mbA)=D^{-1}\mbA+I$, where $D$ is the degree-diagonal matrix of $\mbA$. The mean-aggregator is also a well-known aggregator defined as $G(\mbA)=\tilde{D}^{-1}(\mbA+I)$. The sum-aggregator graph filter is defined by $G(\mbA)=(\mbA+I)$. Let us define $R_{\max}(G(\mbA))$ as the maximum rank of graph filter $G(\mbA)$ over all adjacency matrices in the data set. We let \begin{equation}\label{eq: gmax}G_{\max}=\min(\norm{G(\mathcal{A})}_{\infty}^{\max},\norm{G(\mathcal{A})}_{F}^{\max}),\end{equation} where $\norm{G(\mathcal{A})}_{\infty}^{\max}= \max_{\mbA_q\in \mathcal{A},\mu(f_q,\mbA_q)>0}\norm{G(\mbA_q)}_{\infty}$ and $\norm{G(\mathcal{A})}_{F}^{\max}= \max_{\mbA_q\in \mathcal{A},\mu(f_q,\mbA_q)>0}\norm{G(\mbA_q)}_{F}$.

\subsection{Problem Formulation}
In this study, we investigate two prominent GNN architectures: one-hidden-layer \textbf{Graph Convolutional Networks} (GCN) by~\citet{kipfsemi}, and one-hidden-layer \textbf{Message Passing Graph Neural Networks} (MPGNN) by~\citet{dai2016discriminative} and \citet{gilmer2017neural}. The GCN model is constructed by summating multiple neurons, while MPGNN relies on the summation of multiple Message Passing and Updating (MPU) units. 

\textbf{Neuron unit:}
Here we define the one-hidden-layer GCN~\citep{zhang2020fast} 
consisting of $h$ 
neuron (hidden) units. The $i$--th neuron unit is defined by parameters $W_c[i,:]:=(W_{1,c}(i),W_{2,c}(i))\in\mathcal{W}_c\subset\mathcal{S}^{k+1}$, where $\mathcal{W}_c$ is the parameter space of one neuron unit, $W_{1,c}\in\mathcal{S}^k$ are parameters connecting the aggregated node features to the neuron unit, $W_{2,c}\in \mathcal{S}$ is the parameter connecting the output of neuron unit to output and $\mathcal{S}\subset \mbR$ is a bounded set. We define $w_{2,c}:=\sup_{W_{2,c}\in \mathcal{S}}W_{2,c}$ and $\norm{W_{1,c,m}}_2=\sup_{W_{1,c}\in\mathcal{S}^k}\norm{W_{1,c}}_2$. An aggregation layer 
aggregates the features of neighboring nodes of each node via a graph filter, $G(\cdot)$. In a GCN, the output of the $i-$th neuron unit~\citep{kipfsemi} for the $j-$th node in a graph sample $\mbX_q=\{\mbF_q,\mbA_q\}$ is 
\begin{equation*}
    \begin{split}
        &\phi_c(W_c[i,:],G(\mbA_q)[j,:]\mbF_q):=\\& \quad W_{2,c}(i)\varphi\big(G(\mbA_q)[j,:]\mbF_q \cdot W_{1,c}(i)\big)\,,
\end{split}
\end{equation*} 
where $\varphi:\bbR\mapsto\bbR$ is the activation function.  
The empirical measure over the parameters of the neurons is then 
$m_h^c(\mun) := \frac{1}{h}\sum_{i=1}^h\delta_{(W_{1,c}(i),W_{2,c}(i))}\,,$ where $\{(W_{1,c}(i),W_{2,c}(i))\}_{i=1}^h$ depend on $\mun$.

\textbf{MPU unit:}
Several structures of MPGNNs have been proposed by \cite{li2022generalization,liao2020pac} and \cite{garg2020generalization} for analyzing the generalization error, with inspiration drawn from previous works such as \cite{dai2016discriminative,gilmer2017neural}. In this paper, we utilize the MPGNN model introduced in \cite{li2022generalization} and \cite{liao2020pac}. The parameters for the $i-$th Message Passing and Updating unit are denoted as $W_p[i,:]:=\big(W_{1,p}(i),W_{2,p}(i),W_{3,p}(i)\big)\in\mathcal{W}_p\subset \mathcal{S}^{2k+1}$, where $ W_{2,p}\in \mathcal{S}, \quad W_{1,p},W_{3,p}\in\mathcal{S}^k,
$ and $\mathcal{S}\subset \mbR$ is a bounded set. We define $w_{2,p}:=\sup_{W_{2,p}\in \mathcal{S}}W_{2,p}$,  $\norm{W_{1,p,m}}_2=\sup_{ W_{1,p}\in\mathcal{S}^k}\norm{W_{1,p}}_2$ and $\norm{W_{3,p,m}}_2=\sup_{ \norm{W_{3,p}}_2\in\mathcal{S}^k}\norm{W_{3,p}}_2$.
The output of the $i-$th MPU unit for the $j-$th node in a graph sample $\mbX_q=(\mbF_q,\mbA_q)$ with graph filter $G(\cdot)$ is 
 \begin{equation*}
     \begin{split}
         &\phi_p(W_p[i,:],G(\mbA_q)[j,:]\mbF_q):= W_{2,p}(i)\kappa\Big(\mbF_q[j,:]\cdot W_{3,p}(i) \\ &\quad +\rho(G(\mbA_q)[j,:]\zeta(\mbF_q))\cdot W_{1,p}(i)\Big),
     \end{split}
 \end{equation*}
where the non-linear functions $\zeta:\bbR^{N\times k}\mapsto\bbR^{N\times k}$, $\rho:\bbR^{ k}\mapsto\bbR^{ k}$, and $\kappa:\bbR\mapsto\bbR$ may be chosen from non-linear options such as Tanh or Sigmoid. For an MPGNN the parameter space is $\mathcal{W}_{p}=\bbR^{2k+1}$ and the corresponding empirical measure is $m_{h}^p(\mun):=\frac{1}{h}\sum_{i=1}^{h}\delta_{(W_{1,p}(i),W_{2,p}(i),W_{3,p}(i))},$ where $\{(W_{1,p}(i),W_{2,p}(i),W_{3,p}(i))\}_{i=1}^h$ depend on $\mun$.

 \textbf{One-hidden-layer generic GNN model:}  
 Inspired by the neuron unit in GCNs and the MPU unit in MPGNNs, we introduce a generic model for GNNs, which can be applied to GCNs and MPGNNs.
For each unit in a generic GNN with parameter $W$, which belongs to the parameter space $\mathcal{W}$, the empirical measure 
$m_h(\mun)=\frac{1}{h}\sum_{i=1}^h \delta_{W_i}$ is a measure on 
the parameter space of a unit. 
We denote the unit function for the $j$-th node by 
\[(W,G(\mbA)[j,:]\mbF)\mapsto\phi(W,G(\mbA)[j,:]\mbF).\] The output of the network for the $j-$th node of a graph sample, $\mbX_q=(\mbF_q,\mbA_q)$, can then be represented as
\begin{align}
\label{Eq: General structure}
&\frac{1}{h}\sum_{i=1}^h\phi\left(W_i,G(\mbA_q)[j,:]\mbF_q\right)\\\nonumber\quad &=\int \phi\left(w,G(\mbA_q)[j,:]\mbF_q\right)m_h(\mun)(\mrd w)\\\nonumber&=\bbE_{W\sim m_h(\mun)}[\phi\left(W,G(\mbA_q)[j,:]\mbF_q\right)]\,,
\end{align}
where $W_i$ is the parameter of the $i-$th unit. 
The final step is the pooling of the node features across all nodes for each graph sample as the output of generic GNNs. For this purpose, we introduce the readout function (a.k.a. pooling layer) $\Psi:\mathcal{P}(\mathcal{W})\times \mathcal{X}\mapsto\mbR$ 
as follows:
\begin{equation}\label{eq: Phi def}
\begin{split}
    &\Psi(m_h(\mun),\mbX_q):=\\
    &\quad\psi\Big(\sum_{j=1}^N\bbE_{W\sim m_h(\mun)}[\phi\left(W,G(\mbA_q)[j,:]\mbF_q\right)]\Big),
\end{split}
\end{equation} 
where $\psi:\mathbb{R}\mapsto\mathbb{R}$. In this work, we
consider the mean-readout and sum-readout functions by taking $\psi(x)=\frac{x}{N}$ and $\psi(x)=x$, respectively.
For a GCN and an MPGNN, the outputs of the model after aggregation are, respectively:
\begin{equation}\label{eq:GCN label}
    \begin{split}
&\Psi_c\left(m_h^c(\mun),\mbX_q\right):=\\&\quad\psi\Big(\sum_{j=1}^N\bbE_{W_c\sim m_h^c(\mun)}\big[\phi_c(W_c,G(\mbA_q)[j,:]\mbF_q)\big]\Big)\,,
    \end{split}
\end{equation}
\vspace{-1em}
\begin{equation}\label{eq:MPGNN label}
    \begin{split}
       &\Psi_p\left(m_h^p(\mun),\mbX_q\right):= \\&\quad\psi\Big(\sum_{j=1}^N\bbE_{W_p\sim m_h^p(\mun)}\Big[\phi_p(W_p,G(\mbA_q)[j,:]\mbF_q)\Big]\Big)\,. 
    \end{split}
\end{equation}

\paragraph{Loss function:}
With $\mathcal{Y}$ the label space, the  loss function $\ell: \mathbb{R} \times \mathcal{Y} \to \bbR$ is denoted as $\ell(\widehat{y},y)$, where $\widehat{y}$ is defined in \eqref{eq:GCN label} and \eqref{eq:MPGNN label} for a GCN and for an MPGNN, respectively; the loss function is assumed to be convex. For binary classification, we
take loss functions of the 
form $\ell(\widehat{y},y)=\mathfrak{h}(y\widehat{y})$, where $\mathfrak{h}(\cdot)$ represents a margin-based loss function~\citep{bartlett2006convexity}; examples include the logistic loss function $\mathfrak{h}(y\widehat{y})=\log(1+\exp(-y\widehat{y}))$, the exponential loss function $\mathfrak{h}(y\widehat{y})=\exp(-y\widehat{y})$,  and the square loss function $(1-y\hat{y})^2$.
\citet{liao2020pac} and \citet{garg2020generalization} studied a $\gamma$-margin loss inspired by \citet{neyshabur2018pac}.
\newline \textbf{Over-parameterized one-hidden-layer generic GNN:} As discussed in \cite{MFLD,mei/montanari/nguyen:2018}, based on stochastic gradient descent dynamics in one-hidden-layer neural networks for a large number of hidden neurons and a small step size, the random empirical measure 
can be well approximated by a probability measure. The same behavior (due to the law of large numbers) also holds for one-hidden-layer GNNs. Therefore, for an over-parameterized one-hidden-layer generic GNN, as the number of hidden units $h$ (width of the hidden layer) increases, under some assumptions the distribution $m_h(\mun)$ converges to a continuous distribution over the parameter space of the unit. 
\newline\textbf{True and empirical risks:}
The true risk function (expected loss) based on the loss function $\ell$, with a data measure $\mu$  
and a parameter measure $m\in\mathcal{P}(\mathcal{W})$, is
\begin{equation} \label{eq risk linear}
\begin{split}
     &\mathrm{R}(m,\mu):=\quad \int_{\mathcal{X}\times \mathcal{Y}} \ell\Big(  \Psi(m
     ,\mbx),y\Big)\mu(\mrd \mbx, \mrd y)\,.
     \end{split}
\end{equation}
When the parameter measure is $m(\mun)$, which depends on an empirical measure, we obtain for example $ \mathrm{R}(m^c(\mun),\mu)$ and $ \mathrm{R}(m^p(\mun),\mu)$, as the true risks for a GCN and an MPGNN, respectively, given the observations which are encoded in the empirical measure $\mun$.

The empirical risk is given by
\begin{equation}
    \begin{split}
        &\mathrm{R}(m,\mun) :=\quad\frac{1}{n}\sum_{i=1}^n \ell\Big(  \Psi(m,\mbx_i) ,y_i\Big).
    \end{split}
\end{equation} 
Note that $\mathrm{R}(m^c(\mun),\mun)$ and $\mathrm{R}(m^p(\mun),\mun)$ are empirical risks for a GCN and an MPGNN, respectively.
\newline\textbf{Generalization Error:} 
We would like to study the performance of the model trained with the empirical data set $\mbZn=\{(\mbX_i,Y_i)\}_{i=1}^n$ and evaluated against the true measure $\mu$, using the expected generalization error 
\begin{align}
   \genb(m(\mun),\mu):= \bbE_{\mbZn} \Big[ \gen(m(\mun),\mu) \Big],
\end{align}
where $\gen(m(\mun),\mu) =\mathrm{R}(m(\mun),\mu)-\mathrm{R}(m(\mun),\mun)$.
The generalization errors for a GCN and an MPGNN are, respectively,
$\genb(m^c(\mun),\mu)$ and $\genb(m^p(\mun),\mu)$.

\section{Related Works}

\textbf{Graph Classification and Generalization Error:} Different learning theory methods, e.g., VC-dimension, Rademacher complexity, and PAC-Bayesian, have been used to understand the generalization error of graph classification tasks. In particular, \cite{scarselli2018vapnik} studied the generalization error via VC-dimension analysis. \cite{garg2020generalization} provided some data-dependent generalization error bounds for MPGNNs for binary graph classification via VC-dimension analysis. The PAC-Bayesian approach has also been applied to two GNN models, including GCNs and MPGNNs, but only for hidden layers of bounded width, by \cite{liao2020pac} and \cite{ju2023generalization}. Considering a large random graph model, \cite{maskey2022generalization} proposed a continuous MPGNN and provided a generalization error upper bound for graph classification.
An extension of the neural tangent kernel for GNN as a graph neural tangent kernel was proposed by \citet{du2019graph}, where a high-probability upper bound on the true risk of their approach is derived. Our work differs from the above approaches as we study the generalization error of graph classification tasks under an over-parameterized regime for GCNs and MPGNNs where the width of the hidden layer is infinite. A detailed comparison is provided in Sec.~\ref{subsec:comp}.
 \newline \textbf{Generalization and over-parameterization:} In the under-parameterized regime, that is, when the number of model parameters is significantly less than the number of training data points, the theory of the generalization error has been well-developed 
\citep{vapnik/chervonenkis:2015, bartlett2002rademacher}. However, in the over-parameterized regime, this theory fails. Indeed, deep neural network models can achieve near-zero training loss and still perform well on out-of-sample data \citep{belkin/hsu/ma/mandal:2019,spigler/geigereta:2019,bartlett/montanari:2021}. 
There are three primary strategies for studying and modeling the over-parameterized regime: the neural tangent kernel (NTK)~\citep{jacot2018neural}, random feature~\citep{mei2022generalization}, and mean-field~\citep{mei/montanari/nguyen:2018}. The NTK approach, also known as lazy training, for one-hidden-layer neural networks, utilizes the fact that a one-hidden-layer neural network can be expressed as a linear model under certain assumptions. The random feature model is similar to the NTK approach but assumes constant weights in the single hidden layer of the neural network. The mean-field approach utilizes the exchangeability of neurons to work with the distribution of a single neuron’s parameters. Recently, \citet{nishikawatwo}, \citet{nitanda2022particle} and  \citet{aminian2023mean} studied the generalization error of the one-hidden layer neural network in the mean-field regime, via Rademacher complexity analysis and differential calculus over the measure space, respectively. Nevertheless, the analysis of over-parameterization within GNNs remains to be thoroughly unexplored. Recently, inspired by NTK, \cite{du2019graph} proposed a graph neural tangent kernel (GNTK) method, which is equivalent to some over-parameterized GNN models with some modifications. However, GNTK is different from GCNs and MPGNNs. For a deeper understanding of the over-parameterized regime for such GNN models, i.e., GCNs and MPGNNs, we need to delve into the generalization error analysis.
For this endeavor, the mean-field methodology for GNN models  
is promising for an analytical approach, especially when compared to NTK and random feature models~\citep{fang2021mathematical}. Hence, we use this methodology.
\vspace{-1.5em}
\section{The Generalization Error of 
KL-Regularized Empirical Risk Minimization}
\label{sec_main_results}
In this section, we 
derive
an upper bound on the solution of the KL-regularized risk minimization problem.
The KL-regularised empirical risk minimizer for parameter measure $m$ on $\mathcal{W}$ and empirical measure $\mu_n$ is defined as
\begin{equation}\label{Eq: regularized risk}
    \mathcal{V}^{\alpha}(m,\mun)= \mathrm{R}(m,\mun)+ \frac{1}{\alpha} \KLr(m\|\pi),
\end{equation}
where $\pi(w)$, $w \in \mathcal{W}$, is a prior on the parameter measure with (assumed finite) KL divergence between parameter measure $m$ and prior measure, i.e., $\KLr(m\|\pi)<\infty$, and $\alpha$ is a parameter (the inverse temperature). 
It has been shown in \cite{MFLD} that under a convexity assumption on the risk, a minimizer, denoted by $\mrm^\alpha(\mun)$, of $\mathcal{V}^{\alpha}(\cdot,\mun),$ exists, is unique, and satisfies the equation,
\begin{align}\label{Eq: Gibbs Measure}
    &\mrm^\alpha(\mun) =\frac{\pi}{S_{\alpha,\pi}(\mun)}\exp\Bigg\{-\alpha\frac{\delta \mathrm{R}(\mrm^\alpha, \mun,w)}{\delta m}  \Bigg\},
\end{align}
where $S_{\alpha,\pi}(\mun)$ is the normalizing constant to ensure that $\mrm^\alpha(\mun)$ integrates to $1$. 

For an example of a training algorithm that yields this parameter measure, consider one-hidden-layer (graph) neural networks in the context of the mean-field limit $\lim_{h \rightarrow \infty} m_h(\mun)$. We denote the Gibbs measure related to an over-parameterized one-hidden-layer GCN and an over-parameterized one-hidden-layer MPGNN by $\mrm^{\alpha,c}(\mun)\in \mathcal{P}(\mathcal{W}_c)$ and $\mrm^{\alpha,p}(\mun)\in \mathcal{P}(\mathcal{W}_p)$, respectively. We also consider $W_{c}=(W_{1,c},W_{2,c})$ and $W_{p}=(W_{1,p},W_{2,p},W_{3,p})$. 

The conditional expectation in \eqref{eq:GCN label} and \eqref{eq:MPGNN label} taken over parameters distributed according to $\mrm^\alpha(\mun)$, given the empirical measure $\mun$, 
corresponds to taking an average over samples drawn from the learned parameter measure. 
This connection has been exploited in the
mean-field models of one-hidden-layer neural networks \citep{MFLD,mei/misiakiewicz/montanari:2019,tzen/raginsky:2020}. 

The following assumptions are needed for our main results.
 \begin{assumption}[Loss function]\label{Ass: on loss NN}
      The gradient of the loss function $(\hat y,y)\mapsto \ell(\hat y,y)$ with respect to $\hat y$ is continuous and uniformly bounded for all $\hat y,y\in\mathcal{Y}$, i.e., there is a constant $M_{\ell'} $ such that $|\partial_{\hat y}\ell(\hat y, y)|\leq M_{\ell'}.$\footnote{For an $L$-Lipschitz-continuous loss function,
      $M_{\ell'}=L$.} Furthermore, we assume that the loss function is convex with respect to $\hat y$.
 \end{assumption}
 \begin{assumption}[Bounded loss function]\label{Ass: on loss NN bounded}
      There is a constant $M_\ell >0$ such that the loss function, $(\hat y,y)\mapsto \ell(\hat y,y)$ satisfies 
      $ 0\leq\ell(\hat y, y)\leq M_{\ell}$.
      \end{assumption}
\begin{assumption}[Unit function]\label{ass: bounded Unit function}
    The unit function  $(w,G(\mbA)[j,:]\mbf)\mapsto \phi(w,G(\mbA)[j,:]\mbf)$ with graph filter  $G(\cdot)$  is uniformly bounded, i.e, there is a constant $M_{\phi} $ such that $\sup_{w\in\mathcal{W},\mbf\in\mathcal{F},\mbA\in\mathcal{A}}|\phi(w,G(\mbA)[j,:]\mbf)|\leq M_{\phi}$ for all $j\in[N]$ and $\mbf \in \mathcal{F}$.
\end{assumption}
\begin{assumption}[Readout function] \label{ass: readout function}
There is a constant $L_\psi$ such that $|\psi(x_1)-\psi(x_2)|\leq L_{\psi}|x_1-x_2|$ for all $x_1,x_2\in \mbR$,  and 
$\psi(0)=0$. 
\end{assumption}
For mean-readout and sum-readout functions, we have $L_\psi=1/N$ and $L_\psi=1$, respectively.
\begin{assumption}[Bounded node features]\label{Ass: Feature node bounded}
For every graph, the node features are contained in an $\ell_2$-ball of radius $B_f$. In particular, $\norm{\mbF[i;]}_2\leq B_f$ for all $i\in[N]$.
\end{assumption}

\textbf{Assumptions Discussion:}  The requirements of Lipschitz continuity and convexity (Assumption~\ref{Ass: on loss NN}) for the loss function are met by both logistic and square losses, provided the inputs and unit functions are bounded. The condition of boundedness for unit functions (Assumption~\ref{ass: bounded Unit function}) holds under the premise of bounded inputs and a bounded parameter space for either neuron units or MPU units. Notably, with a bounded activation function, bounding parameters only in the last layer suffices. The sum and mean readout functions satisfy the conditions of Lipschitz continuity and are zero-centered (Assumption~\ref{ass: readout function}). The bounded feature (Assumption~\ref{Ass: Feature node bounded}) is a widely accepted assumption, which can be achieved through input normalization. Note that, similar assumptions have been imposed in the generalization error analysis of graph neural network literature, as discussed in \citep{liao2020pac}, \citep{ju2023generalization}, \citep{garg2020generalization} and \citep{maskey2022generalization} to facilitate the theoretical analysis. While these assumptions can be relaxed, we adhere to the current forms for clarity and simplicity. 

To establish an upper bound on the generalization error of the generic GNN, we employ two approaches: \textbf{(a)} we compute an upper bound for the expected generalization error using functional derivatives in conjunction with symmetrized KL divergence, and \textbf{(b)} we establish a high-probability upper bound using Rademacher Complexity along with symmetrized KL divergence.
\subsection{Generalization Error via Functional Derivative}
We first derive two propositions to obtain intermediate bounds on the
generalization error in terms of KL divergence. The proofs of this section are provided in App.~\ref{app: proofs of main section}. The notation is provided in Sec.~\ref{subsec:prelim}.
\begin{proposition}\label{Prop: KL bound}
  Let Assumptions~\ref{Ass: on loss NN}, \ref{ass: bounded Unit function}, and \ref{ass: readout function} hold.
  Let $\mrm(\mun)\in\mathcal{P}(\mathcal{W})$. Then,
  for the generalization error of a generic GNN model,
    \begin{equation}
    \begin{split}
        &\genb(\mrm(\mun),\mu) \leq \big(M_{\ell'}L_{\psi}N_{\max}M_\phi/\sqrt{2} \big)\\&\qquad\times\bbE_{\mbZn,\ols{Z}_1}\Big[ \sqrt{ \KLr\big(\mrm(\mun)\|\mrm(\mur)\big)}\Big],
    \end{split} \label{ineq:prop1}
    \end{equation}
    where $N_{\max}$ is the maximum number of nodes among all graph samples.
    \end{proposition}
    \begin{remark}
        For the mean-readout function, $L_\psi = \frac{1}{N}$ and the upper bound in Proposition~\ref{Prop: KL bound} can be represented as,
        \begin{equation}
    \begin{split}
        &\genb(\mrm(\mun),\mu) \leq \big(M_{\ell'}M_\phi/\sqrt{2} \big)\\&\qquad\times\bbE_{\mbZn,\ols{Z}_1}\Big[ \sqrt{ \KLr\big(\mrm(\mun)\|\mrm(\mur)\big)}\Big].
    \end{split}
    \end{equation}
    \end{remark}
Proposition~\ref{Prop: KL bound} holds for all $\mrm(\mun)\in\mathcal{P}_2(\mathcal{W})$. 
Using the functional derivative, the following proposition provides a lower bound on the generalization error of the Gibbs measure from~\eqref{Eq: Gibbs Measure}.
\begin{proposition}\label{Prop: lower bound}
     Let Assumptions~\ref{Ass: on loss NN} hold. Then, for the Gibbs measure $m^\alpha(\mun)$  
     in \eqref{Eq: Gibbs Measure},
    \[\begin{split}
    &\genb(\mrm^{\alpha}(\mun),\mu)\\&\geq \frac{n}{2\alpha} \bbE_{\mathbf{Z}_n,\ols{Z}_1}\Big[ \KLr_{\mathrm{sym}}\big(\mrm^{\alpha}(\mun)\|\mrm^{\alpha}(\mur)\big)\Big].
    \end{split}
\]
\end{proposition}

Using Proposition~\ref{Prop: KL bound} for the Gibbs measure  
and Proposition~\ref{Prop: lower bound}, we can derive the following upper bound on the generalization error for a generic GNN.
\begin{theorem}[Generalization error and generic GNNs]\label{thm: main unit upper}
    Let Assumptions~\ref{Ass: on loss NN}, \ref{ass: bounded Unit function} and \ref{ass: readout function} hold. Then, for the generalization error of the Gibbs measure
    $\mrm^\alpha(\mun)$ in \eqref{Eq: Gibbs Measure},
\begin{equation*}
 \genb(\mrm^\alpha(\mun),\mu) 
     \leq  
  \frac{\alpha C}{n},
\end{equation*}
where $C=(M_{\ell'}M_{\phi}L_{\psi}N_{\max})^2$ does not depend on $n$.
\end{theorem}

\begin{remark}[Readout-function comparison]
   In 
   Theorem \ref{thm: main unit upper}, 
   the mean-readout function with $L_{\psi}=1/N$ exhibits a tighter upper bound when compared to the sum-readout function with $L_\psi=1$. 
   For the mean-readout function, in Theorem~\ref{thm: main unit upper} we have $C=(M_{\ell'}M_{\phi})^2$.
\end{remark}

\begin{remark}[Comparison with \cite{aminian2023mean}]
In \citep[Theorem 3.4]{aminian2023mean}, 
an exact representation of the generalization error in terms of the functional derivative of the parameter measure with respect to the data measure, i.e., $\frac{\delta \mrm}{\delta \mu}(\mu_n,z)(\mathrm{d}w),$ is provided. Then, for the Gibbs measure, in \citep[Lemma D.4]{aminian2023mean}, an upper bound on the generalization error requires  
to compute $\frac{\delta \mrm}{\delta \mu}(\mu_n,z)(\mathrm{d}w)$, the functional derivative of parameter measure for the data measure. Instead, we use the convexity of the loss function concerning the parameter measure, Proposition~\ref{Prop: lower bound}, and the general upper bound on the generalization error, Proposition~\ref{Prop: KL bound}, to establish an upper bound on the generalization error of the Gibbs measure, Theorem~\ref{thm: main unit upper}, via symmetrized KL divergence properties. Not only is this simpler, but, more importantly, 
our framework enables us to derive non-trivial upper bounds on the generalization error of the graph neural network (GNN) based on specific graph properties, such as $d_{\max}$, $d_{\min}$, and $R_{\max}$, e.g., see Remarks \ref{remark: graph-filter and infinite norm} and \ref{remark: graph filter and frobenius norm}.
\end{remark}
\begin{remark}[Comparison with \citep{aminian2021exact}]
In \citep[Theorem~1]{aminian2021exact}, an exact characterization of the expected generalization error of the Gibbs algorithm in terms of the symmetrized KL information~\citep{aminian2015capacity} is derived. Then, \citet{aminian2021exact} derived an upper bound on the expected generalization error of the Gibbs algorithm. However, as discussed in \citep[Appendix~F]{aminian2023mean}, the Gibbs algorithm is different from the Gibbs measure and the generalization error analysis of the Gibbs algorithm can not be applied to the mean-field regime.
\end{remark}
\subsection{Generalization Error via Rademacher Complexity} 
Inspired by the Rademacher complexity analysis in \cite{nishikawatwo,nitanda2022particle}, we derive a high-probability upper bound on the generalization error in the mean-field regime. 
\begin{proposition}[Upper bound on the symmetrized KL divergence]\label{prop: bound on sum}
    Under Assumptions~ \ref{Ass: on loss NN}, \ref{ass: bounded Unit function}, and \ref{ass: readout function}, 
    \begin{equation*}
        \KLr_{\mathrm{sym}}(\mrm^\alpha(\mun)\|\pi)\leq 2 N_{\max} M_{\phi} M_{\ell'}L_{\psi}\alpha.
    \end{equation*}
\end{proposition}
Combining Proposition~\ref{prop: bound on sum} with \citep[Lemma~5.5]{chen2020generalized}, uniform bound and Talagrand’s contraction lemma~\cite{mohri2018foundations}, we can derive a high-probability upper bound on the generalization error via Rademacher complexity analysis. The details are provided in App.~\ref{app: rademacher}.
   \begin{theorem}[Generalization error upper bound via Rademacher complexity]\label{thm: main unit upper rademacher}
    Let Assumptions~\ref{Ass: on loss NN}, \ref{ass: bounded Unit function}, \ref{ass: readout function}, and \ref{Ass: on loss NN bounded} hold. Then, for any $\delta\in(0,1)$, with probability at least $1-\delta$, 
    under the distribution of $P_{\mbZn}$,
\begin{align*}
 \mathrm{gen}(\mrm^\alpha(\mun),\mu) 
     &\leq  
  4 N_{\max} M_{\phi}M_{\ell'}L_{\psi} \sqrt{\frac{ N_{\max} M_{\phi} M_{\ell'}L_{\psi}\alpha}{n}}\\&\quad+3M_\ell\sqrt{\frac{\log(2/\delta)}{2n}}\,.
  \end{align*}
\end{theorem}
\begin{remark}[Comparison with \cite{chen2020generalized,nishikawatwo,nitanda2022particle}] 
In \cite{chen2020generalized,nishikawatwo,nitanda2022particle}, it is assumed that there exists a ``true'' distribution $m_{\mathrm{true}}\in \mathcal{P}(\mathcal{W})$  which satisfies $\ell(\Psi(m_{\mathrm{true}},\mbX_i),y_i)=0$ for all $(\mbX_i,y_i)\in\mathcal{Z}$ where $\mu(\mbX_i,y_i)>0$ and the KL-divergence between the true distribution and the prior distribution is finite, i.e., $\KLr(m_{\mathrm{true}}\|\pi)< \infty$.  In particular,
the authors in \cite{chen2020generalized} contributed Theorem 4.5 to provide an upper bound for one-hidden layer neural networks in terms of the chi-square divergence,  $\chi^2(m_{\mathrm{true}}\|\pi)$, which is unknown. In addition, in \cite{nishikawatwo,nitanda2022particle}, the authors proposed upper bounds in terms of $\KLr(m_{\mathrm{true}}\|\pi)$ which cannot be computed. To address this issue, our main contribution in comparison with \cite{chen2020generalized,nishikawatwo,nitanda2022particle} is the utilization of Proposition~\ref{prop: bound on sum}, to obtain a parametric upper bound that can be efficiently computed numerically, overcoming the challenges posed by the unknown term. 
\end{remark}
  \begin{remark}[Comparison with Theorem~\ref{thm: main unit upper}]
   The convergence rate of the generalization error upper bound in Theorem~\ref{thm: main unit upper rademacher} is $O(1/\sqrt{n})$, while the upper bound in Theorem~\ref{thm: main unit upper} achieves a faster convergence rate of $O(1/n)$.
\end{remark}
\subsection{Over-Parameterized One-Hidden-Layer GCN} 
For the over-parameterized one-hidden-layer GCN, we first show the boundedness of the unit function as in Assumption~\ref{ass: bounded Unit function} under an additional assumption. 
\begin{assumption}[Activation functions in neuron units]\label{ass: bounded activation function} The activation function $\varphi:\bbR\mapsto\bbR$ is $L_\varphi$-Lipschitz \footnote{For the Tanh activation function, we have $L_\varphi=1$.}, so that $|\varphi(x_1)-\varphi(x_2)|\leq L_\varphi |x_1-x_2|$ for all $x_1,x_2\in \mbR$,  and zero-centered, i.e., $\varphi(0)=0$.  
\end{assumption}
\begin{lemma}[Upper Bound on the GCN Unit Output]\label{lem: GCN case}
   Let Assumptions~\ref{Ass: Feature node bounded} and \ref{ass: bounded activation function} hold. For a graph sample $(\mbA_q,\mbF_q)$ with $N$ nodes and a graph filter $G(\cdot)$, the following upper bound holds on the summation of GCN neuron units over all nodes:
     \begin{align*}
        &\sum_{j=1}^N|\phi_c(W_c,G(\mbA_q)[j,:]\mbF_q)|\\&\quad\leq N w_{2,c}L_{\varphi}\norm{W_{1,c}}_2B_f G_{\max}.
        \end{align*}
\end{lemma}
Combining Lemma~\ref{lem: GCN case} with Theorem~\ref{thm: main unit upper}, we can derive an upper bound on the generalization error of an over-parameterized one-hidden-layer GCN.
\begin{proposition}[Generalization error and GCN]\label{prop: GCN result}
In a GCN model with the mean-readout function, under the combined assumptions for Theorem~\ref{thm: main unit upper} and Lemma~\ref{lem: GCN case}, the following upper bound holds on the generalization error of the Gibbs measure  $\mrm^{\alpha,c}(\mun)$, 
\begin{equation*}
 \genb(\mrm^{\alpha,c}(\mun),\mu) 
     \leq  
  \frac{\alpha M_{c}^2 M_{\ell'}^2G_{\max}^2}{n}\,,
\end{equation*}
where $M_{c}=w_{2,c}L_{\varphi}\norm{W_{1,c}}_2 B_f$.
\end{proposition}
\begin{remark}[Graph filter and $\norm{G(\mbA)}_{\infty}$]\label{remark: graph-filter and infinite norm}
Proposition~\ref{prop: GCN result} shows that choosing the graph filter with smaller $\norm{G(\mathcal{A})}_{\infty}^{\max}$ can affect the upper bound on the generalization error of GCN. In particular, it shows the effect of the aggregation of the input features on the generalization error upper bound.
For example, if we consider the graph filter $G(\mbA)=\tilde{L}$, then we have $\norm{G(\mathcal{A})}_{\infty}^{\max}\leq \sqrt{(d_{\max}+1)/(d_{\min}+1)}$, for sum-aggregation $G(\mbA)=\mbA+I$ we have $\norm{G(\mathcal{A})}_{\infty}^{\max}\leq d_{\max}+1$,  and for random-walk, i.e., $G(\mbA)=D^{-1}\mbA+I$, we have $\norm{G(\mathcal{A})}_{\infty}^{\max}= 2$. 

\end{remark}
\begin{remark}[Graph filter and $\norm{G(\mbA)}_{F}$]\label{remark: graph filter and frobenius norm}
Similarly, choosing the graph filter with a smaller $\norm{G(\mathcal{A})}_{F}^{\max}$ value can affect the upper bound on the generalization error of GCN. For example, if we consider the graph filter $G(\mbA)=\tilde{L}$, then 
we have $\norm{G(\mathcal{A})}_{F}^{\max}\leq \sqrt{R_{\max}(\tilde{L})}$ and for random walk graph filter $G(\mbA)=\tilde{D}^{-1}\mbA+I$ we have $\norm{G(\mathcal{A})}_{F}^{\max}\leq 2\sqrt{R_{\max}(\tilde{D}^{-1}\mbA+I)}$. 
\end{remark}
In a similar approach to Proposition~\ref{prop: GCN result}, we can derive an upper bound on generalization error of GCN by combining Lemma~\ref{lem: GCN case} with Theorem~\ref{thm: main unit upper rademacher}.
\subsection{Over-Parameterized One-Hidden-Layer MPGNN}
Similar to GCN, we next
investigate the boundedness of the unit function as in Assumption~\ref{ass: bounded Unit function} for MPGNN; again we make an additional assumption.
\begin{assumption}[Non-linear functions in the MPGNN units]\label{ass: bound unit MPGNN}
    The non-linear functions $\zeta:\bbR^{N\times k}\mapsto\bbR^{N\times k}$,  $\rho:\bbR^{k}\mapsto\bbR^{k}$ and $\kappa:\bbR\mapsto\bbR$ 
    satisfy
    $\zeta(\pmb{0}^{N\times k})=\pmb{0}^{N\times k}$, $\rho(\pmb{0}^{k})=\pmb{0}^{k}$ and $\kappa(0)=0$, and are Lipschitz with parameters $L_\zeta$, $L_\rho$, and $L_\kappa$ under vector 2-norm, respectively.
\end{assumption}
Similarly to GCN, we provide the following upper bound on the MPGNN unit output. 
\begin{lemma}[Upper Bound on the MPGNN Unit Output]\label{lem: bound unit mpgnn}
Let Assumptions~~\ref{Ass: Feature node bounded} and \ref{ass: bound unit MPGNN} hold. For a graph sample $(\mbA_q,\mbF_q)$ and a graph filter $G(\cdot)$, for the MPU units over all nodes, 
\begin{equation*}
     \begin{split}
        & \sum_{j=1}^N |\phi_p(W_p,G(\mbA_q)[j,:]\mbF_q)|\\&\quad\leq w_{2,p} L_{\kappa}B_f(\norm{W_{3,p,m}}_2 +L_\rho L_{\zeta}G_{\max}\norm{W_{1,p,m}}_2).
        \end{split}
    \end{equation*}
\end{lemma}

\begin{proposition}[Generalization error and MPGNN]\label{Prop: MPGNN bound}
In an MPGNN with the mean-readout function, under the combined assumptions for Theorem~\ref{thm: main unit upper} and Lemma~\ref{lem: bound unit mpgnn}, 
\begin{equation*}
 \genb(\mrm^{\alpha,p}(\mun),\mu) 
     \leq \frac{\alpha M_{p}^2 M_{\ell'}^2}{n} \,.
\end{equation*}
with $M_{p}=w_{2,p} L_{\kappa}B_f(\norm{W_{3,p}}_2 +G_{\max} L_\rho L_{\zeta}\norm{W_{1,p}}_2)$.
\end{proposition}
Similar discussions as in Remark~\ref{remark: graph-filter and infinite norm} and Remark~\ref{remark: graph filter and frobenius norm} hold for the effect of graph filter choices in an MPGNN. In a similar approach to Proposition~\ref{Prop: MPGNN bound}, we can derive an upper bound on the generalization error of MPGNN by combining Lemma~\ref{lem: bound unit mpgnn} with Theorem~\ref{thm: main unit upper rademacher}. 
\subsection{Comparison to Previous Works}\label{subsec:comp}
\begin{table*}[ht]\label{Tab: comparison}
\centering
\caption{Comparison of generalization bounds in over-parameterized one-hidden layer GCN. The width of the hidden layer, the number of training samples, maximum and minimum degree of graph data set are denoted as $h$, $n$, $d_{\max}$ and $d_{\min}$, respectively. We have $\tilde{d}_{\max}=d_{\max}+1$ and $\tilde{d}_{\min}=d_{\min}+1$. We examine three types of upper bounds: Probability (P), High-Probability (HP) and Expected (E).“N/A” means not applicable.}
\vspace{-2mm}
\resizebox{\linewidth}{!}{
\begin{tabular}{cccccc} 
\toprule
   Approach & 
   $\tilde{d}_{\max},\tilde{d}_{\min}$ & 
   $h$  
   & 
   $n$ & Bound Type \cr  
\midrule
\makecell{VC-Dimension \cite{scarselli2018vapnik}}    & N/A & $O(h^4)$ & $O(1/\sqrt{n})$ & HP \\
\makecell{Rademacher Complexity   \cite{garg2020generalization} } &$O(\tilde{d}_{\max}\log^{1/2}(\tilde{d}_{\max}))$ &$O(h\sqrt{\log(h)})$ & $O(1/\sqrt{n})$   & HP \\
\makecell{PAC-Bayesian \cite{liao2020pac} } &$O(\tilde{d}_{\max})$ &$O(\sqrt{h\log(h)})$ & $O(1/\sqrt{n})$ & HP  \\ 
\makecell{PAC-Bayesian\cite{ju2023generalization}}
&N/A & $O(\sqrt{h})$& $O(1/\sqrt{n})$ & HP \\
\makecell{Continuous MPGNN   \cite{maskey2022generalization}}&N/A & N/A  & $O(1/\sqrt{n})$   & P 
\\
\makecell{Rademacher Complexity (this paper, Theorem~\ref{thm: main unit upper rademacher})} &$O\big((\tilde{d}_{\max}/\tilde{d}_{\min})^{3/4}\big)$& N/A &  $O(1/\sqrt{n})$ & HP \\
\textbf{Functional Derivative (this paper, Theorem~\ref{thm: main unit upper})} &$O(\tilde{d}_{\max}/\tilde{d}_{\min})$& \textbf{N/A} &  \bm{$O(1/n)$} & E\\
\bottomrule
\end{tabular}
}
\vspace{-3mm}
\label{table:comparison}
\end{table*}

We compare our generalization error upper bound with other generalization error upper bounds derived by the VC-dimension approach~\citep{scarselli2018vapnik}, bounding Rademacher Complexity~\citep{garg2020generalization}, PAC-Bayesian bounds via perturbation analysis~\citep{liao2020pac}, PAC-Bayesian bounds based on the Hessian matrix of the loss ~\citep{ju2023generalization}, and probability bounds for continuous MPGNNs on large random graphs~\citep{maskey2022generalization}. We also compare with our upper bound in App.\ref{app: rademacher} obtained via Rademacher complexity in App.\ref{app: rademacher}. To compare different bounds, we analyze their convergence rate concerning the width of the hidden layer $(h)$ and the number of training samples $(n)$. We also examine the type of bounds on the generalization error, including high-probability bounds, probability bounds, and expected bounds. 
In high-probability bounds, the upper bound depends on $\log(1/\delta)$ for $\delta\in (0,1)$, as opposed to $1/\delta$ in probability bounds. For small $\delta$, the high probability bounds are tighter with respect to probability bounds. We use a fixed $\alpha$ in Theorem~\ref{thm: main unit upper} for our comparison in Table~\ref{table:comparison}. More discussion is found in App.~\ref{App: dis com}.

As shown in Tab.~\ref{table:comparison}, the upper bounds in \cite{scarselli2018vapnik,garg2020generalization,liao2020pac} and \cite{ju2023generalization} are vacuous for infinite width ($h\rightarrow \infty$) of one hidden layer. We also provide results for other non-linear functions that are functions of the sum of the final node representations. \cite{maskey2022generalization} proposed an expected upper bound on the square of the generalization error of continuous MPGNNs, which is independent of the width of the layers. Then, via the Markov inequality, they provide a probability upper bound on the generalization error of continuous MPGNNs by considering the mean-readout, obtaining a convergence rate of $O(1/\sqrt{n})$. While the random graph model in \cite{maskey2022generalization} is based on graphons, here we do not assume that the graph samples arise from a specific random graph model. To the best of our knowledge, this is the first work to represent an upper bound on the generalization error with the convergence rate of $O(1/n)$.

Inspired by the NTK approach of \cite{jacot/franck/hongler:2018}, \cite{du2019graph} proposed the graph neural tangent kernel (GNTK) as a GNN model for layers of infinite width. They provided a high-probability upper bound on the true risk based on \cite{bartlett2002rademacher} for the sum-readout function, for their proposed structure, GNTK, which is different from GCNs and MPGNNs.
However, as noted in \cite{fang2021mathematical}, the neural tangent kernel has some limitations for over-parameterized analysis of neural networks when compared to mean-field analysis.
Therefore, we do not compare with \cite{du2019graph}. Finally, the upper bound in \cite{verma2019stability} focuses on stability analysis for one-hidden-layer GNNs in the context of semi-supervised node classification tasks. In \cite{verma2019stability}, the data samples are node features rather than a graph, and 
the node features within each graph sample are assumed to be i.i.d., whereas we only assume that the graph samples themselves are i.i.d. and, therefore, we do not compare with \cite{verma2019stability}.
\section{Experiments}
\looseness=-1 Our investigation 
focuses on the over-parameterized one-hidden-layer case in the context of GNNs, such as GCNs and MPGNNs.
Prior works on graph classification tasks \citep{garg2020generalization,liao2020pac,ju2023generalization,maskey2022generalization,du2019graph} have demonstrated that the upper bounds on generalization error tend to increase as the number of layers in the network grows. An examination of the over-parameterized one-hidden-layer case may be particularly instructive in elucidating the generalization error performance of GNNs in the context of graph classification tasks. 

For this purpose, we investigate the effect of the number of hidden neurons $h$ on the true generalization error of GCNs and MPGNNs for the (semi-)supervised graph classification task on both synthetic and real-world data sets. We use a supervised ratio of $\beta_\text{sup}\in\{0.7, 0.9\}$ for $h\in\{4,8, 16, 32, 64, 126, 256\}$ for our experiments detailed in App.~\ref{App: experiments}. For synthetic data sets, we generate three types of Stochastic Block Models (SBMs) and two types of Erd\H{o}s-R\'enyi (ER) models, with 200 graphs for each type. We also conduct experiments on a bioinformatics data set called PROTEINS~\citep{borgwardt2005protein}. Details on implementation, data sets, and additional results are in App.~\ref{App: experiments}.

In our experiments, we use
the logistic loss $\ell(\Psi(\mrm,\mbx),y)=\log(1+\exp(-\Psi(\mrm,\mbx) y))$ for binary classification, where $y$ is the true label and $\Psi(\mrm,\mbx)$ is the mean- or sum-readout function for GCNs and MPGNNs. The empirical risk is then
$\frac{1}{n}\sum_{i=1}^n \log(1+\exp(-\Psi(\mrm,\mbx_i) y_i)).$

\begin{figure}
\vspace{-1mm}
\centering
\begin{subfigure}[ht]{0.45\linewidth}
    \includegraphics[width=1.1\linewidth]{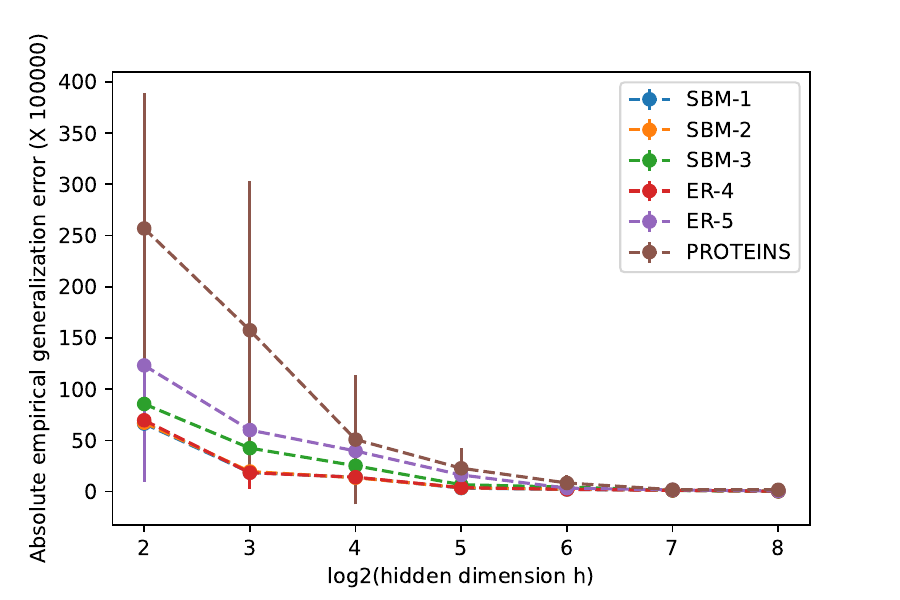}
    \caption{GCN.}
    \label{fig:GCN_mean}
\end{subfigure}
    \begin{subfigure}[ht]{0.45\linewidth}
      \centering
      \includegraphics[width=1.1\linewidth]{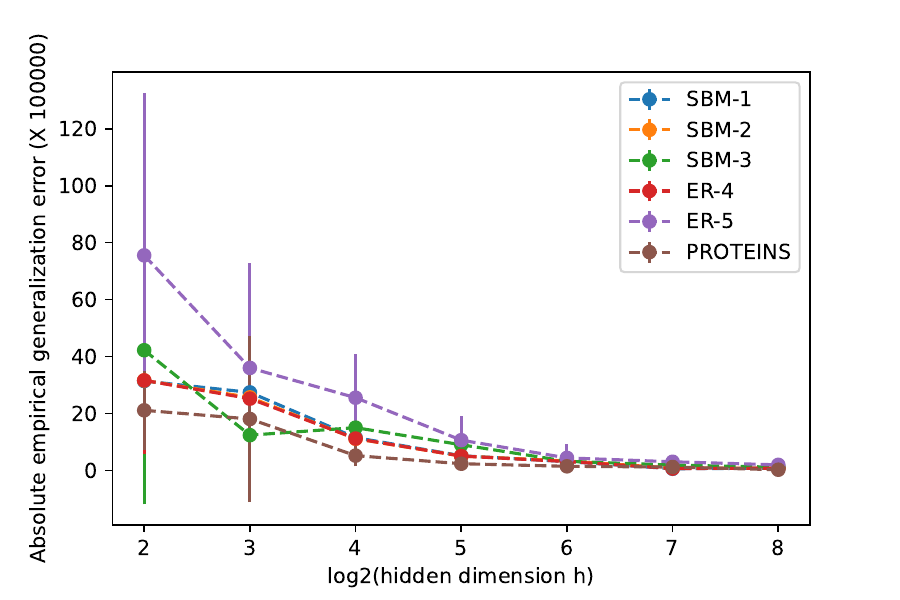}
    \caption{MPGNN.}
    \label{fig:MPGNN_mean}
    \end{subfigure}
    \vspace{-2mm} 
    \caption{Absolute empirical generalization error (\pmb{$\times 10^5$}) for different widths $h$ of the hidden layer. We employ a mean-readout function and a supervised ratio of $\beta_\text{sup}=0.7,$ for GCN and MPGNN. Values are averaged over ten runs. Error bars indicate one standard deviation.}
    \vspace{-4mm}
\label{fig:main_GE}
\end{figure}
From Fig.~\ref{fig:main_GE} (and App.~\ref{App: experiments} with detailed mean and standard deviation values), we observe a consistent trend: as the value of $h$ increases, the absolute generalization error decreases. This observation shows that the upper bounds dependent on the width of the layer fail to capture the trend of generalization error in the over-parameterized regime. We provide extended actual absolute generalization errors in this section and the values of our upper bounds are provided in App.~\ref{App: experiments}.
\section{Conclusions and Future Work}
\label{sec_conclusions}
 This work develops 
generalization error upper bounds for one-hidden-layer GCNs and MPGNNs for graph classification in the over-parameterized regime. Our analysis is based on a mean-field approach. 
Our upper bound on the generalization error of one-hidden-layer GCNs and MPGNNs with the KL-regularized empirical risk minimization is of the order of $O(1/n)$, where $n$ is the number of graph data samples, in the mean-field regime. This order is a significant improvement over previous work, see Table \ref{table:comparison}. 

The main limitation of our work is that it
considers only one hidden layer for graph convolutional networks and message-passing graph neural networks. Inspired by \citet{sirignano/spiliopoulos:2019}, we aim to apply our approach to deep graph neural networks and investigate the effect of depth on the generalization performance. Furthermore, we plan to expand the current framework to study the generalization error of hypergraph neural networks, using the framework introduced in~\cite{feng2019hypergraph}.
\section*{Acknowledgements} Gholamali Aminian, Gesine Reinert, {\L}ukasz Szpruch and Samuel N. Cohen acknowledge the support of the UKRI Prosperity Partnership Scheme (FAIR) under EPSRC Grant EP/V056883/1 and the Alan Turing Institute. Gesine Reinert is also supported in part by EPSRC grants EP/W037211/1 and EP/R018472/1. Yixuan He is supported by a Clarendon scholarship from University of Oxford, and also acknowledges the support from the Alan Turing Institute and G-Research for presenting this work. For the purpose of Open Access, the authors have applied a CC BY public copyright licence to any Author Accepted Manuscript (AAM) version arising from this submission.

\section*{Impact Statement}
This paper presents work whose goal is to advance the field of Machine Learning by providing a theoretical underpinning. There are many indirect potential societal consequences of our work through applications of empirical risk, see for example discussions in  \citet{abrahamsson2006risk} and \citet{tran2021differentially}, and we believe that no direct consequences warrant being highlighted.

\bibliography{Refs}
\bibliographystyle{icml2024}

\clearpage
\appendix
\onecolumn

\section{Preliminaries}\label{App: preliminaries}
Notations in this paper are summarized in Table~\ref{Table: notation}.

\begin{table}[ht]
    \centering
    \caption{Summary of notations in the paper}
   \resizebox{\linewidth}{!}{ \begin{tabular}{cl|cl}
         \toprule
         Notation&  Definition & Notation&  Definition\\
         \midrule
         $\mathcal{W}$ & Parameter space of the model &
         $\mathcal{F}$ & Feature matrices space\\
         $\mathcal{A}$ & Adjacency matrices space&
         $\mbF$ & Matrix of feature nodes of a graph sample\\
         $\mbA$ & Adjacency matrix of a graph sample &
         $\mbX$ & input graph sample where $\mbX=(\mbA,\mbF)$\\
         $Y$ & Label of input graph &
         $Z_i$ & $i$-th graph sample $(\mbX,Y)$\\
         $\mbZn$ & The set of data training samples&
         $D$ & Degree matrix of $A$\\
         $\tilde{L}$ & Symmetric normalized graph filter&
        $n$ & Number of graph data samples  \\
        $N_{\max}$ & Maximum number of nodes of all graph samples&
        $d_{\max}$ & Maximum node degree of all graph samples \\
        $d_{\min}$ & Minimum node degree of all graph samples &
        $G(\mbA)$ & Graph filter with input matrix $\mbA$\\
        $\mun$ & Empirical data measure &
        $\mur$ & Replace-one sample empirical data measure \\
        $M_{\phi}$ & Bound on unit function &
        $M_{\ell'}$ & Bound on $|\partial_{\hat y}\ell(\hat y, y)|$\\
        $L_\zeta$ & Lipschitz parameter of function $\zeta(\cdot)$&
        $L_\rho$& Lipschitz parameter of function $\rho(\cdot)$\\
        $L_\kappa$& Lipschitz parameter of function $\kappa(\cdot)$&
        $L_\psi$ & Lipschitz parameter of readout function $\psi(\cdot)$\\
        $L_{\varphi}$& Lipschitz parameter of activation function $\psi(\cdot)$&
        $B_f$ & Bound on node features\\
        $\alpha$ & Inverse temperature&
        $h$ & Width of hidden layer\\
        $\pi$ & Prior measure over parameters&
        $\mrm^{\alpha}$ & The Gibbs measure for General GNN model\\
        $\mrm^{\alpha,c}$ & The Gibbs measure for GCN model&
        $\mrm^{\alpha,p}$ & The Gibbs measure for MPGNN model\\
        $m_{\mathrm{true}}$ & True distribution over data samples&
        $G_{\max}$ & $\min(\norm{G(\mathcal{A})}_{\infty}^{\max},\norm{G(\mathcal{A})}_{F}^{\max})$\\
        $\Psi(\cdot)$ & Readout function & $(W_{1,c},W_{2,c})$ & Parameters of Neuron unit\\
        $(W_{1,p},W_{2,p},W_{3,p})$ & Parameters of MPU unit & $\genb(\mrm(\mun),\mu)$ & Generalization error under parameter measure $\mrm(\mun)$\\ $\beta_{\mathrm{sup}}$ & Supervised ratio & $\norm{\mathbf{Y}}_{F}$ & $\sqrt{\sum_{j=1}^k\sum_{i=1}^q\mathbf{Y}^2[j,i]}$ \\
        $\norm{\mathbf{Y}}_{\infty}$& $\max_{1\leq j\leq k}\sum_{i=1}^q|\mathbf{Y}[j,i]|$ & $\KLr(p\|q)$ & $\int_{\mathbb{R}^d} p(x)\log(p(x)/q(x))\mrd x$\\
        $\KLr_{\mathrm{sym}}(p\|q)$ & $\KLr(p\|q)+\KLr(q\|p)$ & $\phi(\cdot)$ & Unit function \\
         \bottomrule
    \end{tabular}}
    \label{Table: notation}
\end{table}

Let us recapitulate all the assumptions required for our proofs. 

\begin{repassumption}{Ass: on loss NN}[Loss function]
      The loss function, $(\hat y,y)\mapsto \ell(\hat y,y)$, satisfies the following conditions,
      \begin{enumerate}[(i)]
          \item\label{ass: bounded grad} The gradient of the loss function $(\hat y,y)\mapsto \ell(\hat y,y)$ with respect to $\hat y$ is continuous and uniformly bounded for all $\hat y,y\in\mathcal{Y}$, i.e., there is a constant $M_{\ell'} $ such that $|\partial_{\hat y}\ell(\hat y, y)|\leq M_{\ell'}.$
          \item\label{ass: convex loss} We assume that the loss function is convex with respect to $\hat y$.
      \end{enumerate}
 \end{repassumption}
 
\begin{repassumption}{ass: bounded Unit function}[Unit function]
    The unit function  $(w,G(\mbA)[j,:]\mbf)\mapsto \phi(w,G(\mbA)[j,:]\mbf)$ with graph filter  $G(\cdot)$  is uniformly bounded, i.e, there is a constant $M_{\phi} $ such that $\sup_{w\in\mathcal{W},\mbf\in\mathcal{F},\mbA\in\mathcal{A}}|\phi(w,G(\mbA)[j,:]\mbf)|\leq M_{\phi}$ for all $j\in[N]$ and $\mbf \in \mathcal{F}$.
\end{repassumption}

\begin{repassumption}{ass: readout function}[Readout function] The function $\psi:\mbR\mapsto\mbR$ is $L_{\psi}$-Lipschitz-continuous, i.e., there is a constant $L_\psi$ such that $|\psi(x_1)-\psi(x_2)|\leq L_{\psi}|x_1-x_2|$ for all $x_1,x_2\in \mbR$,  and zero-centered, i.e., $\psi(0)=0$.
\end{repassumption}

\begin{repassumption}{Ass: Feature node bounded}[Bounded node features]
For every graph, the node features are contained in an $\ell_2$-ball of radius $B_f$. In particular, $\norm{\mbF[i;]}_2\leq B_f$ for all $i\in[N]$.
\end{repassumption}

\begin{repassumption}{ass: bounded activation function}[Activation functions in neuron unit]  The activation function $\varphi:\bbR\mapsto\bbR$ is $L_\varphi$-Lipschitz, so that $|\varphi(x_1)-\varphi(x_2)|\leq L_\varphi |x_1-x_2|$ for all $x_1,x_2\in \mbR$,  and zero-centered, i.e., $\phi(0)=0$.  
\end{repassumption}

\begin{repassumption}{ass: bound unit MPGNN}[Non-linear functions in the MPGNN unit]
    The non-linear functions $\zeta:\bbR^{N\times k}\mapsto\bbR^{N\times k}$,  $\rho:\bbR^{k}\mapsto\bbR^{k}$ and $\kappa:\bbR\mapsto\bbR$ are zero-centered, i.e., $\zeta(\pmb{0}^{N\times k})=\pmb{0}^{N\times k}$, $\rho(\pmb{0}^{k})=\pmb{0}^{k}$ and $\kappa(0)=0$, and Lipschitz with parameters $L_\zeta$, $L_\rho$ and $L_\kappa$ under vector 2-norm\footnote{For function $\zeta(\cdot)$, we consider that it is Lipschitz under a vector 2-norm, }, respectively.
\end{repassumption}

\textbf{Total variation distance:} The total variation distance between two densities $p(x)$ and $q(x)$, is defined as $\mathbb{TV}(p,q):=\frac{1}{2}\int_{\mathbb{R}^d} |p(x)-q(x)|\mrd x$.

The following lemmas are needed for our proofs.
\begin{lemma}[Donsker’s representation of KL divergence]\label{lemma: donsker}
    Let us consider the variational representation of the KL divergence between two probability distributions $m_1$ and $m_2$ on a common space $\psi$ given by~\citet{polyanskiy2022information},
\begin{align}
    \KLr(m_1||m_2)=\sup_{f} \int_\psi f dm_1 -\log\int_\psi e^f dm_2,
\end{align}
where $f\in \mathfrak{F}=\{f:\psi\rightarrow\mbR \text{ s.t.  } \bbE_{m_2}[e^f]< \infty \}$.
\end{lemma}
\begin{lemma}[Kantorovich-Rubenstein duality of total variation distance]\label{lem: tv} The Kantorovich-Rubenstein duality (variational representation) of the total variation distance is as follows~\cite{polyanskiy2022information}:
\begin{equation}\label{Eq: tv rep}
    \mathbb{TV}(m_1,m_2)=\frac{1}{2L}\sup_{g \in \mathcal{G}_L}\left\{\mathbb{E}_{Z\sim m_1}[g(Z)]-\mathbb{E}_{Z\sim m_2}[g(Z)]\right\},
\end{equation}
where $\mathcal{G}_L=\{g: \mathcal{Z}\rightarrow \mathbb{R}, ||g||_\infty \leq L \}$. 
\end{lemma}
\begin{lemma}[Bound on infinite norm of the symmetric normalized graph filter]\label{lem: bound on inf norm of L}
Consider a graph sample $G$ with adjacency matrix $A$. For the symmetric normalized graph filter, we have $\norm{\tilde{D}^{-1/2}(A+I)\tilde{D}^{-1/2}}_{\infty}\leq \sqrt{\frac{d_{\max}+1}{d_{\min}+1}}$ where $d_{\max}$ is the maximum degree of graph $G$ with adjacency matrix $A$. 
\end{lemma}
\begin{proof}
Recall that $\tilde{A}=A+I$. We have,
\begin{equation}
\begin{split}
\norm{\tilde{L}}_{\infty}&=
\norm{\tilde{D}^{-1/2}\tilde{A}\tilde{D}^{-1/2}}_{\infty}\\
&=\max_{i\in[N]}\sum_{j=1}^N \frac{\tilde{A}_{ij}}{\sqrt{d_{i}+1}\sqrt{d_{j}+1}}\\
&\leq\frac{1}{\sqrt{d_{\min}+1}} \max_{i\in[N]}\sum_{j=1}^N \frac{\tilde{A}_{ij}}{\sqrt{d_{i}+1}}\\
&\leq  \sqrt{\frac{d_{\max}+1}{d_{\min}+1}}.
\end{split}
\end{equation}
\end{proof}
\begin{lemma}[\citealp{meyer2023matrix}]\label{lemma: bound on F-norm}
For a matrix $\mathbf{Y}\in \mathbb{R}^{k\times q}$, we have $\norm{\mathbf{Y}}_{F}\leq \sqrt{R}\norm{\mathbf{Y}}_2$, where $r$ is the rank of  $\mathbf{Y}$ and $\norm{\mathbf{Y}}_2$ is the 2-norm of $\mathbf{Y}$ which is equal to the maximum singular value of $\mathbf{Y}$.
\end{lemma}
\begin{lemma}[\citealp{verma2019stability}]\label{lemma: bound on second norm}
    For the symmetric normalized graph filter, i.e., $\tilde{L}=\tilde{D}^{-1/2}\tilde{\mathbf{A}}\tilde{D}^{-1/2}$, we have $\norm{\tilde{L}}_2=1$. For the random walk graph filter, i.e., $D^{-1}\mathbf{A}+I$, we have $\norm{D^{-1}\mathbf{A}+I}_2=2$, where $\norm{A}_2$ is the 2-norm of matrix A.
\end{lemma}
\begin{lemma}[Hoeffding lemma~\cite{wainwright2019high}]\label{lemma: hoeffding}
For bounded random variable, $a\leq X\leq b$, and all $\lambda\in\mathbb{R}$, we have,
\begin{align}
    \mathbb{E}[\exp(\lambda (X-\mathbb{E}[X]))]\leq \exp(\lambda (b-a)^2/8).
\end{align}
\end{lemma}
\begin{lemma}[Pinsker's inequality]\label{lem: pinsker}
    The following upper bound holds on the total variation distance between two measures $m_1$ and $m_2$~\cite{polyanskiy2022information}:
    \begin{equation}
        \mathbb{TV}(m_1,m_2)\leq \sqrt{\frac{\KLr(m_1\|m_2)}{2}}. 
    \end{equation}
\end{lemma}
In the following, we apply some lemmata from \citet{aminian2023mean} where we use $\ell\big(\Psi(\mrm(\mun),\ols{\mathbf{x}}_1),\ols{\mathbf{y}}_1\big)$ instead of $\ell\big(\mrm(\mun),\ols{Z}_1\big)$ in \citet{aminian2023mean}.

\begin{lemma}{\citep[Proposition~3.3]{aminian2023mean}}\label{lem: convex lower} Given Assumption~\ref{Ass: on loss NN}, the following lower bound holds on the generalization error of a Generic GNN,
    \begin{equation*}
        \genb(\mrm^{\alpha}(\mun),\mu)\geq \bbE_{\mbZn,\ols{Z}_1}\Big[\int_{\mathcal{W}} \partial_{\hat y}\ell\big(\Psi(\mrm(\mur),\ols{\mathbf{x}}_1),\ols{\mathbf{y}}_1\big)\frac{\delta \Psi}{\delta \mrm}(\mrm(\mur),\ols{\mathbf{x}}_1,w)(\mrm(\mun)-\mrm(\mur))(\mathrm{d}w)\Big] .
    \end{equation*}
\end{lemma}

\begin{remark}[Another representation of Lemma~\ref{lem: convex lower}]\label{remark: lemma conv}
    Due to the fact that data samples and $\ols{Z}_1$ are i.i.d., Lemma~\ref{lem: convex lower} can be represented as follows,
     \begin{equation*}
        \genb(\mrm^{\alpha}(\mun),\mu)\geq \bbE_{\mbZn,\ols{Z}_1}\Big[\int_{\mathcal{W}} \partial_{\hat y}\ell\big(\Psi(\mrm(\mun),\mathbf{x}_1),\mathbf{y}_1\big)\frac{\delta \Psi}{\delta \mrm}(\mrm(\mun),\mathbf{x}_1,w)(\mrm(\mur)-\mrm(\mun))(\mathrm{d}w)\Big] .
    \end{equation*}
\end{remark}

\begin{lemma}{\citep[Theorem~3.2]{aminian2023mean}}\label{lem: gen rep}
    Given Assumption~\ref{Ass: on loss NN}.\ref{ass: bounded grad}, the following representation of the generalization error holds:
   \begin{equation*}
        \genb(\mrm^{\alpha}(\mun),\mu)=\bbE_{\mbZn,\ols{Z}_1}\Big[\int_{0}^1\int_{\mathcal{W}} \partial_{\hat y}\ell\big(\Psi(\mrm_{\lambda}(\mun),\ols{\mathbf{x}}_1),\ols{\mathbf{y}}_1\big)\frac{\delta \Psi}{\delta \mrm}(\mrm_{\lambda}(\mun),\ols{\mathbf{x}}_1,w)(\mrm(\mun)-\mrm(\mur))(\mathrm{d}w) \mrd \lambda\Big]
    \end{equation*}
    where $\mrm_{\lambda}(\mun)=\mrm(\mur)+\lambda(\mrm(\mun)-\mrm(\mur))$ for $\lambda\in[0,1]$.
\end{lemma}

The following preliminaries are needed for our Rademacher complexity analysis. 

\textbf{Rademacher complexity:} For a hypothesis set,  $\mathcal{H}$ of functions $f_h:\mathcal{Z}\mapsto \mbR$ and set $\pmb{\sigma}=\{\sigma_i\}_{i=1}^n$, the empirical Rademacher complexity $\hat{\mathfrak{R}}_{\mathrm{Z}}(\mathcal{H})$ with respect to set $\mbZn$ is defined as:
\begin{equation}
  \hat{\mathfrak{R}}_{\mbZn}(\mathcal{H}):=\mathbb{E}_{\pmb{\sigma}}\Big[\sup_{f_h\in\mathcal{H}}\frac{1}{n}\sum_{i=1}^n \sigma_i f_h(Z_i)\Big]\,,
\end{equation}
where $\sigma=\{ \sigma_i\}_{i=1}^n$ are i.i.d random variables and $\sigma_i\in\{-1,1\}$ for all $i\in[r]$ with equal probability.

In addition to the previous assumptions, for Rademacher complexity analysis we also need the following assumption.
\begin{repassumption}{Ass: on loss NN bounded}[Bounded loss function]
      The loss function, $(\hat y,y)\mapsto \ell(\hat y,y)$, is bounded for all $\hat y,y\in\mathcal{Y}$, i.e., $ 0\leq\ell(\hat y, y)\leq M_{\ell}$.
      \end{repassumption}

\begin{lemma}[Uniform bound~\cite{mohri2018foundations}] \label{lem:uniform_bound}
  Let $\mathcal{F}_u$ be the set of functions $f:\mathcal{Z}\to[0,M_\ell]$  and 
  $\mu$ be a distribution over $\mathcal{Z}$.
  Let $S=\{z_i\}_{i=1}^n$ be a set of size $n$ i.i.d. drawn from $\mathcal{Z}$.
  Then, for any $\delta \in (0,1)$, with probability at least $1-\delta$ over the choice of $S$, we have
  \begin{equation*}
    \sup_{f \in \mathcal{F}_n} \left\{ \bbE_{Z\sim \mu}[f(Z)] - \frac{1}{n}\sum_{i=1}^nf(z_i) \right\} 
    \leq 2 \rademacher_S(\mathcal{F}_n) + 3M_\ell\sqrt{\frac{1}{2n}\log\frac{2}{\delta}}.
  \end{equation*}
\end{lemma}

The contraction lemma helps us to estimate the Rademacher complexity.
\begin{lemma}[Talagrand's contraction lemma~\citep{shalev2014understanding}] \label{lem:contraction}
  Let $\phi_i : \mbR \rightarrow \mbR$ $(i\in \{1,\ldots,n\})$ be $L$-Lipschitz functions 
  and $\mathcal{F}_n$ be a set of functions from $\mathcal{Z}$ to $\mbR$. 
  Then it follows that for any $\{z_i\}_{i=1}^n \subset \mathcal{Z}$,
  \begin{equation*}
    \bbE_\sigma\left[ \sup_{f\in\mathcal{F}_n} \frac{1}{n}\sum_{i=1}^n\sigma_i \phi_i( f(z_i))\right]
    \leq L \bbE_\sigma\left[ \sup_{f \in \mathcal{F}_r} \frac{1}{n}\sum_{i=1}^n\sigma_i f(z_i)\right].
  \end{equation*}
\end{lemma}

    

The units of GCNs and MPGNN are shown in Figure~\ref{Fig:main}.
\begin{figure}[tbh]
     \centering
     \begin{subfigure}[b]{0.45\textwidth}
         \centering
         \begin{tikzpicture}[scale=0.55]

 \draw[->] (-1.5,5.5) -- (-0.5,5.5);
  \draw[->] (1,5.5) -- (1.5,5.5);
 \draw[->] (-1.5,4) -- (1.5,4);
 \draw[->] (-1.5,3) -- (1.5,3);
 \draw[->] (3,4) -- (4.75,4);
 \draw[->] (3.5,5) -- (4.75,5);
  \draw[->] (6.25,4) -- (6.75,4);

\node at (-1,6){$\mathbf{A}_q$};
\node at (0,4.5){$\mathbf{F}_q$};
\node at (0,3.5){$W_{1,c}$};
\node at (4,5.5){$W_{2,c}$};

\draw [fill=white!93!blue,rounded corners=2pt](-0.5,5) rectangle (1,6);

\draw [fill=white!93!blue,rounded corners=2pt](1.5,2.5) rectangle (3,6);

\draw [fill=white!93!blue,rounded corners=2pt](4.75,2.5) rectangle (6.25,6);

\node at (2.25,4.25){$\varphi(\cdot)$};
\node at (5.5,4.25){$\phi_c(\cdot)$};

\node at (0.25,5.5){$G(\cdot)$};

\end{tikzpicture}
\caption{Neuron Unit}
\label{Fig: gcn}
     \end{subfigure}
     \begin{subfigure}[b]{0.45\textwidth}
         \centering
         \begin{tikzpicture}[scale=0.55]

 \draw[->] (-1.5,5.5) -- (-0.5,5.5);
  \draw[->] (1,5.5) -- (1.5,5.5);
 \draw[->] (-1.5,4) -- (1.5,4);
 \draw[->] (3,4) -- (4.75,4);
 \draw[->] (3.5,5) -- (4.75,5);
  
  \draw[->] (3,2.75) -- (4.75,2.75);

  \draw[->] (6.25,4) -- (8,4);
  \draw[->] (6.5,5) -- (8,5);
   \draw[->] (9.5,4) -- (10,4);

\node at (-1,6){$\mathbf{A}_q$};
\node at (-1,4.5){$\mathbf{F}_q$};
\node at (4,3.25){$W_{3,p}$};
\node at (4,5.5){$W_{1,p}$};
\node at (7.25,5.5){$W_{2,p}$};

\draw [fill=white!93!blue,rounded corners=2pt](-0.5,5) rectangle (1,6);

\draw [fill=white!93!blue,rounded corners=2pt](-0.5,3.5) rectangle (1,4.5);

\node at (0.25,5.5){$G(\cdot)$};
\node at (0.25,4){$\zeta(\cdot)$};

\draw [fill=white!93!blue,rounded corners=2pt](1.5,3.5) rectangle (3,6);

\draw [fill=white!93!blue,rounded corners=2pt](4.75,2.5) rectangle (6.25,6);

\draw [fill=white!93!blue,rounded corners=2pt](8,2.5) rectangle (9.5,6);

\node at (2.25,4.75){$\rho(\cdot)$};
\node at (5.5,4.25){$\kappa(\cdot)$};
\node at (8.75,4.25){$\phi_p(\cdot)$};
\end{tikzpicture}
\caption{MPU Unit}
\label{Fig: MPGNN}
     \end{subfigure}
     \caption{Units of GCN and MPGNN }
     \label{Fig:main}
 \end{figure}
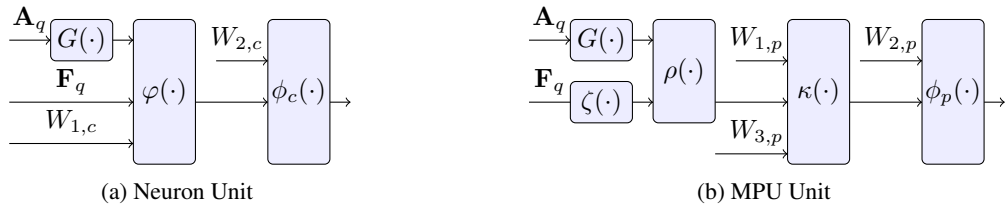

\section{Other Related Works}
\textbf{Mean-field:} Our study employs the mean-field framework utilized in a recent line of research \cite{chizat/bach:2018,mei/montanari/nguyen:2018,mei/misiakiewicz/montanari:2019,sirignano/spiliopoulos:2019,MFLD}. The convergence of gradient descent for training one-hidden layer NNs with infinite width under certain structural assumptions is established by \citet{chizat/bach:2018}. The study of \citet{mei/montanari/nguyen:2018} proved the global convergence of noisy stochastic gradient descent and established approximation bounds between finite and infinite neural networks. Furthermore, \citet{mei/misiakiewicz/montanari:2019} demonstrated that this approximation error can be independent of the input dimension in certain cases, and established that the residual dynamics of noiseless gradient descent are close to the dynamics of NTK-based kernel regression under some conditions.

 \textbf{Graph Representation Learning:} 
Numerous Graph Neural Networks (GNNs) have emerged for graph-based tasks, spanning node, edge, and graph levels. GCNs, introduced in~\citet{kipfsemi}, simplify the Cheby-Filter from~\citet{defferrard2016convolutional} for one-hop neighbors. MPGNNs, as proposed in~\citet{gilmer2017neural}, outline a general GNN framework, treating graph convolutions as message-passing among nodes and edges. For graph-level tasks, like graph classification, a typical practice involves applying a graph readout (pooling) layer after graph filtering layers, composed of graph filters. The readout (pooling) layer aggregates node representations to create a graph-wide embedding. Common graph readout choices encompass set pooling methods like direct sum application, mean~\citep{hamilton2020graph}, or maximum~\citep{mesquita2020rethinking}, as well as combinations of LSTMs with attention~\citep{vinyals2015order} and graph coarsening techniques leveraging graph structure~\citep{ying2018hierarchical, cangea2018towards, gao2019graph}. In this paper, we investigate how mean and sum impact the generalization errors of GCNs and MPGNNs.

\textbf{Node Classification and Generalization Error:} For node classification tasks, \citet{verma2019stability} discussed the generalization error under node classification for GNNs via algorithm stability analysis. The work by \citet{zhou2021generalization} extended the results in \citet{verma2019stability} and found that increasing the depth of GCN enlarges its generalization error for the node classification task. A Rademacher complexity analysis was applied to GCNs for the node classification task by \citet{lv2021generalization}. Based on transductive Rademacher complexity, a high-probability upper bound on the generalization error of the transductive node classification task was proposed by \citet{esser2021learning}. The transductive modeling of node classification was studied in \citet{oono2020optimization} and \citet{tang2023towards}. \citet{cong2021provable} presented an upper bound on the generalization error of GCNs for node classification via transductive uniform stability, building on the work of \citet{el2006stable}.
In contrast, our research focuses on the task of graph classification, which involves (semi-)supervised learning on graph data samples rather than semi-supervised learning for node classification.

\textbf{Neural Tangent Kernels:}The neural tangent kernel (NTK) model, as described by \citet{jacot/franck/hongler:2018}, elucidates the learning dynamics inherent in neural networks when subjected to appropriate scaling conditions. This explication relies on the linearization of learning dynamics in proximity to its initialization. Conclusive evidence pertaining to the (quantitative) global convergence of gradient-based techniques for neural networks has been established for both regression problems \cite{suzuki2021deep} and classification problems \cite{cao2019generalization}. The NTK model, founded upon linearization, is constrained in its ability to account for the phenomenon of "feature learning" within neural networks, wherein parameters exhibit the capacity to traverse and adapt to the inherent structure of the learning problem. Fundamental to this analysis is the linearization of training dynamics, necessitating the imposition of appropriate scaling conditions on the model~\cite{chizat1812note}. Consequently, this framework proves inadequate for elucidating the feature learning aspect of neural networks \cite{yang2020feature,fang2021mathematical}. Notably, empirical investigations have demonstrated the superior expressive power of deep learning over kernel methods concerning approximation and estimation errors \cite{ghorbani2021linearized}. In certain contexts, it has been observed that neural networks, optimized through gradient-based methodologies, surpass the predictive performance of the NTK model, and more broadly, kernel methods, concerning generalization error or true risk \cite{allen2019can}.

\section{Proofs and details of Section~\ref{sec_main_results}}\label{app: proofs of main section}

\subsection{Generic GNN}
\begin{repproposition}{Prop: KL bound}\textbf{(restated)}
  Let Assumptions~\ref{Ass: on loss NN}, \ref{ass: bounded Unit function}, and \ref{ass: readout function} hold. Then, the following bound holds on the generalization error of generic GNN,
    \begin{equation}
    \begin{split}
        &\Big|\genb(\mrm(\mun),\mu)\Big| \leq \big(M_{\ell'}L_{\psi}N_{\max}M_\phi/\sqrt{2} \big)\bbE_{\mbZn,\ols{Z}_1}\Big[ \sqrt{ \KLr\big(\mrm(\mun)\|\mrm(\mur)\big)}\Big].
    \end{split}
    \end{equation}
\end{repproposition}
In the following, we provide two technical proofs for Proposition~\ref{Prop: KL bound}.

\begin{proof}[\textbf{Proof of Proposition~\ref{Prop: KL bound} via Lemma~\ref{lemma: donsker} and Lemma~\ref{lemma: hoeffding}}]
 From Lemma~\ref{lemma: donsker}, the following representation of KL divergence holds between two probability distributions $m_1$ and $m_2$ on a common space $\mathcal{W}$,
\begin{align}\label{Eq: donsker}
    \KLr(m_1\|m_2)=\sup_{f\in\mathfrak{F}} \bbE_{W\sim m_1}[f(W)] -\log{\big( \bbE_{W\sim m_2}[\exp{f(W)}]\big)},
\end{align}
where $f\in \mathfrak{F}=\{f:\mathcal{W}\rightarrow\mbR \text{ s.t.  } \bbE_{W\sim m_2}[e^{f(W)}]< \infty \}$.
Lemma~\ref{lem: gen rep} yields

\[
\begin{split}
        &\genb(\mrm^{\alpha}(\mun),\mu)\\
        &=\bbE_{\mbZn,\ols{Z}_1}\Big[\int_{0}^1\int_{\mathcal{W}} \partial_{\hat y}\ell\big(\Psi(\mrm_{\lambda}(\mun),\ols{\mathbf{x}}_1),\ols{\mathbf{y}}_1\big)\frac{\delta \Psi}{\delta \mrm}(\mrm_{\lambda}(\mun),\ols{\mathbf{x}}_1,w)(\mrm(\mun)-\mrm(\mur))(\mathrm{d}w) \mrd \lambda\Big]\\
        &=\bbE_{\mbZn,\ols{Z}_1}\Big[\int_{0}^1 \partial_{\hat y}\ell\big(\Psi(\mrm_{\lambda}(\mun),\ols{\mathbf{x}}_1),\ols{\mathbf{y}}_1\big)\Big(\bbE_{W\sim \mrm(\mun)}\Big[\frac{\delta \Psi}{\delta \mrm}(\mrm_{\lambda}(\mun),\ols{\mathbf{x}}_1,W)\Big]\\&\qquad\qquad-\bbE_{W\sim \mrm(\mur)}\Big[\frac{\delta \Psi}{\delta \mrm}(\mrm_{\lambda}(\mun),\ols{\mathbf{x}}_1,W)\Big]\Big) \mrd \lambda\Big].
    \end{split}
\]  
For the function $f(w)=\lambda_1\partial_{\hat y}\ell\big(\Psi(\mrm_{\lambda}(\mun),\ols{\mathbf{x}}_1),\ols{\mathbf{y}}_1\big)\frac{\delta \Psi}{\delta \mrm}(\mrm_{\lambda}(\mun),\ols{\mathbf{x}}_1,w)$,
due to Assumptions~\ref{Ass: on loss NN}, \ref{ass: bounded Unit function} and \ref{ass: readout function}, we have $\Big|\partial_{\hat y}\ell\big(\Psi(\mrm_{\lambda}(\mun),\ols{\mathbf{x}}_1),\ols{\mathbf{y}}_1\big)\frac{\delta \Psi}{\delta \mrm}(\mrm_{\lambda}(\mun),\ols{\mathbf{x}}_1,w)\Big|\leq M_{\ell'}L_{\psi}N_{\max}M_\phi$. Hence $f \in \mathfrak{F}$.

From \eqref{Eq: donsker}, for 
$m_1=\mrm(\mun)$ and $m_2=\mrm(\mur)$, we have for all $\lambda_1\in\mbR$
\begin{align}\label{eq: donsker bound}
     &\lambda_1\bbE_{W\sim \mrm(\mun)}\Big[\Big(\partial_{\hat y}\ell\big(\Psi(\mrm_{\lambda}(\mun),\ols{\mathbf{x}}_1)],\ols{\mathbf{y}}_1\big)\frac{\delta \Psi}{\delta \mrm}(\mrm_{\lambda}(\mun),\ols{\mathbf{x}}_1,W)\Big)\Big]\\\nonumber
     &\leq \KLr(\mrm(\mun)\|\mrm(\mur))+ \log\Big(\bbE_{W\sim\mrm^{\alpha}(\mur)}\Big[\exp\Big\{\lambda_1\Big(\partial_{\hat y}\ell\big(\Psi(\mrm_{\lambda}(\mun),\ols{\mathbf{x}}_1),\ols{\mathbf{y}}_1\big)\frac{\delta \Psi}{\delta \mrm}(\mrm_{\lambda}(\mun),\ols{\mathbf{x}}_1,W)\Big)\Big\}\Big]\Big).
\end{align}
With Lemma~\ref{lemma: hoeffding},   
\begin{equation}\label{Eq: subgaus bound}
\begin{split}
&\bbE_{W\sim\mrm^{\alpha}(\mur)}\Big[\exp\Big\{\lambda_1\Big(\partial_{\hat y}\ell\big(\Psi(\mrm_{\lambda}(\mun),\ols{\mathbf{x}}_1),\ols{\mathbf{y}}_1\big)\Big(\frac{\delta \Psi}{\delta \mrm}(\mrm_{\lambda}(\mun),\ols{\mathbf{x}}_1,W)-\bbE_{W\sim\mrm^{\alpha}(\mur)}\Big[\frac{\delta \Psi}{\delta \mrm}(\mrm_{\lambda}(\mun),\ols{\mathbf{x}}_1,W)\Big]\Big)\Big)\Big\}\Big]\\&\quad\leq \exp{\frac{\lambda_1^2\sigma_1^2}{2}},
\end{split}
\end{equation}
where $\sigma_1=M_{\ell'}L_{\psi}N_{\max}M_\phi$. Combining \eqref{Eq: subgaus bound} with \eqref{eq: donsker bound}, we can derive the following:
\begin{align*}
     &\lambda_1\Big(\partial_{\hat y}\ell\big(\Psi(\mrm_{\lambda}(\mun),\ols{\mathbf{x}}_1),\ols{\mathbf{y}}_1\big)\Big(\bbE_{W\sim \mrm(\mun)}\Big[\frac{\delta \Psi}{\delta \mrm}(\mrm_{\lambda}(\mun),\ols{\mathbf{x}}_1,W)\Big]-\bbE_{W\sim \mrm(\mur)}\Big[\frac{\delta \Psi}{\delta \mrm}(\mrm_{\lambda}(\mun),\ols{\mathbf{x}}_1,W)\Big]\Big)\Big)\\
     &\leq \KLr(\mrm(\mun)\|\mrm(\mur))+ \frac{\lambda_1^2\sigma_1^2}{2},
\end{align*}
 This is a nonnegative parabola in $\lambda_1$, whose discriminant must be nonpositive. Therefore, we have,
\begin{equation}\label{eq: bound 1}
\begin{split}
         &\Big|\partial_{\hat y}\ell\big(\Psi(\mrm_{\lambda}(\mun),\ols{\mathbf{x}}_1),\ols{\mathbf{y}}_1\big)\Big(\bbE_{W\sim \mrm(\mun)}\Big[\frac{\delta \Psi}{\delta \mrm}(\mrm_{\lambda}(\mun),\ols{\mathbf{x}}_1,W)\Big]-\bbE_{W\sim \mrm(\mur)}\Big[\frac{\delta \Psi}{\delta \mrm}(\mrm_{\lambda}(\mun),\ols{\mathbf{x}}_1,W)\Big]\Big)\Big|\\
     &\leq \sqrt{2\sigma_1^2\KLr(\mrm(\mun)\|\mrm(\mur))}.
     \end{split}
\end{equation}

The upper bound in \eqref{eq: bound 1} holds for all $\lambda\in[0,1]$. Therefore, combining with Lemma~\ref{lem: gen rep}, we have,
\begin{equation}
\begin{split}
         &\genb(\mrm^{\alpha}(\mun),\mu)
         \\&=
         \bbE_{\mbZn,\ols{Z}_1}\Big[\int_{0}^1 \partial_{\hat y}\ell\big(\Psi(\mrm_{\lambda}(\mun),\ols{\mathbf{x}}_1),\ols{\mathbf{y}}_1\big)\Big(\bbE_{W\sim \mrm(\mun)}\Big[\frac{\delta \Psi}{\delta \mrm}(\mrm_{\lambda}(\mun),\ols{\mathbf{x}}_1,W)\Big]\\&\qquad\qquad-\bbE_{W\sim \mrm(\mur)}\Big[\frac{\delta \Psi}{\delta \mrm}(\mrm_{\lambda}(\mun),\ols{\mathbf{x}}_1,W)\Big]\Big) \mrd \lambda\Big]\\
         \\&\leq
         \bbE_{\mbZn,\ols{Z}_1}\Bigg[\int_{0}^1 \Big|\partial_{\hat y}\ell\big(\Psi(\mrm_{\lambda}(\mun),\ols{\mathbf{x}}_1),\ols{\mathbf{y}}_1\big)\Big(\bbE_{W\sim \mrm(\mun)}\Big[\frac{\delta \Psi}{\delta \mrm}(\mrm_{\lambda}(\mun),\ols{\mathbf{x}}_1,W)\Big]\\&\qquad\qquad-\bbE_{W\sim \mrm(\mur)}\Big[\frac{\delta \Psi}{\delta \mrm}(\mrm_{\lambda}(\mun),\ols{\mathbf{x}}_1,W)\Big]\Big)\Big| \mrd \lambda\Bigg]\\
         &\leq\bbE_{\mbZn,\ols{Z}_1}\Big[ \sqrt{2\sigma_1^2\KLr(\mrm(\mun)\|\mrm(\mur))}\Big].
     \end{split}
\end{equation}
This 
completes the proof.
\end{proof}

\begin{proof}[\textbf{Proof of Proposition~\ref{Prop: KL bound} via Lemma~\ref{lem: tv} and Lemma~\ref{lem: pinsker}}]

From Lemma~\ref{lem: gen rep}, it yields,
\[
\begin{split}
        &\genb(\mrm^{\alpha}(\mun),\mu)\\
        &=\bbE_{\mbZn,\ols{Z}_1}\Big[\int_{0}^1\int_{\mathcal{W}} \partial_{\hat y}\ell\big(\Psi(\mrm_{\lambda}(\mun),\ols{\mathbf{x}}_1),\ols{\mathbf{y}}_1\big)\frac{\delta \Psi}{\delta \mrm}(\mrm_{\lambda}(\mun),\ols{\mathbf{x}}_1,w)(\mrm(\mun)-\mrm(\mur))(\mathrm{d}w) \mrd \lambda\Big]\\
        &=\bbE_{\mbZn,\ols{Z}_1}\Big[\int_{0}^1 \partial_{\hat y}\ell\big(\Psi(\mrm_{\lambda}(\mun),\ols{\mathbf{x}}_1),\ols{\mathbf{y}}_1\big)\Big(\bbE_{W\sim \mrm(\mun)}\Big[\frac{\delta \Psi}{\delta \mrm}(\mrm_{\lambda}(\mun),\ols{\mathbf{x}}_1,W)\Big]\\&\qquad\qquad-\bbE_{W\sim \mrm(\mur)}\Big[\frac{\delta \Psi}{\delta \mrm}(\mrm_{\lambda}(\mun),\ols{\mathbf{x}}_1,W)\Big]\Big) \mrd \lambda\Big].
    \end{split}
\]  
For the function $f(w)=\lambda_1\partial_{\hat y}\ell\big(\Psi(\mrm_{\lambda}(\mun),\ols{\mathbf{x}}_1),\ols{\mathbf{y}}_1\big)\frac{\delta \Psi}{\delta \mrm}(\mrm_{\lambda}(\mun),\ols{\mathbf{x}}_1,w)$,
due to Assumptions~\ref{Ass: on loss NN}, \ref{ass: bounded Unit function} and \ref{ass: readout function}, we have $\Big|\partial_{\hat y}\ell\big(\Psi(\mrm_{\lambda}(\mun),\ols{\mathbf{x}}_1),\ols{\mathbf{y}}_1\big)\frac{\delta \Psi}{\delta \mrm}(\mrm_{\lambda}(\mun),\ols{\mathbf{x}}_1,w)\Big|\leq M_{\ell'}L_{\psi}N_{\max}M_\phi$. Hence, from lemma~\ref{lem: tv}, it yields,
\begin{equation}\label{eq: tv bound}
    \big|\genb(\mrm^{\alpha}(\mun),\mu)\big|\leq M_{\ell'}L_{\psi}N_{\max}M_\phi \bbE_{\mbZn,\ols{Z}_1}\Big[\mathbb{TV}(\mrm(\mun),\mrm(\mur))\Big].
\end{equation}
Using Lemma~\ref{lem: pinsker} completes the proof.
\end{proof}

\begin{repproposition}{Prop: lower bound}\textbf{(restated)}
     Let Assumptions~\ref{Ass: on loss NN} hold. Then, the following lower bound holds on the generalization error of the Gibbs measure $m^\alpha(\mun)$,
    \[\begin{split}
    &\genb(\mrm^{\alpha}(\mun),\mu)\geq \frac{n}{2\alpha} \bbE_{\mathbf{Z}_n,\ols{Z}_1}\Big[ \KLr_{\mathrm{sym}}\big(\mrm^{\alpha}(\mun)\|\mrm^{\alpha}(\mur)\big)\Big].
    \end{split}
\]
\end{repproposition}

\begin{proof}
For simplicity of proof, we abbreviate
\begin{align}
    \tilde{\ell}\big(m,z\big)&:=\ell\big(\Psi(m,x),y\big),\\\label{eq: chain}
    \frac{\delta \tilde{\ell}}{\delta m}\big(m,z,w\big)&:=\partial_{\hat{y}}\ell(\Psi(m,x),y)\frac{\delta \Psi}{\delta m}(m,x,w),
\end{align}
where $z=(x,y)$ and \eqref{eq: chain} follows from chain rule. From Assumption~\ref{Ass: on loss NN}, where the loss function is convex with respect to $\Psi(m,x)$ and due to the fact that $\Psi(m,x)$ is linear with respect to parameter measure $m$, then we have the convexity of $\tilde{\ell}\big(m,z\big)$ with respect to parameter measure. Recall that from \eqref{Eq: Gibbs Measure} we have 
\[\mrm^{\alpha}(\mun)=\frac{\pi}{S_{\alpha,\pi}(\mun)}\exp\Bigg\{-\alpha\Bigg[\frac{\delta \mathrm{R}(\mrm^\alpha, \mun,w)}{\delta m} \Bigg] \Bigg\},\]
and
\[\mrm^{\alpha}(\mur)=\frac{\pi}{S_{\alpha,\pi}(\mur)}\exp\Bigg\{-\alpha\Bigg[\frac{\delta \mathrm{R}(\mrm^\alpha, \mur,w)}{\delta m} \Bigg] \Bigg\}.\]
We need to compute the expectation of 
$\KLr_{\mathrm{sym}}\big(\mrm^{\alpha}(\mur)\|\mrm^{\alpha}(\mun)\big)$ where $\KLr_{\mathrm{sym}}(p\|q)=\KLr(p\|q)+\KLr(q\|p)$. For that purpose,
    \begin{align*}
        &\bbE_{\mathbf{Z}_n,\ols{Z}_1}\Big[\KLr\big(\mrm^{\alpha}(\mun)\|\mrm^{\alpha}(\mur)\big)+\KLr\big(\mrm^{\alpha}(\mur)\|\mrm^{\alpha}(\mun)\big)\Big]\\
        &=\bbE_{\mathbf{Z}_n,\ols{Z}_1}\Big[\int_{\mathcal{W}}\log(\mrm^{\alpha}(\mun)/\mrm^{\alpha}(\mur))(\mrm^{\alpha}(\mun)-\mrm^{\alpha}(\mur))(\mrd w)\Big]\\
        &= \bbE_{\mathbf{Z}_n,\ols{Z}_1}\Big[\int_{\mathcal{W}}\log(S_{\alpha,\pi}(\mur)/S_{\alpha,\pi}(\mun)) (\mrm^{\alpha}(\mun)-\mrm^{\alpha}(\mur))(\mrd w)\Big]\\
        &\quad + \bbE_{\mathbf{Z}_n,\ols{Z}_1}\Big[\bbE_{W\sim\mrm^{\alpha}(\mur)}\Big[\pi(W)\Big]-\bbE_{W\sim\mrm^{\alpha}(\mun)}\Big[\pi(W)\Big]\Big]\\
        &\quad +\alpha\Big(\bbE_{\mathbf{Z}_n,\ols{Z}_1}\Big[\bbE_{W\sim\mrm^{\alpha}(\mun)}\Big[\frac{\delta \mathrm{R}}{\delta m}(\mrm^{\alpha}(\mur), \mur,W)-\frac{\delta \mathrm{R}}{\delta m}(\mrm^{\alpha}(\mun), \mun,W)\Big]\Big]\\
        &\quad- \bbE_{\mathbf{Z}_n,\ols{Z}_1}\Big[ \bbE_{W\sim\mrm^{\alpha}(\mur)}\Big[\frac{\delta \mathrm{R}}{\delta m}(\mrm^{\alpha}(\mur), \mur,W)-\frac{\delta \mathrm{R}}{\delta m}(\mrm^{\alpha}(\mun), \mun,W)\Big]\Big]\Big).
    \end{align*}
    Let us define the following terms,
    \begin{equation}
        \begin{split}
           I_1&:=\bbE_{\mathbf{Z}_n,\ols{Z}_1}\Big[\int_{\mathcal{W}}\log(S_{\alpha,\pi}(\mur)/S_{\alpha,\pi}(\mun)) (\mrm^{\alpha}(\mun)-\mrm^{\alpha}(\mur))(\mrd w)\Big],\\
         I_2&:=\bbE_{\mathbf{Z}_n,\ols{Z}_1}\Big[\bbE_{W\sim\mrm^{\alpha}(\mur)}\Big[\pi(W)\Big]-\bbE_{W\sim\mrm^{\alpha}(\mun)}\Big[\pi(W)\Big]\Big],\\
        I_3&:=\bbE_{\mathbf{Z}_n,\ols{Z}_1}\Big[\bbE_{W\sim\mrm^{\alpha}(\mun)}\Big[\frac{\delta \mathrm{R}}{\delta m}(\mrm^{\alpha}(\mur), \mur,W)\Big]-\bbE_{W\sim\mrm^{\alpha}(\mur)}\Big[\frac{\delta \mathrm{R}}{\delta m}(\mrm^{\alpha}(\mur), \mur,W)\Big]\Big],\\
        I_4&:= \bbE_{\mathbf{Z}_n,\ols{Z}_1}\Big[\bbE_{W\sim\mrm^{\alpha}(\mur)}\Big[\frac{\delta \mathrm{R}}{\delta m}(\mrm^{\alpha}(\mun), \mun,W)\Big]-\bbE_{W\sim\mrm^{\alpha}(\mun)}\Big[\frac{\delta \mathrm{R}}{\delta m}(\mrm^{\alpha}(\mun), \mun,W)\Big]\Big].
        \end{split}
    \end{equation}
   
    Note that $\log(S_{\alpha,\pi}(\mur)/S_{\alpha,\pi}(\mun))$ is not a function of parameters. Therefore, we have,
 \begin{equation}\label{eq: I1}
 \begin{split}
 I_1&=\bbE_{\mathbf{Z}_n,\ols{Z}_1}\Big[\int_{\mathcal{W}}\log(S_{\alpha,\pi}(\mur)/S_{\alpha,\pi}(\mun)) (\mrm^{\alpha}(\mun)-\mrm^{\alpha}(\mur))(\mrd w)\Big]\\
 &=0.
 \end{split}
 \end{equation}
 Also $\pi(W)$ is not a function of data samples, therefore, we have
\begin{equation}\label{eq: I2}
\begin{split}
I_2&=\bbE_{\mathbf{Z}_n,\ols{Z}_1}\Big[\bbE_{W\sim\mrm^{\alpha}(\mur)}\Big[\pi(W)\Big]-\bbE_{W\sim\mrm^{\alpha}(\mun)}\Big[\pi(W)\Big]\Big]\\
&=0.
\end{split}
\end{equation}
By considering,
\begin{equation}\begin{split}\label{eq: R2}&\frac{\delta \mathrm{R}}{\delta m}(\mrm^{\alpha}(\mun), \mun,w)=\frac{1}{n}\frac{\delta\tilde{\ell}}{\delta m}(\mrm^{\alpha}(\mun),Z_1,w)+\frac{1}{n}\sum_{i=2}^n \frac{\delta\tilde{\ell}}{\delta m}(\mrm^{\alpha}(\mun),Z_i,w) ,
   \end{split}\end{equation}
    \begin{equation}\begin{split}\label{eq: R1}&\frac{\delta \mathrm{R}}{\delta m}(\mrm^{\alpha}(\mur), \mur,w)=\frac{1}{n}\frac{\delta\tilde{\ell}}{\delta m}(\mrm^{\alpha}(\mur),\ols{Z}_1,w)+\frac{1}{n}\sum_{i=2}^n \frac{\delta\tilde{\ell}}{\delta m}(\mrm^{\alpha}(\mur),Z_i,w),\end{split}\end{equation}
then, we have,
 \begin{equation}\begin{split}\label{eq: R6}
I_3&=\bbE_{\mathbf{Z}_n,\ols{Z}_1}\Big[\bbE_{W\sim\mrm^{\alpha}(\mun)}\Big[\frac{\delta \mathrm{R}}{\delta m}(\mrm^{\alpha}(\mur), \mur,W)\Big]-\bbE_{W\sim\mrm^{\alpha}(\mur)}\Big[\frac{\delta \mathrm{R}}{\delta m}(\mrm^{\alpha}(\mur), \mur,W)\Big]\Big]\\&=
\frac{1}{n}\bbE_{\mathbf{Z}_n,\ols{Z}_1}\Big[\bbE_{W\sim\mrm^{\alpha}(\mun)}\Big[\frac{\delta\tilde{\ell}}{\delta m}(\mrm^{\alpha}(\mur),\ols{Z}_1,W)\Big]-\bbE_{W\sim\mrm^{\alpha}(\mur)}\Big[\frac{\delta\tilde{\ell}}{\delta m}(\mrm^{\alpha}(\mur),\ols{Z}_1,W)\Big]\Big]\\&\quad+\frac{1}{n}\sum_{i=2}^n \bbE_{\mathbf{Z}_n,\ols{Z}_1}\Big[\bbE_{W\sim\mrm^{\alpha}(\mun)}\Big[\frac{\delta\tilde{\ell}}{\delta m}(\mrm^{\alpha}(\mur),Z_i,W)\Big]-\bbE_{W\sim\mrm^{\alpha}(\mur)}\Big[\frac{\delta\tilde{\ell}}{\delta m}(\mrm^{\alpha}(\mur),Z_i,W)\Big]\Big].
\end{split}\end{equation}

\begin{equation}\begin{split}\label{eq: R7}
I_4&=\bbE_{\mathbf{Z}_n,\ols{Z}_1}\Big[\bbE_{W\sim\mrm^{\alpha}(\mur)}\Big[\frac{\delta \mathrm{R}}{\delta m}(\mrm^{\alpha}(\mun), \mun,W)\Big]-\bbE_{W\sim\mrm^{\alpha}(\mun)}\Big[\frac{\delta \mathrm{R}}{\delta m}(\mrm^{\alpha}(\mun), \mun,W)\Big]\Big]\\&=
\frac{1}{n}\bbE_{\mathbf{Z}_n,\ols{Z}_1}\Big[\bbE_{W\sim\mrm^{\alpha}(\mur)}\Big[\frac{\delta\tilde{\ell}}{\delta m}(\mrm^{\alpha}(\mun),Z_1,W)\Big]-\bbE_{W\sim\mrm^{\alpha}(\mur)}\Big[\frac{\delta\tilde{\ell}}{\delta m}(\mrm^{\alpha}(\mun),Z_1,W)\Big]\Big]\\&\quad+\frac{1}{n}\sum_{i=2}^n \bbE_{\mathbf{Z}_n,\ols{Z}_1}\Big[\bbE_{W\sim\mrm^{\alpha}(\mur)}\Big[\frac{\delta\tilde{\ell}}{\delta m}(\mrm^{\alpha}(\mun),Z_i,W)\Big]-\bbE_{W\sim\mrm^{\alpha}(\mun)}\Big[\frac{\delta\tilde{\ell}}{\delta m}(\mrm^{\alpha}(\mun),Z_i,W)\Big]\Big].
\end{split}\end{equation}

 For $2\leq j\leq n$ due to the fact that data samples are i.i.d, we have:
    \begin{equation}\label{eq: equal 1}
    \bbE_{\mathbf{Z}_n,\ols{Z}_1}[\tilde{\ell}(\mrm(\mun), Z_j)]= \bbE_{\mathbf{Z}_n,\ols{Z}_1}[\tilde{\ell}(\mrm(\mur), Z_j)].
\end{equation}
Via the convexity of the loss function, $\tilde{\ell}\big(m,z\big)$, with respect to the parameter measure, $m$, from Lemma~\ref{lem: convex lower}, it holds that,
  \begin{equation}\label{eq: equal 22}
  \begin{split}
     &  \bbE_{\mathbf{Z}_n,\ols{Z}_1}[\tilde{\ell}(\mrm(\mun), Z_j)-\tilde{\ell}(\mrm(\mur), Z_j)]\\
   & \geq \bbE_{\mathbf{Z}_n,\ols{Z}_1}\Big[ \int \frac{\delta \tilde{\ell}}{\delta m}(\mrm(\mur), Z_j,w)(\mrm(\mun)-\mrm(\mur))(\mrd w)\Big].
      \end{split}
\end{equation}
Therefore, we obtain for $2\leq j\leq n$,
\begin{equation}\label{eq: equal 3}
  \begin{split}
   0\geq \bbE_{\mathbf{Z}_n,\ols{Z}_1}\Big[ \int \frac{\delta \tilde{\ell}}{\delta m}(\mrm(\mur), Z_j,w)(\mrm(\mun)-\mrm(\mur))(\mrd w)\Big].
      \end{split}
\end{equation}
Similarly, we can show that, for $2\leq j\leq n$,
 \begin{equation}\label{eq: equal 2}
  \begin{split}
   0\geq \bbE_{\mathbf{Z}_n,\ols{Z}_1}\Big[ \int \frac{\delta \tilde{\ell}}{\delta m}(\mrm(\mun), Z_j,w)(\mrm(\mur)-\mrm(\mun))(\mrd w)\Big].
      \end{split}
\end{equation}

Combining \eqref{eq: equal 3} and \eqref{eq: equal 2} with \eqref{eq: R6} and \eqref{eq: R7}, the following holds: 
  \begin{equation}\label{eq: I3}
      \begin{split}
          I_3\leq \frac{1}{n}\bbE_{\mathbf{Z}_n,\ols{Z}_1}\Big[\bbE_{W\sim\mrm^{\alpha}(\mun)}\Big[\frac{\delta\tilde{\ell}}{\delta m}(\mrm^{\alpha}(\mur),\ols{Z}_1,W)\Big]-\bbE_{W\sim\mrm^{\alpha}(\mur)}\Big[\frac{\delta\tilde{\ell}}{\delta m}(\mrm^{\alpha}(\mur),\ols{Z}_1,W)\Big]\Big],
      \end{split}
  \end{equation}
  and 
  \begin{equation}\label{eq: I4}
      \begin{split}
          I_4\leq \frac{1}{n}\bbE_{\mathbf{Z}_n,\ols{Z}_1}\Big[\bbE_{W\sim\mrm^{\alpha}(\mur)}\Big[\frac{\delta\tilde{\ell}}{\delta m}(\mrm^{\alpha}(\mun),Z_1,W)\Big]-\bbE_{W\sim\mrm^{\alpha}(\mun)}\Big[\frac{\delta\tilde{\ell}}{\delta m}(\mrm^{\alpha}(\mun),Z_1,W)\Big]\Big].
      \end{split}
  \end{equation}
Therefore, from \eqref{eq: I1}, \eqref{eq: I2}, \eqref{eq: I3}, and \eqref{eq: I4} we have,
\begin{equation}\label{eq: last 3}
    \begin{split}
        &\bbE_{\mathbf{Z}_n,\ols{Z}_1}\Big[\KLr\big(\mrm^{\alpha}(\mun)\|\mrm^{\alpha}(\mur)\big)+\KLr\big(\mrm^{\alpha}(\mur)\|\mrm^{\alpha}(\mun)\big)\Big]\\&
        =I_1+I_2+\alpha(I_3+I_4)\\
        &\leq\frac{\alpha}{n}\bbE_{\mathbf{Z}_n,\ols{Z}_1}\Big[\bbE_{W\sim\mrm^{\alpha}(\mun)}\Big[\frac{\delta\tilde{\ell}}{\delta m}(\mrm^{\alpha}(\mur),\ols{Z}_1,W)\Big]-\bbE_{W\sim\mrm^{\alpha}(\mur)}\Big[\frac{\delta\tilde{\ell}}{\delta m}(\mrm^{\alpha}(\mur),\ols{Z}_1,W)\Big]\Big] \\&
        +\frac{\alpha}{n}\bbE_{\mathbf{Z}_n,\ols{Z}_1}\Big[\bbE_{W\sim\mrm^{\alpha}(\mur)}\Big[\frac{\delta\tilde{\ell}}{\delta m}(\mrm^{\alpha}(\mun),Z_1,W)\Big]-\bbE_{W\sim\mrm^{\alpha}(\mun)}\Big[\frac{\delta\tilde{\ell}}{\delta m}(\mrm^{\alpha}(\mun),Z_1,W)\Big]\Big].
    \end{split}
\end{equation}
   From Lemma~\ref{lem: convex lower}, we have:
   \begin{equation}\label{eq: last 1}
   \begin{split}
    &\bbE_{\mathbf{Z}_n,\ols{Z}_1}\Big[\bbE_{W\sim\mrm^{\alpha}(\mun)}\Big[\frac{\delta\tilde{\ell}}{\delta m}(\mrm^{\alpha}(\mur),\ols{Z}_1,W)\Big]-\bbE_{W\sim\mrm^{\alpha}(\mur)}\Big[\frac{\delta\tilde{\ell}}{\delta m}(\mrm^{\alpha}(\mur),\ols{Z}_1,W)\Big]\Big]\\
    &\leq \genb(\mrm^\alpha(\mun),\mu).
    \end{split}
\end{equation}
Similarly from Remark~\ref{remark: lemma conv}, we have,
       \begin{equation}\label{eq: last 2}
       \begin{split}
    &\bbE_{\mathbf{Z}_n,\ols{Z}_1}\Big[\bbE_{W\sim\mrm^{\alpha}(\mur)}\Big[\frac{\delta\tilde{\ell}}{\delta m}(\mrm^{\alpha}(\mun),Z_1,W)\Big]-\bbE_{W\sim\mrm^{\alpha}(\mun)}\Big[\frac{\delta\tilde{\ell}}{\delta m}(\mrm^{\alpha}(\mun),Z_1,W)\Big]\Big]\\
    &\leq \genb(\mrm^{\alpha}(\mun),\mu).
    \end{split}
\end{equation}
The final results holds by combining \eqref{eq: last 1} and \eqref{eq: last 2} with \eqref{eq: last 3}.
\end{proof}

\begin{reptheorem}{thm: main unit upper}\textbf{(restated)}
     Let Assumptions~\ref{Ass: on loss NN}, \ref{ass: bounded Unit function} and \ref{ass: readout function} hold. Then, the following upper bound holds on the generalization error of the Gibbs measure, i.e., $\mrm^\alpha(\mun)$,
\begin{equation*}
 \genb(\mrm^\alpha(\mun),\mu) 
     \leq  
  \frac{\alpha C}{n},
\end{equation*}
where $C=(M_{\ell'}M_{\phi}L_{\psi}N_{\max})^2.$
\end{reptheorem}
\begin{proof}
\begin{align}\nonumber
    &\frac{n}{2\alpha} \bbE_{\mathbf{Z}_n,\ols{Z}_1}\Big[ \KLr\big(\mrm^{\alpha}(\mun)\|\mrm^{\alpha}(\mur)\big)\Big]\\\label{eq:1}
    &\leq\frac{n}{2\alpha} \bbE_{\mathbf{Z}_n,\ols{Z}_1}\Big[ \KLr_{\mathrm{sym}}\big(\mrm^{\alpha}(\mun)\|\mrm^{\alpha}(\mur)\big)\Big]\\\label{eq:2}
    &\leq\genb(\mrm^{\alpha}(\mun),\mu)\\\label{eq:3}
    &\leq \big(M_{\ell'}L_{\psi}N_{\max}M_\phi/\sqrt{2} \big)\bbE_{\mbZn,\ols{Z}_1}\Big[ \sqrt{ \KLr\big(\mrm(\mun)\|\mrm(\mur)\big)}\Big]\\\label{eq:4}
    &\leq \big(M_{\ell'}L_{\psi}N_{\max}M_\phi/\sqrt{2} \big) \sqrt{ \bbE_{\mbZn,\ols{Z}_1}\big[\KLr\big(\mrm(\mun)\|\mrm(\mur)\big)\big]},
\end{align}
    where \eqref{eq:1}, \eqref{eq:2}, \eqref{eq:3}, \eqref{eq:4}, follow from the fact that $\KLr(\cdot\|\cdot)\leq \KLr_{\mathrm{sym}}(\cdot\|\cdot)$, Proposition~\ref{Prop: lower bound}, Proposition~\ref{Prop: KL bound} and Jensen-Inequality, respectively. Therefore, we have,
    \begin{align}\label{eq: final thm}
    &\sqrt{ \bbE_{\mbZn,\ols{Z}_1}\big[\KLr\big(\mrm(\mun)\|\mrm(\mur)\big)\big]}\leq\Big(\frac{\sqrt{2}\alpha M_{\ell'}L_{\psi}N_{\max}M_\phi}{n} \Big).
\end{align}
The final result follows from combining \eqref{eq: final thm} with the following inequality,
\[\genb(\mrm^{\alpha}(\mun),\mu)\leq \big(M_{\ell'}L_{\psi}N_{\max}M_\phi/\sqrt{2} \big) \sqrt{ \bbE_{\mbZn,\ols{Z}_1}\big[\KLr\big(\mrm(\mun)\|\mrm(\mur)\big)\big]}.\]
\end{proof}
\subsection{GCN}
\begin{replemma}{lem: GCN case}\textbf{(restated)}
    Let Assumptions~\ref{Ass: Feature node bounded} and \ref{ass: bounded activation function} hold. For a graph sample, $(\mbA_q,\mbF_q)$ with $N$ nodes, and a graph filter $G(\cdot)$, the following upper bound holds on summation of GCN neuron units over all nodes:
     \begin{equation*}
     \begin{split}
        &\sum_{j=1}^N|\phi_c(W_c,G(\mbA_q)[j,:]\mbF_q)|\leq N w_{2,c}L_{\varphi}\norm{W_{1,c}}_2B_f\min(\norm{G(\mbA_q)}_{\infty},\norm{G(\mbA_q)}_F).
        \end{split}
    \end{equation*}
\end{replemma}
\begin{proof}
 Recall that $\norm{G(\mbA_q)}_{\infty}=\max_{j} \sum_{i=1}^N |G(\mbA_q)[j,i]|$, $\norm{G(\mbA_q)}_{F}=\sqrt{\sum_{j=1}^N \sum_{i=1}^N (G(\mbA_q)[j,i])^2}$ and $\norm{W_{1,c,m}}_2=\sup_{W_{1,c}\in\mathcal{S}^k}\norm{W_{1,c}}_2$. Then, we have,
 \begin{equation}
 \begin{split}
        \sum_{j=1}^N|\phi_c(W_c,G(\mbA_q)[j,:]\mbF_q)|)|&\leq  \sum_{j=1}^N \sup_{W_{2,c} \in \mathcal{S}}|W_{2,c}||\varphi(G(\mbA_q)[j,:]\mbF_q).W_{1,c})| \\
        &\leq  \sum_{j=1}^N \sup_{W_{1,c} \in \mathcal{S}^k} w_{2,c}|\varphi(G(\mbA_q)[j,:]\mbF_q).W_{1,c,m})|\\
        &\leq N w_{2,c} L_{\varphi}\norm{W_{1,c,m}}_2 \max_{j}\norm{G(\mbA_q)[j,:]\mbF_q}_2 \\
        &\leq  N w_{2,c}L_{\varphi} \norm{W_{1,c,m}}_2 \max_j\norm{\sum_{i=1}^N G(\mbA_q)[j,i]\mbF_q[:,i]}_2\\
        &\leq N w_{2,c}L_{\varphi} \norm{W_{1,c,m}}_2 \max_j \Big(\sum_{i=1}^N |G(\mbA_q)[j,i]|\norm{\mbF_q[:,i]}_2\Big)\\
        &\leq N w_{2,c}L_{\varphi} \norm{W_{1,c,m}}_2 B_f \max_j\Big(\sum_{i=1}^N |G(\mbA_q)[j,i]|\Big)\\
&\leq N w_{2,c}L_{\varphi} \norm{W_{1,c,m}}_2 B_f\norm{G(\mbA_q)}_{\infty}.
    \end{split}
    \end{equation}
   
    We also have,
     \begin{equation}
     \begin{split}
        \sum_{j=1}^N|\phi_c(W_c,G(\mbA_q)[j,:]\mbF_q)|)|&\leq  \sum_{j=1}^N \sup_{W_{2,c} \in \mathcal{S}}|W_{2,c}||\varphi(G(\mbA_q)[j,:]\mbF_q).W_{1,c})| \\
        &\leq  \sum_{j=1}^N \sup_{W_{1,c} \in \mathcal{S}^k} w_{2,c}|\varphi(G(\mbA_q)[j,:]\mbF_q).W_{1,c,m})|\\
        &\leq  w_{2,c} L_{\varphi}\norm{W_{1,c,m}}_2 \sum_{j=1}^N \norm{G(\mbA_q)[j,:]\mbF_q}_2 \\
        &\leq   w_{2,c}L_{\varphi} \norm{W_{1,c,m}}_2 \sum_{j=1}^N \norm{\sum_{i=1}^N G(\mbA_q)[j,i]\mbF_q[:,i]}_2\\
        &\leq  w_{2,c}L_{\varphi} \norm{W_{1,c,m}}_2 \sum_{j=1}^N \sum_{i=1}^N |G(\mbA_q)[j,i]|\norm{\mbF_q[:,i]}_2\\
        &\leq  w_{2,c}L_{\varphi} \norm{W_{1,c,m}}_2 B_f \sum_{j=1}^N\Big(\sum_{i=1}^N |G(\mbA_q)[j,i]|\Big)\\
&\leq N w_{2,c}L_{\varphi} \norm{W_{1,c,m}}_2 B_f \norm{G(\mbA_q)}_{F}.
    \end{split}
    \end{equation}
    This completes the proof.
\end{proof}

\begin{repproposition}{prop: GCN result}\textbf{(restated)}
In a GCN with mean-readout function, under the combined assumptions for Theorem~\ref{thm: main unit upper} and Lemma~\ref{lem: GCN case}, the following upper bound holds on the generalization error of the Gibbs measure  $\mrm^{\alpha,c}(\mun)$, 
\begin{equation*}
 \genb(\mrm^{\alpha,c}(\mun),\mu) 
     \leq  
  \frac{\alpha M_{c}^2 M_{\ell'}^2G_{\max}^2}{n}\,,
\end{equation*}
where $M_{c}=w_{2,c}L_{\varphi}\norm{W_{1,c}}_2 B_f$, $G_{\max}=\min\big(\norm{G(\mathcal{A})}_{\infty}^{\max},\norm{G(\mathcal{A})}_{F}^{\max}\big)$, $\norm{G(\mathcal{A})}_{\infty}^{\max}= \max_{\mbA_q\in \mathcal{A},\mu(f_q,\mbA_q)>0}\norm{G(\mbA_q)}_{\infty}$ and $\norm{G(\mathcal{A})}_{F}^{\max}= \max_{\mbA_q\in \mathcal{A},\mu(f_q,\mbA_q)>0}\norm{G(\mbA_q)}_{F}$.
\end{repproposition}
\begin{proof}
The result follows directly by combining Lemma~\ref{lem: GCN case} with Theorem~\ref{thm: main unit upper} and assuming the mean-readout function. 
\end{proof}
For sum-readout, we can modify the result as follows,
\begin{corollary}[GCN and sum-readout]\label{Cor: sum-read and GCN}
In a GCN with the sum-readout function, under the combined assumptions for Theorem~\ref{thm: main unit upper} and Lemma~\ref{lem: GCN case}, the following upper bound holds on the generalization error of the Gibbs measure  $\mrm^{\alpha,c}(\mun)$, 
\begin{equation*}
 \genb(\mrm^{\alpha,c}(\mun),\mu) 
     \leq  
  \frac{\alpha M_{c}^2 M_{\ell'}^2G_{\max}^2}{n}\,,
\end{equation*}
where $M_{c}=w_{2,c} N_{\max}L_{\varphi}\norm{W_{1,c}}_2 B_f$, $G_{\max}=\min\big(\norm{G(\mathcal{A})}_{\infty}^{\max},\norm{G(\mathcal{A})}_{F}^{\max}\big)$, $\norm{G(\mathcal{A})}_{\infty}^{\max}= \max_{\mbA_q\in \mathcal{A},\mu(f_q,\mbA_q)>0}\norm{G(\mbA_q)}_{\infty}$ and $\norm{G(\mathcal{A})}_{F}^{\max}= \max_{\mbA_q\in \mathcal{A},\mu(f_q,\mbA_q)>0}\norm{G(\mbA_q)}_{F}$.
\end{corollary}
 We can present a modified version of Proposition~\ref{prop: GCN result} that accommodates a bounded activation function. It is important to note that our analysis assumes the activation function to be Lipschitz continuous, a condition that holds for the popular Tanh function. However, for scenarios where the activation function is bounded, i.e., $$\sup_{\mbA,\mbF}|\varphi(G(\mbA)[j,:]\mbF)|\leq M_{\varphi}\,$$ we can propose an updated upper bound in Proposition~\ref{prop: GCN result} as follows,
\begin{corollary}\label{Cor: GCN}
    Let us assume the same assumptions in Proposition~\ref{prop: GCN result} and a bounded activation function, i.e., $\sup_{\mbA,\mbF}|\varphi(G(\mbA)[j,:]\mbF)|\leq M_{\varphi}\,$ for $j\in[N]$ in GCN. Then the following upper bound holds on the generalization error of the Gibbs measure, $\mrm^{\alpha,c}(\mun)$,  
\begin{equation*}
 \genb(\mrm^{\alpha,c}(\mun),\mu) 
     \leq  
  \frac{\alpha M_{c}^2 M_{\ell'}^2}{n}\,,
\end{equation*}
where $M_{c}=w_{2,c} \min\big(M_{\varphi},L_{\varphi}\norm{W_{1,c}}_2 G_{\max} B_f\big)$, $G_{\max}=\min\big(\norm{G(\mathcal{A})}_{\infty}^{\max},\norm{G(\mathcal{A})}_{F}^{\max}\big)$, $\norm{G(\mathcal{A})}_{\infty}^{\max}= \max_{\mbA_q\in \mathcal{A},\mu(f_q,\mbA_q)>0}\norm{G(\mbA_q)}_{\infty}$ and $\norm{G(\mathcal{A})}_{F}^{\max}= \max_{\mbA_q\in \mathcal{A},\mu(f_q,\mbA_q)>0}\norm{G(\mbA_q)}_{F}$.
\end{corollary}
\subsection{MPGNN}
\begin{replemma}{lem: bound unit mpgnn}\textbf{(restated)}
Let Assumptions~\ref{Ass: Feature node bounded} and \ref{ass: bound unit MPGNN} hold. For a graph sample $(\mbA_q,\mbF_q)$ and a graph filter $G(\cdot)$, the following upper bound holds on the summation of MPU units over all nodes:
     \begin{equation}
       \sum_{j=1}^N |\phi_p(W_p,G(\mbA_q)[j,:]\mbF_q)|\leq w_{2,p} L_{\kappa}B_f(\norm{W_{3,p,m}}_2 +L_\rho L_{\zeta}G_{\max}\norm{W_{1,p,m}}_2).
    \end{equation}
\end{replemma}
\begin{proof}
 Recall that $\norm{G(\mbA_q)}_{\infty}=\max_{j} \sum_{i=1}^N |G(\mbA_q)[j,i]|$, $\norm{G(\mbA_q)}_{F}=\sqrt{\sum_{j=1}^N \sum_{i=1}^N (G(\mbA_q)[j,i])^2}$, $\norm{W_{1,p,m}}_2=\sup_{W_{1,c}\in\mathcal{S}^k}\norm{W_{1,p}}_2$ and $\norm{W_{3,p,m}}_2=\sup_{W_{1,c}\in\mathcal{S}^k}\norm{W_{3,p}}_2$. Then, we have,
    \begin{align*}
        \sum_{j=1}^N|\phi_p(W_p,G(\mbA_q)[j,:]\mbF_q)|&\leq
        \sup_{
        W_{2,p}\in \mathcal{S}}|W_{2,p}|L_{\kappa}\Big(\norm{F_q[j,:]W_{3,p}+\rho(G(\mbA_q)[j,:]\zeta(\mbF_q))W_{1,p}}_2\Big)\\
        &\leq \sum_{j=1}^N\sup_{\substack{W_{3,p},W_{1,p}\in \mathcal{S}^k}} w_{2,p} L_{\kappa}\Big(\norm{F_q[j,:]W_{3,p}}_2+\norm{\rho(G(\mbA_q)[j,:]\zeta(\mbF_q))W_{1,p}}_2\Big)\\
        &\leq  \sum_{j=1}^N w_{2,p} L_{\kappa}\Big(\norm{F_q[j,:]}_2\norm{W_{3,p,m}}_2+\norm{\rho(G(\mbA_q)[j,:]\zeta(\mbF_q))}_2\norm{W_{1,p,m}}_2\Big)
        \\&\leq  \sum_{j=1}^N w_{2,p} L_{\kappa}\Big(\norm{F_q[j,:]}_2\norm{W_{3,p,m}}_2+L_{\rho}L_{\zeta}\norm{G(\mbA_q)[j,:]\mbF_q}_2\norm{W_{1,p,m}}_2\Big)
        \\
        &\leq N w_{2,p} L_{\kappa}B_f(\norm{W_{3,p,m}}_2 +L_\rho L_{\zeta}\norm{G(\mbA_q)}_{\infty}\norm{W_{1,p,m}}_2), 
    \end{align*}
    Similar to Lemma~\ref{lem: GCN case}, we can show that,
        \begin{align*}
        \sum_{j=1}^N|\phi_p(W_p,G(\mbA_q)[j,:]\mbF_q)|&\leq
        N w_{2,p} L_{\kappa}B_f(\norm{W_{3,p,m}}_2 +L_\rho L_{\zeta}\norm{G(\mbA_q)}_{F}\norm{W_{1,p,m}}_2).
    \end{align*}

\end{proof}

\begin{repproposition}{Prop: MPGNN bound}\textbf{(restated)}
 In an MPGNN with the mean-readout function, under the combined assumptions for Theorem~\ref{thm: main unit upper} and Lemma~\ref{lem: bound unit mpgnn}, the following upper bound holds on the generalization error of the Gibbs measure $\mrm^{\alpha,p}(\mun)$,
\begin{equation*}
 \genb(\mrm^{\alpha,p}(\mun),\mu) 
     \leq \frac{\alpha M_{p}^2 M_{\ell'}^2}{n} \,,
\end{equation*}
with $M_{p}=w_{2,p} L_{\kappa}B_f(\norm{W_{3,p}}_2 +G_{\max} L_\rho L_{\zeta}\norm{W_{1,p}}_2)$, $G_{\max}=\min\big(\norm{G(\mathcal{A})}_{\infty}^{\max},\norm{G(\mathcal{A})}_{F}^{\max}\big)$, $\norm{G(\mathcal{A})}_{\infty}^{\max}= \max_{\mbA_q\in \mathcal{A},\mu(f_q,\mbA_q)>0}\norm{G(\mbA_q)}_{\infty}$ and $\norm{G(\mathcal{A})}_{F}^{\max}= \max_{\mbA_q\in \mathcal{A},\mu(f_q,\mbA_q)>0}\norm{G(\mbA_q)}_{F}$.
\end{repproposition}
\begin{proof}
The result follows directly by combining Lemma~\ref{lem: bound unit mpgnn} with Theorem~\ref{thm: main unit upper} and assuming the mean-readout function. \end{proof}
For the sum-readout function, similar to Corollary~\ref{Cor: sum-read and GCN}, we have,
\begin{corollary}[MPGNN and sum-readout]\label{Cor: MPGNN and sum}
 In an MPGNN with the sum-readout function, under the combined assumptions for Theorem~\ref{thm: main unit upper} and Lemma~\ref{lem: bound unit mpgnn}, the following upper bound holds on the generalization error of the Gibbs measure $\mrm^{\alpha,p}(\mun)$,
\begin{equation*}
 \genb(\mrm^{\alpha,p}(\mun),\mu) 
     \leq \frac{\alpha M_{p}^2 M_{\ell'}^2}{n} \,.
\end{equation*}
with $M_{p}=w_{2,p} N_{\max} L_{\kappa}B_f(\norm{W_{3,p}}_2 +G_{\max} L_\rho L_{\zeta}\norm{W_{1,p}}_2)$, $G_{\max}=\min\big(\norm{G(\mathcal{A})}_{\infty}^{\max},\norm{G(\mathcal{A})}_{F}^{\max}\big)$, $\norm{G(\mathcal{A})}_{\infty}^{\max}= \max_{\mbA_q\in \mathcal{A},\mu(f_q,\mbA_q)>0}\norm{G(\mbA_q)}_{\infty}$ and $\norm{G(\mathcal{A})}_{F}^{\max}= \max_{\mbA_q\in \mathcal{A},\mu(f_q,\mbA_q)>0}\norm{G(\mbA_q)}_{F}$.
\end{corollary}
In a similar approach to the proof in Corollary~\ref{Cor: GCN}, we can derive an upper bound on the generalization error of MPGNN based on the upper bound on the $\kappa(.)$ function.
\begin{corollary}\label{cor: mpgnn}
    Let us assume the same assumptions in Proposition~\ref{prop: GCN result} and the bounded $\kappa(.)$ function, i.e., $\sup_{x\in\mbR}|\kappa(x)|\leq M_{\kappa}\,,$ in MPGNN. Then the following upper bound holds on the generalization error of the Gibbs measure, $\mrm^{\alpha,p}(\mun)$, 
\begin{equation*}
 \genb(\mrm^{\alpha,p}(\mun),\mu) 
     \leq \frac{\alpha M_{p}^2 M_{\ell'}^2}{n} \,,
\end{equation*}
with $M_{p}=w_{2,p}\min\big( M_{\kappa},L_{\kappa}B_f(\norm{W_{3,p}}_2 +G_{\max} L_\rho L_{\zeta}\norm{W_{1,p}}_2)\big)$, $G_{\max}=\min\big(\norm{G(\mathcal{A})}_{\infty}^{\max},\norm{G(\mathcal{A})}_{F}^{\max}\big)$, $\norm{G(\mathcal{A})}_{\infty}^{\max}= \max_{\mbA_q\in \mathcal{A},\mu(f_q,\mbA_q)>0}\norm{G(\mbA_q)}_{\infty}$ and $\norm{G(\mathcal{A})}_{F}^{\max}= \max_{\mbA_q\in \mathcal{A},\mu(f_q,\mbA_q)>0}\norm{G(\mbA_q)}_{F}$.
\end{corollary}

\subsection{Details of Remark~\ref{remark: graph-filter and infinite norm} and Remark~\ref{remark: graph filter and frobenius norm} }

Using Lemma~\ref{lem: bound on inf norm of L}, we have $\norm{G(\mathcal{A})}_{\infty}^{\max}\leq \sqrt{(d_{\max}+1)/(d_{\min}+1)}$. Similar to Lemma~\ref{lem: bound on inf norm of L}, for sum-aggregation $G(\mbA)=\mbA+I$ we have $\norm{G(\mathcal{A})}_{\infty}^{\max}\leq d_{\max}+1$, and for random-walk $G(\mbA)=D^{-1}\mbA+I$ we have $\norm{G(\mathcal{A})}_{\infty}^{\max}= 2$.

Regarding to $\norm{G(\tilde{L})}_{F}^{\max}$, using Lemma~\ref{lemma: bound on F-norm}, we have $\norm{G(\tilde{L})}_{F}^{\max}\leq \sqrt{R_{\max}(\tilde{L})}\norm{G(\tilde{L})}_{2}^{max}$. Then, via Lemma~\ref{lemma: bound on second norm}, we have $\norm{G(\tilde{L})}_{2}=1$. Similarly, for the random walk graph filter $G(\mbA)=\tilde{D}^{-1}\mbA+I$ we have $\norm{G(\mathcal{A})}_{F}^{\max}\leq \norm{G(\tilde{D}^{-1}\mbA+I)}_{2}^{max}\sqrt{R_{\max}(\tilde{D}^{-1}\mbA+I)}$ and $\norm{G(\tilde{D}^{-1}\mbA+I)}_{2}^{max}=2$.

\section{Generalization Error Upper Bound via Rademacher Complexities}\label{app: rademacher}
Inspired by the Rademacher Complexity analysis in \cite{nishikawatwo} and \cite{nitanda2022particle}, we provide an upper bound on the generalization error of an over-parameterized one-hidden generic GNN model via Rademacher complexity analysis.

For Rademacher complexity analysis, we 
define the following hypothesis set for generic GNN functions based on the mean-field regime characterized by KL divergence between $(m,\pi)$,
\begin{equation}\label{eq:fklh} \begin{split}
    &\mathcal{F}_{\KLr}(H)\triangleq\Big\{\sum_{j=1}^N\mathbb{E}_{W\sim m}\Big[\phi\left(W,G(\mbA_q)[j,:]\mbF_q\right)\Big]: \KLr(m\|\pi)\leq H\Big\}.
    \end{split}
\end{equation}
To establish an upper bound on the generalization error of the generic GNN, we first use the following lemma to bound the Rademacher complexity of hypothesis set $\mathcal{F}_{\KLr}(H)$.

\begin{lemma}{\citep[Based on Lemma~5.5]{chen2020generalized}}\label{Lem: Rad complex}
    Under Assumption~\ref{ass: bounded Unit function}, the following bound holds on the empirical Rademacher complexity of the hypothesis set $\mathcal{F}_{\KLr}(H)$,
    \begin{equation}
        \hat{\mathfrak{R}}_{\mbZn}(\mathcal{F}_{\KLr}(H))\leq N M_{\phi}\sqrt{\frac{2H}{n}}\,.
    \end{equation}
\end{lemma}
\begin{proof}
Without loss of generality, we denote $\phi_T(W,\mbX_i):=\sum_{j=1}^N\phi\left(W,G(\mbA_i)[j,:]\mbF_i\right)$. 
Then, with \eqref{eq:fklh}, 
\begin{align*}
    \mathcal{F}_{\KLr}(H)
    &=\Big\{\mathbb{E}_{W\sim m}\Big[\phi_T(W,\mbX_q)\Big]: \KLr(m\|\pi)\leq H\Big\}.
\end{align*}
    From Donsker’s representation of KL for the definition of $\rademacher_{\mbZn}(\mathcal{F}_{\KLr}(H))$ and considering a constant  $\lambda>0$, we have
    \begin{align*}
        \rademacher_{\mbZn}(\mathcal{F}_{\KLr}(H))&=\frac{1}{\lambda}\bbE_{\sigma}\Bigg[\sup_{\substack{m:\KLr(m\|\pi)\leq H,\\ m<<\pi}}\frac{\lambda}{n}\sum_{i=1}^n\sigma_i\mathbb{E}_{W\sim m}\phi_T(W,\mbX_i)\Bigg]\\
         &\leq \frac{1}{\lambda}\bbE_{\sigma}\Bigg[\sup_{\substack{m:\KLr(m\|\pi)\leq H,\\ m<<\pi}}\Bigg\{\KLr(m\|\pi)
         +\log(\bbE_{W\sim \pi}\exp(\frac{\lambda}{n} \sum_{i=1}^n \sigma_i\phi_T(W,\mbX_i)))\Bigg\}\Bigg]\\
        &\leq \frac{1}{\lambda}\bbE_{\sigma}\Bigg[H+\log(\bbE_{W\sim \pi} \exp(\frac{\lambda}{n} \sum_{i=1}^n \sigma_i\phi_T(W,\mbX_i)))\Bigg]\\
        &\leq \frac{1}{\lambda}\Bigg[H+\log(\bbE_{W\sim \pi} \bbE_{\sigma}\exp(\frac{\lambda}{n} \sum_{i=1}^n \sigma_i\phi_T(W,\mbX_i)))\Bigg].
    \end{align*}
    Here we applied Jensen's inequality in the last line.
    Note that, $\{\sigma_i\}_{i=1}^n$ are i.i.d. Rademacher random variables. Then, from the Hoeffding inequality with respect to Rademacher random variables, we have
    \begin{align*}
     \bbE_{\sigma}\exp(\frac{\lambda}{n} \sum_{i=1}^n \sigma_i\phi_T(W,\mbX_i))\leq \exp\Bigg\{\frac{\lambda^2}{2n^2} \sum_{i=1}^n \phi_T^2(W,\mbX_i)\Bigg\}.
    \end{align*}
    Note that, we have $\phi_T(W,\mbX_i)\leq NM_{\phi}$. Therefore, we have,
    \begin{align*}
        \rademacher_{\mbZn}(\mathcal{F}_{\KLr}(H))&\leq \frac{1}{\lambda}\Big(H+\frac{N^2M_\phi^2\lambda^2}{2n}\Big)\\
        &= \frac{H}{\lambda}+\frac{N^2M_\phi^2\lambda}{2n},
    \end{align*}
    which is minimized at $\lambda=\sqrt{\frac{2nH}{N^2M_\phi^2}}$, yielding
    \[\rademacher_{\mbZn}(\mathcal{F}_{\KLr}(H))\leq N M_\phi\sqrt{\frac{2H}{n}}.\]
\end{proof}
\begin{remark}[Comparison with Lemma~5.5 in \cite{chen2020generalized}]
In \citep[Lemma~5.5]{chen2020generalized}, the authors assume a Gaussian prior $\pi$. However, in Lemma~\ref{Lem: Rad complex}, we do not assume a  Gaussian prior. Instead, we assume a  bounded unit output.   
\end{remark}

Note that, in \citet{nishikawatwo} and \citet{nitanda2022particle}, it is assumed that there exists a ``true'' distribution $m_{\mathrm{true}}\in \mathcal{P}(\mathcal{W})$  which satisfies $\ell(\Psi(m_{\mathrm{true}},\mbX_i),y_i)=0$ for all $(\mbX_i,y_i)\in\mathcal{Z}$ where $\mu(\mbX_i,y_i)>0$ and therefore we have $\mathrm{R}(m_{\mathrm{true}},\mun)=0$. In addition, it is assumed that the KL-divergence between the true distribution and the prior distribution is finite, i.e., $\KLr(m_{\mathrm{true}}\|\pi)< \infty$. Due to the fact that the Gibbs measure is the minimizer of \eqref{Eq: regularized risk}, $\mathcal{V}^{\alpha}(m,\mun)$, we have
\begin{align}
    \frac{1}{\alpha} \KLr(\mrm^\alpha(\mun)\|\pi)&\leq \mathrm{R}(\mrm^\alpha(\mun),\mun)+ \frac{1}{\alpha} \KLr(\mrm^\alpha(\mun)\|\pi)\\
    &\leq \mathrm{R}(m_{\mathrm{true}},\mun)+ \frac{1}{\alpha} \KLr(m_{\mathrm{true}}\|\pi)\\
    &\leq \frac{1}{\alpha} \KLr(m_{\mathrm{true}}\|\pi).
\end{align}
Therefore, the value of $H$ is estimated by $\KLr(m_{\mathrm{true}}\|\pi)$ in \citet{nishikawatwo} and \citet{nitanda2022particle}. However, $m_{\mathrm{true}}$ is unknown and cannot be computed.

Considering the Gibbs measure, we provide the following proposition to estimate the value of $H$ for the hypothesis set $\mathcal{F}_{\KLr}(H)$ in terms of known parameters of problem formulation. This is our main theoretical contribution in the area of Rademacher Complexity analysis; it is related to results by
\cite{chen2020generalized,nishikawatwo,nitanda2022particle}.
\begin{repproposition}{prop: bound on sum}[Upper bound on the symmetrized KL divergence]
    Under Assumptions~\ref{Ass: on loss NN}, \ref{ass: bounded Unit function}, and \ref{ass: readout function}, the following upper bound holds on the symmetrized KL divergence between the Gibbs measure $m^\alpha(\mun)$ and the prior measure $\pi$,
    \begin{equation*}
        \KLr_{\mathrm{sym}}(\mrm^\alpha(\mun)\|\pi)\leq 2 N M_{\phi} M_{\ell'}L_{\psi}\alpha.
    \end{equation*}
\end{repproposition}
\begin{proof}
The functional derivative of the empirical risk of a generic GNN concerning a measure $m$ is
\[ \frac{\delta \mathrm{R}(\mrm^\alpha(\mun), \mun,w)}{\delta m} =\frac{1}{n}\sum_{i=1}^n \partial_{\hat y}\ell\big(\bbE_{W\sim \mrm^\alpha(\mun)}[\Phi(W,\mbx_i)],y_i\big)\Phi(w,\mbx_i). \]
Note that $\frac{\delta \mathrm{R}(\mrm^\alpha(\mun), \mun,w)}{\delta m}$ is $N M_\phi M_{\ell'}L_{\psi}$-Lipschitz with respect to $w$ under 
the metric $d(w,w') = \mathds{1}_{w\neq w'}$. Recall that 
\begin{equation*}
    \mrm^\alpha(\mun)=\frac{\pi  }{S_{\alpha,\pi}(\mu_n)}\exp\Big\{-\alpha\frac{\delta \mathrm{R}(\mrm^\alpha(\mun), \mun,w)}{\delta m}\Big\}.
\end{equation*}
We can compute the symmetrized KL divergence as follows:
\begin{align*}
        \KLr_{\mathrm{sym}}(\mrm^\alpha(\mun)\|\pi)&=\KLr(\mrm^\alpha(\mun)\|\pi)+\KLr(\pi \| \mrm^\alpha(\mun))\\
        &= \bbE_{\mrm^\alpha(\mun)}\Big[\log(\frac{\mrm^\alpha(\mun)}{\pi})\Big]+\bbE_{\pi}\Big[\log(\frac{\pi}{\mrm^\alpha(\mun)})\Big]\\
        &=\bbE_{\pi}[\log(S_{\alpha,\pi})]-\bbE_{\mrm^\alpha(\mun)}[\log(S_{\alpha,\pi})]\\
        &\quad+\alpha\bbE_{\mrm^\alpha(\mun)}\Bigg[\frac{\delta \mathrm{R}(\mrm^\alpha(\mun), \mun,w)}{\delta m} \Bigg]-\alpha\bbE_{\pi}\Bigg[\frac{\delta \mathrm{R}(\mrm^\alpha(\mun), \mun,w)}{\delta m} \Bigg]\\
        &\leq \alpha\Bigg|\bbE_{\mrm^\alpha(\mun)}\Bigg[\frac{\delta \mathrm{R}(\mrm^\alpha(\mun), \mun,w)}{\delta m} \Bigg]-\bbE_{\pi}\Bigg[\frac{\delta \mathrm{R}(\mrm^\alpha(\mun), \mun,w)}{\delta m} \Bigg]\Bigg|
        \\&\leq 2 \alpha L_R  \mathbb{TV}(\mrm^\alpha(\mun),\pi),
    \end{align*}
    where $L_R=N M_\phi M_{\ell'}L_{\psi}$. The last and second to the last inequalities follow from the total variation distance representation, \eqref{Eq: tv rep}, and the fact that 
    $$\bbE_{\pi}[\log(S_{\alpha,\pi})]-\bbE_{\mrm^\alpha(\mun)}[\log(S_{\alpha,\pi})]=0\,.$$
    Using that $\mathbb{TV}(\mrm^\alpha(\mun),\pi)\leq 1$ completes the proof.
\end{proof}

Combining Lemma~\ref{Lem: Rad complex}, Proposition~\ref{prop: bound on sum}, Lemma~\ref{lem:contraction} (the Uniform bound), and Lemma~\ref{lem:contraction} (Talagrand’s contraction), we 
now derive the following upper bound on the generalization error for a generic GNN in the mean-field regime.
\begin{reptheorem}{thm: main unit upper rademacher}[Generalization error upper bound via Rademacher complexity]
    Let Assumptions~\ref{Ass: on loss NN}, \ref{ass: bounded Unit function}, \ref{ass: readout function}, and \ref{Ass: on loss NN bounded} hold. Then, for any $\delta\in(0,1)$, with probability at least $1-\delta$, the following upper bound holds on the generalization error of the Gibbs measure, i.e., $\mrm^\alpha(\mun)$, under the distribution of $P_{\mbZn}$,
\begin{equation*}
\begin{split}
 \mathrm{gen}(\mrm^\alpha(\mun),\mu) 
     &\leq  
  4 N_{\max} M_{\phi}M_{\ell'}L_{\psi} \sqrt{\frac{ N_{\max} M_{\phi} M_{\ell'}L_{\psi}\alpha}{n}}+3M_\ell\sqrt{\frac{\log(2/\delta)}{2n}}\,.
  \end{split}
\end{equation*}
\end{reptheorem}
\begin{proof}

From the uniform bound (Lemma~\ref{lem:uniform_bound}) and Talagrand's contraction lemma~(Lemma~\ref{lem:contraction}), we have for any $\delta \in (0,1)$
\begin{align}\label{eq: bound uniform 1}
 \mathrm{gen}(\mrm^\alpha(\mun),\mu) 
     &\leq  2M_{\ell'}L_{\psi}\rademacher_{\mbZn}(\mathcal{F}_{\KLr}(H))+3M_\ell \sqrt{\frac{1}{2n}\log(\frac{2}{\delta})}.
     \end{align}
   Combining Lemma~\ref{Lem: Rad complex} with \eqref{eq: bound uniform 1} results in
   \begin{align}\label{eq: bound uniform 2}
 \mathrm{gen}(\mrm^\alpha(\mun),\mu) 
     \leq 2N_{\max}M_{\phi}M_{\ell'}L_{\psi}\sqrt{\frac{2H}{n}}+3M_\ell \sqrt{\frac{1}{2n}\log(\frac{2}{\delta})}.
   \end{align}
Next, we find a suitable value of $H$ to be used in the definition \eqref{eq:fklh} of $\mathcal{F}_{\KLr}(H)$. For that purpose, note that
\begin{equation}
   \begin{split}
   \KLr(\mrm^\alpha(\mun)\|\pi)&\leq \KLr(\mrm^\alpha(\mun)\|\pi) +\KLr(\mrm^\alpha(\mun)\|\pi)\\
   &\leq 2 N_{\max} M_{\phi} M_{\ell'}L_{\psi}\alpha \mathbb{TV}(\mrm^\alpha(\mun),\pi)\\
   &\leq 2 N_{\max} M_{\phi} M_{\ell'}L_{\psi}\alpha,
   \end{split} 
\end{equation}
where the last inequality follows from $\mathbb{TV}(\mrm^\alpha(\mun),\pi)\leq 1$. Therefore, we choose $H=2 N_{\max} M_{\phi} M_{\ell'}L_{\psi}\alpha$ in \eqref{eq: bound uniform 2}. This choice completes the proof.
\end{proof}
Note that the upper bound in Theorem~\ref{thm: main unit upper rademacher} can be combined with Lemma~\ref{lem: GCN case} and Lemma~\ref{lem: bound unit mpgnn} to provide upper bounds on the generalization error of a GCN and an MPGNN in the mean-field regime, respectively. In addition to the assumptions in Theorem~\ref{thm: main unit upper}, however, in Theorem~\ref{thm: main unit upper rademacher} we need an extra assumption (Assumption~\ref{Ass: on loss NN bounded}).

Similar results to Corollary~\ref{Cor: sum-read and GCN} and Corollary~\ref{Cor: MPGNN and sum} for Rademacher complexity upper bounds based on sum-readout function can be derived using similar approach. Also, similar to Corollary~\ref{Cor: GCN} and Corollary~\ref{cor: mpgnn}, we can derive Rademacher complexity upper bounds based on mean-readout function.

\section{More Discussion for Table~\ref{table:comparison}}\label{App: dis com}
We also examine the dependency of our bound for the maximum/minimum node degree of graph data set $(d_{\max},d_{\min})$ in Table~\ref{table:comparison}. In \citet{ju2023generalization}, the upper bound is dependent on the spectral norm ($L_2$ norm of the graph filter) which is independent of $(d_{\max},d_{\min})$. The upper bound proposed by \cite{maskey2022generalization}, is dependent on $D_{\mathcal{X}}$ and the dimension of the space of graphon instead of the maximum and minimum degree. Note that, our bounds also depend on the maximum rank of the adjacency matrix in the graph data set.

\section{Experiments}\label{App: experiments}

\subsection{Implementation Details} 
\paragraph{Hardware and setup.} Experiments were conducted on two compute nodes, each with 8 Nvidia Tesla T4 GPUs, 96 Intel Xeon Platinum 8259CL CPUs @ 2.50GHz and $378$GB RAM. With this setup, all experiments were completed within one day. Note that we have six data sets, each with seven values of the width of the hidden layer, $h,$ for two different supervised ratio values, $\beta_\text{sup},$ on three model types (GCN, GCN\_RW\footnote{GCN with random walk as aggregation method}, and MPGNN) with two different readout functions (mean and sum), for ten different random seeds. Therefore, there are a total of $6\times 7\times 2\times 3\times 2\times 10=5,040$ single runs in our empirical analysis. 

\paragraph{Code.} Implementation code is provided at \url{https://github.com/SherylHYX/GNN_MF_GE}. We thank the authors of \cite{liao2020pac} for kindly sharing their code with us.

\paragraph{Training.} We train 200 epochs for synthetic data sets and 50 epochs for PROTEINS. The batch size is 128 for all data sets.

\paragraph{Regularization.}
Inspired by \citet{chen2020generalized}, we propose the following regularized empirical risk minimization for GCN.
\begin{align} \mathrm{R}(m_h^c(\mun),\mun) = \frac{1}{n}\sum_{i=1}^n \ell\big(  \bbE[\Psi_c(m_h^c(\mun),\mbx_i)],y_i\big)+\frac{1}{h\alpha}\sum_{i=1}^h \frac{\norm{W_c[i,:]}_2^2}{2},
\end{align}
where $W_c[i,:]$ denotes the parameters of the $i-$th neuron unit. For MPGNN, we consider the following regularized empirical risk minimization:

\begin{align} \mathrm{R}(m_h^p(\mun),\mun) = \frac{1}{n}\sum_{i=1}^n \ell\big(  \Psi_p(m_h^p(\mun),\mbx_i),y_i\big)+\frac{1}{h\alpha}\sum_{i=1}^h\frac{\norm{W_p[i,:]}_2^2}{2},
\end{align}
where $W_p[i,:]$ is the parameters of $i-$th MPU unit.

\paragraph{Optimizer.} Taking the regularization term into account, we use Stochastic Gradient Descend (SGD) from PyTorch as the optimizer and $\ell_2$ regularization with weight decay $\frac{1}{h\alpha}$ to avoid overfitting, where $h$ is the width of the hidden layer, and $\alpha$ is a tuning parameter which we set to be 100. We use a learning rate of 0.005 and a momentum of 0.9 throughout.

\paragraph{Tanh function.} For ease of bound computation, we use Tanh for GCN as the activation function and for MPGNN as the non-linear function $\kappa(.)$. 

\subsection{Data Sets} 
We generate five synthetic data sets from two random graph models using NetworkX~\citep{hagberg2008exploring}. The first three synthetic data sets correspond to Stochastic Block Models (SBMs) and the remaining two correspond to Erd\H{o}s-R\'enyi (ER) models. The synthetic models have the following settings: 
\begin{enumerate}
    \item Stochastic-Block-Models-1 (SBM-1), where each graph has 100 nodes, two blocks with size $40$ and $60$, respectively. The edge probability matrix is $$\begin{bmatrix}
0.25 & 0.13 \\
0.13 & 0.37
\end{bmatrix}.$$
\item Stochastic-Block-Models-2 (SBM-2), where each graph has 100 nodes, and three blocks with sizes $25, 25,$ and $50$, respectively. The edge probability matrix is $$\begin{bmatrix}
0.25 &0.05& 0.02 \\
0.05&0.35&0.07\\
0.02&0.07&0.40
\end{bmatrix}.$$
\item Stochastic-Block-Models-3 (SBM-3), where each graph has 50 nodes, and three blocks with sizes $15, 15,$ and $20$, respectively. The edge probability matrix is $$\begin{bmatrix}
0.5&0.1&0.2\\
0.1&0.4&0.1\\
0.2&0.1&0.4
\end{bmatrix}.$$
\item Erd\H{o}s-R\'enyi-Models-4 (ER-4), where each graph has 100 nodes, with edge probability 0.7.
\item Erd\H{o}s-R\'enyi-Models-5 (ER-5), where each graph has 20 nodes, with edge probability 0.5.
\end{enumerate}
Each synthetic data set has 200 graphs, the number of classes is 2, and the random train-test split ratio is $\beta_\text{sup}:(1-\beta_\text{sup})$, where in our experiments we vary $\beta_\text{sup}$ in $\{0.7, 0.9\}$. For each random
graph of an individual synthetic data set, we generate the 16-dimension random Gaussian node feature (normalized to have unit $\ell_2$ norm) and a binary class label following a uniform distribution. In addition to the synthetic data sets, we have one real-world bioinformatics data set called PROTEINS~\citep{borgwardt2005protein}. In PROTEINS, nodes are secondary structure elements, and two nodes are connected by an edge if they are neighbors in the amino-acid sequence or in 3D space. Table~\ref{tab:data_sets} summarizes the statistics for the data sets in our experiments.
\begin{table*}[htb]
\centering
\caption{Summary statistics for the data sets.}
\label{tab:data_sets}
\begin{tabular}{l|rrrrrrrrrrrrrrrrrr}
\toprule
Statistics/Data Set&SBM-1&SBM-2&SBM-3&ER-4&ER-5&PROTEINS\\
\midrule
Maximum number of nodes ($N_{\max})$&100&100&50&100&20&620\\
Number of graphs&200&200&200&200&200&1113\\
Feature dimension&16&16&16&16&16&3\\
Maximum node degree&14&35&22&87&16&25\\
Minimum node degree&6&1&1&52&2&0\\
\bottomrule
\end{tabular}
\end{table*}
\subsection{Bound Computation}
 As our upper bounds in Corollaries~\ref{Cor: GCN} and \ref{cor: mpgnn} are applicable for the over-parameterized regime in a continuous space of parameters, we estimate the upper bounds for a large number of hidden units, $h$. For this purpose, we need to compute the following parameters from the model and the data set:
\begin{itemize}
    \item $B_f$: We compute the $l_2$ norm of the node features before GNN aggregation and find the max $\norm{F[j,:]}_2$ for all training and test data. This would be the $B_f$ term in Corollaries~\ref{Cor: GCN} and \ref{cor: mpgnn}.
    \item $d_{\text{max}}$:  Maximum degree of a node in all graph samples.
    \item $L_\varphi$: Lipschitzness of the activation function. For Tanh, we have $L_\varphi=1$. 

    \item $w_{2,c}$: For GCN,  we can choose the maximum  value of $|W_{2,c}|$ as $w_{2,c}$. For MPGNN, we consider the maximum value of $|W_{2,p}|$ among all MPU units as $w_{2,c}$.

    \item $M_\ell$ and $M_{\ell'}$:  As we consider the Tanh function as our activation function in GCN and as the $\kappa(.)$ function in MPGNN, we have $M_{\ell}=\log(1+\exp(-w_{2,c}*B))$ and we also have $M_{\ell'}=1$. Note that the maximum value of the derivative of the logistic loss function $\ell(\Psi(W,\mbx),y)=\log(1+\exp(-\Psi(W,\mbx) y))$ is $1$. 
    \item $\norm{W_{1,c,m}}_2$: We consider the maximum value of $\norm{W_{1,c}}_2$ among all neuron units.
    \item $\norm{W_{1,p,m}}_2$ and $\norm{W_{3,p,m}}_2$: We consider the maximum values of $\norm{W_{1,p}}_2$ and $\norm{W_{3,p}}_2$ among all MPU units as $\norm{W_{1,p,m}}_2$ and $\norm{W_{3,p,m}}_2$, respectively.
\end{itemize}
\subsection{Extended Experimental Results}

We compute the actual absolute empirical generalization errors (the difference between the test loss value and the training loss value) as well as empirical generalization error bounds for GCN, GCN\_RW, and MPGNN with either the mean-readout function or the sum-readout function. Here GCN\_RW denotes a variant of GCN where the symmetric normalization of the adjacency matrix is replaced by the random walk normalization. Specifically, Table~\ref{tab:GE_GCN_mean_full} reports the effect of the width $h$ of the hidden layer on the actual absolute empirical generalization errors for GCN when we employ a mean-readout function, over various data sets and supervised ratio values $\beta_\text{sup}$, where Table~\ref{tab:GE_GCN_RW_mean_full} and
Table~\ref{tab:GE_MPGNN_mean_full} reports those
for GCN\_RW and MPGNN when we employ a mean-readout function, respectively.
Tables~\ref{tab:GE_GCN_sum_full}, \ref{tab:GE_GCN_RW_sum_full}, and \ref{tab:GE_MPGNN_sum_full} report the actual absolute empirical generalization errors for GCN, GCN\_RW, and MPGNN with the sum-readout function, respectively. Table~\ref{tab:bounds_full} reports the empirical generalization error bound values for all three model types for both mean and sum readout functions with hidden size $h=256$. 

\input{tables_tex/SI_tables}

\subsection{Discussion}
In all our experimental findings, 
increasing the width of the hidden layer leads to a reduction in generalization errors. Regarding the comparison of mean-readout and sum-readout, based on our results in Corollary~\ref{Cor: sum-read and GCN} and Proposition~\ref{prop: GCN result} for GCN and Corollary~\ref{Cor: MPGNN and sum} and Proposition~\ref{Prop: MPGNN bound} for MPGNN, the upper bound on the generalization error of GCN and MPGNN with mean-readout is less than the upper bound on the generalization error with sum-readout function in GCN and MPGNN, respectively. This pattern is similarly evident in the empirical generalization error results presented in Tables~\ref{tab:GE_GCN_mean_full} and \ref{tab:GE_GCN_sum_full} for GCN and Tables~\ref{tab:GE_MPGNN_mean_full} and \ref{tab:GE_MPGNN_sum_full} for MPGNN. As shown in Table~\ref{tab:bounds_full} and Table~\ref{tab:bounds_full_Rademacher}, the upper bounds on the generalization error of GCN under $\tilde{L}$ and random walk as graph filters are similar. The empirical generalization error of GCN under $\tilde{L}$ and random walk in Tables~\ref{tab:GE_GCN_mean_full} and \ref{tab:GE_GCN_RW_mean_full} are also similar. Tables~\ref{tab:bounds_full_h128} to \ref{tab:bounds_full_Rademacher_h64} provides other bound computation results based on smaller $h$ values. We conclude that these are very similar to the bounds computed using $h=256,$ and hence $h=256$ should be a proper approximation to the mean-field regime.

\end{document}

%% file: tables_tex/SI_tables.tex
\input{tables_tex/GCN_mean_ge}

\input{tables_tex/GCN_RW_mean_ge}

\input{tables_tex/MPGNN_mean_ge}

\input{tables_tex/GCN_sum_ge}

\input{tables_tex/GCN_RW_sum_ge}

\input{tables_tex/MPGNN_sum_ge}
\input{tables_tex/bounds}

%% file: tables_tex/GCN_mean_ge.tex
\begin{table*}[htb]
\centering
\caption{Absolute empirical generalization error (\pmb{$\times 10^5$}) for different widths $h$ of the hidden layer for GCN, for which we employ a mean-readout function and various supervised ratios of $\beta_\text{sup}$. We report the mean plus/minus one standard deviation over ten runs.}
\label{tab:GE_GCN_mean_full}
\resizebox{\linewidth}{!}{\begin{tabular}{lrrrrrrrrrrrrrrrrrr}
\toprule
Data Set&$\beta_\text{sup}$&$h=4$&$h=8$&$h=16$&$h=32$&$h=64$&$h=128$&$h=256$\\
\midrule
SBM-1&$0.7$&$66.37\pm44.09$&$18.66\pm15.44$&$13.91\pm9.46$&$3.42\pm2.96$&$2.04\pm1.68$&$1.23\pm0.99$&$0.29\pm0.27$\\
SBM-1&$0.9$&$101.34\pm73.25$&$36.35\pm37.01$&$12.76\pm14.21$&$12.00\pm4.81$&$3.56\pm2.69$&$0.86\pm0.60$&$0.60\pm0.39$\\
SBM-2&$0.7$&$67.27\pm44.95$&$19.57\pm16.12$&$13.54\pm9.79$&$3.79\pm3.16$&$2.36\pm1.71$&$1.25\pm0.93$&$0.31\pm0.22$\\
SBM-2&$0.9$&$103.32\pm77.20$&$36.67\pm40.96$&$15.38\pm15.17$&$11.85\pm4.41$&$3.51\pm2.41$&$0.96\pm0.56$&$0.60\pm0.38$\\
SBM-3&$0.7$&$85.51\pm54.72$&$42.42\pm23.06$&$25.18\pm10.59$&$6.72\pm4.32$&$4.36\pm3.63$&$1.13\pm0.62$&$0.33\pm0.23$\\
SBM-3&$0.9$&$108.83\pm139.12$&$61.35\pm45.92$&$24.04\pm22.08$&$13.07\pm13.03$&$4.19\pm3.06$&$1.81\pm1.50$&$0.73\pm0.61$\\
ER-4&$0.7$&$69.63\pm45.13$&$18.39\pm15.77$&$14.04\pm10.49$&$3.86\pm3.36$&$2.17\pm1.67$&$1.24\pm1.02$&$0.33\pm0.25$\\
ER-4&$0.9$&$102.21\pm79.73$&$38.18\pm37.76$&$13.49\pm13.99$&$12.32\pm4.67$&$3.52\pm2.61$&$0.88\pm0.56$&$0.59\pm0.36$\\
ER-5&$0.7$&$123.16\pm113.58$&$60.06\pm29.55$&$39.63\pm23.92$&$16.21\pm7.46$&$3.34\pm2.76$&$1.63\pm1.02$&$0.59\pm0.62$\\
ER-5&$0.9$&$203.04\pm97.39$&$100.28\pm95.76$&$48.95\pm33.67$&$21.76\pm17.30$&$6.91\pm6.40$&$1.98\pm1.92$&$1.00\pm0.62$\\
PROTEINS&$0.7$&$256.91\pm132.64$&$157.51\pm145.88$&$50.78\pm63.32$&$22.77\pm19.97$&$8.37\pm7.71$&$1.85\pm1.28$&$1.72\pm1.38$\\
PROTEINS&$0.9$&$95.66\pm85.97$&$67.25\pm52.78$&$26.35\pm15.27$&$12.36\pm7.95$&$2.18\pm1.85$&$1.11\pm1.16$&$0.33\pm0.28$\\
\bottomrule
\end{tabular}}
\end{table*}

%% file: tables_tex/GCN_RW_mean_ge.tex
\begin{table*}[htb]
\centering
\caption{Empirical absolute generalization error (\pmb{$\times 10^5$}) for different widths $h$ of the hidden layer for GCN\_RW, for which we employ a mean-readout function and various supervised ratios of $\beta_\text{sup}$. We report the mean plus/minus one standard deviation over ten runs.}
\label{tab:GE_GCN_RW_mean_full}
\resizebox{\linewidth}{!}{\begin{tabular}{lrrrrrrrrrrrrrrrrrr}
\toprule
Data Set&$\beta_\text{sup}$&$h=4$&$h=8$&$h=16$&$h=32$&$h=64$&$h=128$&$h=256$\\
\midrule
SBM-1&$0.7$&$66.71\pm43.86$&$18.49\pm15.25$&$13.85\pm9.33$&$3.41\pm3.08$&$2.05\pm1.69$&$1.23\pm0.99$&$0.29\pm0.26$\\
SBM-1&$0.9$&$101.61\pm72.92$&$35.81\pm36.26$&$12.65\pm13.99$&$12.01\pm4.85$&$3.60\pm2.71$&$0.85\pm0.59$&$0.59\pm0.39$\\
SBM-2&$0.7$&$66.87\pm46.15$&$19.40\pm16.29$&$13.70\pm9.86$&$3.87\pm3.37$&$2.35\pm1.73$&$1.24\pm0.94$&$0.32\pm0.22$\\
SBM-2&$0.9$&$102.64\pm77.20$&$37.06\pm40.15$&$15.15\pm15.06$&$11.90\pm4.43$&$3.54\pm2.42$&$0.95\pm0.54$&$0.59\pm0.38$\\
SBM-3&$0.7$&$84.58\pm53.86$&$42.57\pm22.71$&$25.41\pm10.69$&$6.67\pm4.39$&$4.37\pm3.60$&$1.13\pm0.60$&$0.33\pm0.22$\\
SBM-3&$0.9$&$107.79\pm138.73$&$61.28\pm46.17$&$24.22\pm22.75$&$13.08\pm12.93$&$4.19\pm3.09$&$1.81\pm1.50$&$0.74\pm0.63$\\
ER-4&$0.7$&$69.60\pm45.04$&$18.40\pm15.75$&$14.04\pm10.50$&$3.85\pm3.36$&$2.17\pm1.67$&$1.24\pm1.02$&$0.35\pm0.26$\\
ER-4&$0.9$&$102.20\pm79.69$&$38.14\pm37.83$&$13.49\pm13.99$&$12.32\pm4.67$&$3.52\pm2.61$&$0.88\pm0.56$&$0.59\pm0.37$\\
ER-5&$0.7$&$121.54\pm111.05$&$59.94\pm29.04$&$39.85\pm23.30$&$16.13\pm7.67$&$3.38\pm2.81$&$1.62\pm1.02$&$0.61\pm0.62$\\
ER-5&$0.9$&$200.88\pm97.85$&$99.61\pm94.82$&$49.04\pm34.17$&$21.71\pm17.31$&$6.93\pm6.44$&$1.98\pm1.92$&$1.02\pm0.62$\\
PROTEINS&$0.7$&$257.77\pm133.07$&$157.76\pm146.31$&$50.83\pm63.41$&$22.79\pm19.97$&$8.37\pm7.72$&$1.85\pm1.28$&$1.72\pm1.38$\\
PROTEINS&$0.9$&$95.82\pm85.65$&$67.42\pm52.74$&$26.37\pm15.19$&$12.38\pm7.96$&$2.18\pm1.84$&$1.11\pm1.16$&$0.32\pm0.28$\\
\bottomrule
\end{tabular}}
\end{table*}

%% file: tables_tex/MPGNN_mean_ge.tex
\begin{table*}[htb]
\centering
\caption{Empirical absolute generalization error (\pmb{$\times 10^5$}) for different widths $h$ of the hidden layer for MPGNN, for which we employ a mean-readout function and various supervised ratios of $\beta_\text{sup}$. We report the mean plus/minus one standard deviation over ten runs.}
\label{tab:GE_MPGNN_mean_full}
\resizebox{\linewidth}{!}{\begin{tabular}{lrrrrrrrrrrrrrrrrrr}
\toprule
Data Set&$\beta_\text{sup}$&$h=4$&$h=8$&$h=16$&$h=32$&$h=64$&$h=128$&$h=256$\\
\midrule
SBM-1&$0.7$&$31.39\pm27.27$&$27.39\pm37.85$&$11.50\pm9.70$&$5.13\pm5.20$&$3.10\pm2.19$&$0.56\pm0.42$&$0.77\pm0.52$\\
SBM-1&$0.9$&$54.32\pm35.73$&$24.44\pm17.76$&$13.64\pm8.20$&$7.76\pm3.12$&$2.68\pm1.35$&$1.80\pm1.00$&$0.87\pm0.57$\\
SBM-2&$0.7$&$31.59\pm27.47$&$25.64\pm36.66$&$11.10\pm9.45$&$5.03\pm5.17$&$3.28\pm2.19$&$0.80\pm0.56$&$0.77\pm0.56$\\
SBM-2&$0.9$&$58.06\pm36.49$&$26.80\pm19.50$&$14.48\pm8.15$&$7.38\pm3.80$&$2.71\pm1.58$&$1.86\pm0.96$&$0.88\pm0.63$\\
SBM-3&$0.7$&$42.21\pm54.18$&$12.36\pm9.99$&$15.00\pm12.97$&$8.99\pm5.73$&$3.18\pm2.86$&$1.86\pm1.40$&$1.10\pm0.72$\\
SBM-3&$0.9$&$72.91\pm35.61$&$26.80\pm27.89$&$11.07\pm6.90$&$9.62\pm7.27$&$3.78\pm2.42$&$2.55\pm2.27$&$0.86\pm0.79$\\
ER-4&$0.7$&$31.61\pm26.04$&$25.19\pm36.25$&$11.17\pm9.81$&$5.01\pm4.88$&$3.10\pm2.14$&$0.64\pm0.49$&$0.74\pm0.56$\\
ER-4&$0.9$&$58.24\pm34.28$&$25.84\pm18.96$&$13.27\pm7.39$&$7.47\pm3.25$&$2.71\pm1.43$&$1.79\pm1.01$&$0.89\pm0.56$\\
ER-5&$0.7$&$75.54\pm57.19$&$35.99\pm36.89$&$25.52\pm15.42$&$10.60\pm8.65$&$4.38\pm4.97$&$3.01\pm1.79$&$1.93\pm1.02$\\
ER-5&$0.9$&$112.86\pm40.24$&$54.54\pm62.52$&$32.05\pm20.23$&$13.84\pm10.95$&$6.55\pm4.59$&$1.96\pm2.12$&$1.45\pm0.75$\\
PROTEINS&$0.7$&$21.09\pm13.77$&$18.05\pm29.26$&$5.17\pm3.57$&$2.33\pm2.56$&$1.42\pm1.23$&$1.17\pm1.18$&$0.22\pm0.23$\\
PROTEINS&$0.9$&$109.26\pm71.68$&$57.88\pm42.90$&$25.44\pm22.23$&$13.80\pm10.87$&$6.96\pm4.22$&$4.37\pm2.09$&$2.17\pm1.58$\\
\bottomrule
\end{tabular}}
\end{table*}

%% file: tables_tex/GCN_sum_ge.tex
\begin{table*}[htb]
\centering
\caption{Empirical absolute generalization error (\pmb{$\times 10^5$}) for different widths $h$ of the hidden layer for GCN, for which we employ a sum-readout function and various supervised ratios of $\beta_\text{sup}$. We report the mean plus/minus one standard deviation over ten runs.}
\label{tab:GE_GCN_sum_full}
\resizebox{\linewidth}{!}{\begin{tabular}{lrrrrrrrrrrrrrrrrrr}
\toprule
Data Set&$\beta_\text{sup}$&$h=4$&$h=8$&$h=16$&$h=32$&$h=64$&$h=128$&$h=256$\\
\midrule
SBM-1&$0.7$&$32043.32\pm16566.13$&$23188.17\pm11633.58$&$12027.73\pm5808.93$&$3292.38\pm1512.93$&$954.86\pm485.91$&$309.84\pm195.80$&$69.97\pm53.92$\\
SBM-1&$0.9$&$11014.75\pm9029.01$&$8587.25\pm5798.39$&$4112.02\pm2643.86$&$1061.97\pm842.66$&$483.87\pm310.76$&$158.10\pm111.88$&$59.79\pm38.18$\\
SBM-2&$0.7$&$29492.62\pm15154.35$&$21897.35\pm10576.63$&$11528.95\pm5424.90$&$3180.05\pm1592.76$&$893.67\pm510.50$&$302.50\pm177.62$&$65.90\pm47.15$\\
SBM-2&$0.9$&$10939.57\pm9443.50$&$8767.78\pm6205.08$&$4198.28\pm2726.61$&$1029.97\pm840.74$&$452.85\pm299.91$&$162.18\pm111.69$&$59.96\pm35.63$\\
SBM-3&$0.7$&$24106.77\pm12571.98$&$14354.73\pm5584.92$&$5151.51\pm3047.06$&$1978.71\pm1092.38$&$300.63\pm198.73$&$124.01\pm81.82$&$28.69\pm19.98$\\
SBM-3&$0.9$&$11512.44\pm7587.11$&$6449.65\pm4699.86$&$2671.03\pm2532.28$&$1033.62\pm1133.25$&$274.73\pm169.96$&$110.43\pm83.63$&$35.22\pm25.80$\\
ER-4&$0.7$&$31101.54\pm15568.03$&$22961.55\pm10676.48$&$12133.39\pm5563.73$&$3385.93\pm1638.43$&$960.17\pm499.79$&$311.68\pm194.79$&$71.25\pm53.81$\\
ER-4&$0.9$&$11349.50\pm9665.34$&$8915.33\pm6303.86$&$4305.80\pm2789.51$&$1092.02\pm868.09$&$487.43\pm312.86$&$159.23\pm109.41$&$59.54\pm36.57$\\
ER-5&$0.7$&$12857.92\pm6902.99$&$6270.49\pm2471.73$&$2971.75\pm1483.12$&$594.54\pm481.81$&$222.13\pm105.15$&$54.96\pm21.80$&$20.09\pm17.30$\\
ER-5&$0.9$&$3459.34\pm2468.52$&$2126.09\pm1573.64$&$937.35\pm715.44$&$410.62\pm305.72$&$161.93\pm130.67$&$43.76\pm37.61$&$20.60\pm13.61$\\
PROTEINS&$0.7$&$2901.59\pm1921.36$&$2763.40\pm1750.78$&$2686.88\pm1763.09$&$2639.22\pm1548.08$&$2262.60\pm1594.35$&$1395.51\pm1028.10$&$552.28\pm408.37$\\
PROTEINS&$0.9$&$5496.67\pm121.21$&$5461.98\pm30.99$&$5418.24\pm53.89$&$5683.96\pm131.28$&$6172.92\pm113.16$&$4711.74\pm182.59$&$2162.89\pm181.60$\\
\bottomrule
\end{tabular}}
\end{table*}

%% file: tables_tex/GCN_RW_sum_ge.tex
\begin{table*}[htb]
\centering
\caption{Empirical absolute generalization error (\pmb{$\times 10^5$}) for different widths $h$ of the hidden layer for GCN\_RW, for which we employ a sum-readout function and various supervised ratios of $\beta_\text{sup}$. We report the mean plus/minus one standard deviation over ten runs.}
\label{tab:GE_GCN_RW_sum_full}
\resizebox{\linewidth}{!}{\begin{tabular}{lrrrrrrrrrrrrrrrrrr}
\toprule
Data Set&$\beta_\text{sup}$&$h=4$&$h=8$&$h=16$&$h=32$&$h=64$&$h=128$&$h=256$\\
\midrule
SBM-1&$0.7$&$32202.69\pm16672.76$&$23276.51\pm11663.99$&$12093.87\pm5792.22$&$3319.33\pm1528.94$&$959.65\pm484.71$&$310.31\pm195.87$&$70.17\pm53.66$\\
SBM-1&$0.9$&$11047.91\pm9063.60$&$8605.98\pm5812.74$&$4117.02\pm2640.91$&$1051.64\pm836.74$&$492.20\pm312.51$&$157.20\pm110.82$&$59.17\pm37.52$\\
SBM-2&$0.7$&$29464.80\pm15051.55$&$21963.74\pm10570.57$&$11625.39\pm5466.76$&$3205.53\pm1653.48$&$900.03\pm512.29$&$302.58\pm179.68$&$66.34\pm48.32$\\
SBM-2&$0.9$&$10969.08\pm9674.08$&$8786.41\pm6312.44$&$4219.59\pm2785.49$&$1024.60\pm849.98$&$458.05\pm305.83$&$161.38\pm110.93$&$59.30\pm35.73$\\
SBM-3&$0.7$&$24204.66\pm12196.40$&$14446.60\pm5486.24$&$5198.20\pm3006.84$&$1986.90\pm1079.75$&$304.06\pm200.22$&$124.71\pm82.10$&$28.79\pm19.92$\\
SBM-3&$0.9$&$11521.04\pm7603.94$&$6424.98\pm4711.22$&$2711.02\pm2530.53$&$1038.27\pm1121.70$&$272.09\pm168.42$&$111.38\pm83.39$&$35.89\pm26.55$\\
ER-4&$0.7$&$31094.88\pm15550.76$&$22956.86\pm10665.25$&$12131.93\pm5561.57$&$3385.06\pm1636.21$&$960.45\pm499.69$&$311.72\pm194.88$&$71.27\pm53.77$\\
ER-4&$0.9$&$11345.07\pm9661.97$&$8913.06\pm6303.20$&$4303.07\pm2791.26$&$1091.45\pm867.72$&$487.14\pm312.64$&$159.09\pm109.52$&$59.55\pm36.57$\\
ER-5&$0.7$&$12866.49\pm6889.85$&$6276.96\pm2479.75$&$2982.31\pm1475.88$&$590.39\pm481.24$&$222.85\pm104.62$&$54.99\pm21.79$&$20.23\pm17.47$\\
ER-5&$0.9$&$3435.09\pm2462.73$&$2118.73\pm1592.37$&$938.36\pm717.11$&$408.90\pm303.14$&$161.82\pm132.05$&$43.84\pm37.65$&$20.88\pm13.62$\\
PROTEINS&$0.7$&$2900.35\pm1919.12$&$2763.56\pm1749.91$&$2687.83\pm1762.29$&$2640.40\pm1547.35$&$2262.67\pm1594.47$&$1395.50\pm1028.99$&$552.57\pm408.40$\\
PROTEINS&$0.9$&$5501.96\pm117.09$&$5470.32\pm27.49$&$5425.94\pm53.69$&$5690.28\pm131.20$&$6181.97\pm113.28$&$4720.11\pm182.87$&$2166.19\pm181.78$\\
\bottomrule
\end{tabular}}
\end{table*}

%% file: tables_tex/MPGNN_sum_ge.tex
\begin{table*}[htb]
\centering
\vspace{-5pt}
\caption{Empirical absolute generalization error (\pmb{$\times 10^5$}) for different widths $h$ of the hidden layer for MPGNN, for which we employ a sum-readout function and various supervised ratios of $\beta_\text{sup}$. We report the mean plus/minus one standard deviation over ten runs.}
\label{tab:GE_MPGNN_sum_full}
\resizebox{\linewidth}{!}{\begin{tabular}{lrrrrrrrrrrrrrrrrrr}
\toprule
Data Set&$\beta_\text{sup}$&$h=4$&$h=8$&$h=16$&$h=32$&$h=64$&$h=128$&$h=256$\\
\midrule
SBM-1&$0.7$&$6585.43\pm4142.24$&$5259.12\pm5278.60$&$2267.42\pm999.08$&$1870.21\pm1126.94$&$863.73\pm615.32$&$382.73\pm199.14$&$221.63\pm108.73$\\
SBM-1&$0.9$&$4891.93\pm4654.50$&$3103.23\pm3360.01$&$1857.44\pm1867.96$&$911.86\pm949.42$&$444.46\pm290.26$&$311.74\pm232.01$&$153.68\pm150.55$\\
SBM-2&$0.7$&$6501.10\pm4180.16$&$5044.84\pm5075.24$&$2273.68\pm998.62$&$1910.63\pm1122.18$&$851.73\pm613.75$&$390.12\pm198.82$&$228.47\pm108.99$\\
SBM-2&$0.9$&$5001.03\pm4333.61$&$3282.25\pm3330.88$&$1853.53\pm1942.58$&$866.83\pm949.21$&$474.32\pm316.04$&$318.67\pm232.69$&$160.81\pm156.28$\\
SBM-3&$0.7$&$6344.20\pm5047.26$&$2025.66\pm1769.09$&$1928.82\pm1554.69$&$1118.35\pm712.80$&$555.00\pm350.92$&$277.29\pm221.90$&$145.94\pm124.73$\\
SBM-3&$0.9$&$3740.18\pm2378.95$&$1523.19\pm1165.79$&$935.91\pm700.15$&$579.09\pm388.71$&$247.42\pm228.09$&$180.99\pm110.70$&$76.95\pm49.52$\\
ER-4&$0.7$&$6407.19\pm4182.56$&$5045.12\pm5052.70$&$2261.19\pm1028.58$&$1882.95\pm1098.92$&$848.09\pm617.37$&$383.24\pm188.61$&$219.20\pm108.37$\\
ER-4&$0.9$&$4492.42\pm4098.88$&$3340.34\pm3534.57$&$1754.96\pm1911.25$&$841.93\pm963.63$&$450.14\pm282.38$&$307.74\pm232.70$&$155.84\pm149.04$\\
ER-5&$0.7$&$1175.50\pm960.26$&$1004.51\pm913.75$&$838.87\pm482.98$&$362.88\pm261.34$&$135.72\pm67.97$&$112.66\pm71.25$&$46.78\pm37.25$\\
ER-5&$0.9$&$1843.43\pm1152.07$&$1115.01\pm1019.58$&$748.95\pm399.20$&$270.69\pm242.98$&$156.37\pm94.23$&$77.84\pm50.28$&$40.04\pm18.55$\\
PROTEINS&$0.7$&$1084.56\pm780.16$&$1007.44\pm975.67$&$1306.72\pm887.61$&$1084.37\pm588.17$&$1085.74\pm537.60$&$843.62\pm527.95$&$762.83\pm537.17$\\
PROTEINS&$0.9$&$558.80\pm88.90$&$839.32\pm97.22$&$1277.49\pm147.09$&$1848.46\pm100.86$&$2281.52\pm64.97$&$2573.71\pm48.59$&$2510.18\pm29.38$\\
\bottomrule
\end{tabular}}
\end{table*}

%% file: tables_tex/bounds.tex
\begin{table*}[htb]
\centering
\vspace{-5pt}
\caption{Empirical generalization error bounds via functional derivative for width $h=256$ of the hidden layer for different model types and readout functions, for various supervised ratios of $\beta_\text{sup}$. We report the mean plus/minus one standard deviation over ten runs.}
\label{tab:bounds_full}
\resizebox{\linewidth}{!}{\begin{tabular}{lr|rrr|rrrrrrrrrrrrrr}
\toprule
\multirow{ 2}{*}{Data Set}&\multirow{ 2}{*}{$\beta_\text{sup}$}&\multicolumn{3}{c|}{mean readout}&\multicolumn{3}{c}{sum readout}\\
&&GCN&GCN\_RW&MPGNN&GCN&GCN\_RW&MPGNN\\
\midrule
SBM-1&$0.7$&$0.014\pm0.001$&$0.015\pm0.001$&$0.244\pm0.013$&$2520476.175\pm156887.505$&$2669876.975\pm164650.807$&$4005.391\pm220.699$\\
SBM-1&$0.9$&$0.011\pm0.001$&$0.012\pm0.001$&$0.190\pm0.010$&$1885484.212\pm142336.573$&$1996144.750\pm150151.799$&$3118.617\pm164.941$\\
SBM-2&$0.7$&$0.023\pm0.001$&$0.025\pm0.001$&$0.341\pm0.019$&$4089494.725\pm260354.119$&$4438159.850\pm275790.885$&$5603.943\pm323.262$\\
SBM-2&$0.9$&$0.018\pm0.001$&$0.020\pm0.001$&$0.266\pm0.015$&$3049281.675\pm228393.096$&$3310971.200\pm245013.883$&$4364.249\pm244.524$\\
SBM-3&$0.7$&$0.015\pm0.001$&$0.016\pm0.001$&$0.249\pm0.013$&$624203.919\pm38222.136$&$657592.244\pm40341.332$&$4091.722\pm236.617$\\
SBM-3&$0.9$&$0.011\pm0.001$&$0.012\pm0.001$&$0.193\pm0.010$&$476286.562\pm25613.135$&$501607.622\pm26855.429$&$3182.215\pm179.787$\\
ER-4&$0.7$&$0.005\pm0.000$&$0.005\pm0.000$&$0.122\pm0.005$&$811877.012\pm49965.822$&$815294.519\pm50332.475$&$1993.695\pm90.394$\\
ER-4&$0.9$&$0.004\pm0.000$&$0.004\pm0.000$&$0.095\pm0.004$&$607448.919\pm45159.041$&$609996.025\pm45468.797$&$1552.299\pm68.335$\\
ER-5&$0.7$&$0.008\pm0.000$&$0.009\pm0.001$&$0.168\pm0.008$&$53656.979\pm2753.653$&$58993.716\pm3014.306$&$2752.105\pm142.909$\\
ER-5&$0.9$&$0.006\pm0.000$&$0.007\pm0.000$&$0.131\pm0.006$&$41149.381\pm2515.410$&$45243.126\pm2763.916$&$2137.883\pm111.558$\\
PROTEINS&$0.7$&$0.005\pm0.000$&$0.005\pm0.000$&$0.045\pm0.003$&$363676.153\pm93112.546$&$363892.670\pm93165.563$&$2122.023\pm27.214$\\
PROTEINS&$0.9$&$0.004\pm0.000$&$0.004\pm0.000$&$0.035\pm0.002$&$402014.978\pm50920.169$&$402314.366\pm50931.723$&$1657.479\pm21.872$\\
\bottomrule
\end{tabular}}
\end{table*}

\begin{table*}[htb]
\centering
\vspace{-5pt}
\caption{Empirical generalization error bounds via Rademacher Complexities for width $h=256$ of the hidden layer for different model types and readout functions, for various supervised ratios of $\beta_\text{sup}$. We report the mean plus/minus one standard deviation over ten runs.}
\label{tab:bounds_full_Rademacher}
\resizebox{\linewidth}{!}{\begin{tabular}{lr|rrr|rrrrrrrrrrrrrr}
\toprule
\multirow{ 2}{*}{Data Set}&\multirow{ 2}{*}{$\beta_\text{sup}$}&\multicolumn{3}{c|}{mean readout}&\multicolumn{3}{c}{sum readout}\\
&&GCN&GCN\_RW&MPGNN&GCN&GCN\_RW&MPGNN\\
\midrule
SBM-1&$0.7$&$0.465\pm0.002$&$0.465\pm0.002$&$3.774\pm0.063$&$304382.491\pm9500.618$&$304607.622\pm9518.604$&$4814.476\pm84.398$\\
SBM-1&$0.9$&$0.410\pm0.002$&$0.410\pm0.002$&$3.329\pm0.055$&$260826.333\pm7491.220$&$260911.930\pm7550.637$&$4245.814\pm74.430$\\
SBM-2&$0.7$&$0.465\pm0.002$&$0.465\pm0.002$&$3.774\pm0.063$&$305824.653\pm9984.800$&$305862.869\pm10206.895$&$4814.185\pm84.265$\\
SBM-2&$0.9$&$0.410\pm0.002$&$0.410\pm0.002$&$3.329\pm0.055$&$261423.833\pm7225.399$&$261561.489\pm7270.949$&$4245.654\pm74.354$\\
SBM-3&$0.7$&$0.465\pm0.002$&$0.465\pm0.002$&$3.774\pm0.063$&$104673.709\pm2587.052$&$104675.531\pm2549.178$&$4816.482\pm84.778$\\
SBM-3&$0.9$&$0.410\pm0.002$&$0.410\pm0.002$&$3.329\pm0.055$&$91149.288\pm1448.412$&$91135.775\pm1449.661$&$4247.080\pm74.968$\\
ER-4&$0.7$&$0.465\pm0.002$&$0.465\pm0.002$&$3.774\pm0.063$&$305177.969\pm10269.198$&$305189.119\pm10276.768$&$4814.444\pm84.389$\\
ER-4&$0.9$&$0.410\pm0.002$&$0.410\pm0.002$&$3.329\pm0.055$&$261501.731\pm7297.492$&$261507.970\pm7300.305$&$4245.799\pm74.568$\\
ER-5&$0.7$&$0.465\pm0.002$&$0.465\pm0.002$&$3.600\pm0.137$&$26211.781\pm520.483$&$26209.152\pm522.064$&$4557.881\pm205.714$\\
ER-5&$0.9$&$0.410\pm0.002$&$0.410\pm0.002$&$3.175\pm0.120$&$22835.085\pm242.080$&$22832.954\pm242.523$&$4014.455\pm180.352$\\
PROTEINS&$0.7$&$0.197\pm0.001$&$0.197\pm0.001$&$0.844\pm0.030$&$59715.427\pm7057.934$&$59730.912\pm7058.069$&$2089.489\pm20.112$\\
PROTEINS&$0.9$&$0.174\pm0.001$&$0.174\pm0.001$&$0.747\pm0.026$&$64516.645\pm4037.346$&$64535.982\pm4035.333$&$1848.825\pm18.308$\\
\bottomrule
\end{tabular}}
\end{table*}

\begin{table*}[htb]
\centering
\vspace{-5pt}
\caption{Empirical generalization error bounds via functional derivative for width $h=128$ of the hidden layer for different model types and readout functions, for various supervised ratios of $\beta_\text{sup}$. We report the mean plus/minus one standard deviation over ten runs.}
\label{tab:bounds_full_h128}
\resizebox{\linewidth}{!}{\begin{tabular}{lr|rrr|rrrrrrrrrrrrrr}
\toprule
\multirow{ 2}{*}{Data Set}&\multirow{ 2}{*}{$\beta_\text{sup}$}&\multicolumn{3}{c|}{mean readout}&\multicolumn{3}{c}{sum readout}\\
&&GCN&GCN\_RW&MPGNN&GCN&GCN\_RW&MPGNN\\
\midrule
SBM-1&$0.7$&$0.014\pm0.001$&$0.015\pm0.001$&$0.244\pm0.013$&$2520476.175\pm156887.505$&$2669876.975\pm164650.807$&$4005.391\pm220.699$\\
SBM-1&$0.9$&$0.011\pm0.001$&$0.012\pm0.001$&$0.190\pm0.010$&$1885484.212\pm142336.573$&$1996144.750\pm150151.799$&$3118.617\pm164.941$\\
SBM-2&$0.7$&$0.023\pm0.001$&$0.025\pm0.001$&$0.341\pm0.019$&$4089494.725\pm260354.119$&$4438159.850\pm275790.885$&$5603.943\pm323.262$\\
SBM-2&$0.9$&$0.018\pm0.001$&$0.020\pm0.001$&$0.266\pm0.015$&$3049281.675\pm228393.096$&$3310971.200\pm245013.883$&$4364.249\pm244.524$\\
SBM-3&$0.7$&$0.015\pm0.001$&$0.016\pm0.001$&$0.249\pm0.013$&$624203.919\pm38222.136$&$657592.244\pm40341.332$&$4091.722\pm236.617$\\
SBM-3&$0.9$&$0.011\pm0.001$&$0.012\pm0.001$&$0.193\pm0.010$&$476286.562\pm25613.135$&$501607.622\pm26855.429$&$3182.215\pm179.787$\\
ER-4&$0.7$&$0.005\pm0.000$&$0.005\pm0.000$&$0.122\pm0.005$&$811877.012\pm49965.822$&$815294.519\pm50332.475$&$1993.695\pm90.394$\\
ER-4&$0.9$&$0.004\pm0.000$&$0.004\pm0.000$&$0.095\pm0.004$&$607448.919\pm45159.041$&$609996.025\pm45468.797$&$1552.299\pm68.335$\\
ER-5&$0.7$&$0.008\pm0.000$&$0.009\pm0.001$&$0.168\pm0.008$&$53656.979\pm2753.653$&$58993.716\pm3014.306$&$2752.105\pm142.909$\\
ER-5&$0.9$&$0.006\pm0.000$&$0.007\pm0.000$&$0.131\pm0.006$&$41149.381\pm2515.410$&$45243.126\pm2763.916$&$2137.883\pm111.558$\\
PROTEINS&$0.7$&$0.005\pm0.000$&$0.005\pm0.000$&$0.045\pm0.003$&$363676.153\pm93112.546$&$363892.670\pm93165.563$&$2122.023\pm27.214$\\
PROTEINS&$0.9$&$0.004\pm0.000$&$0.004\pm0.000$&$0.035\pm0.002$&$402014.978\pm50920.169$&$402314.366\pm50931.723$&$1657.479\pm21.872$\\
\bottomrule
\end{tabular}}
\end{table*}

\begin{table*}[htb]
\centering
\vspace{-5pt}
\caption{Empirical generalization error bounds via Rademacher Complexities for width $h=128$ of the hidden layer for different model types and readout functions, for various supervised ratios of $\beta_\text{sup}$. We report the mean plus/minus one standard deviation over ten runs.}
\label{tab:bounds_full_Rademacher_h128}
\resizebox{\linewidth}{!}{\begin{tabular}{lr|rrr|rrrrrrrrrrrrrr}
\toprule
\multirow{ 2}{*}{Data Set}&\multirow{ 2}{*}{$\beta_\text{sup}$}&\multicolumn{3}{c|}{mean readout}&\multicolumn{3}{c}{sum readout}\\
&&GCN&GCN\_RW&MPGNN&GCN&GCN\_RW&MPGNN\\
\midrule
SBM-1&$0.7$&$0.465\pm0.002$&$0.465\pm0.002$&$3.774\pm0.063$&$304382.491\pm9500.618$&$304607.622\pm9518.604$&$4814.476\pm84.398$\\
SBM-1&$0.9$&$0.410\pm0.002$&$0.410\pm0.002$&$3.329\pm0.055$&$260826.333\pm7491.220$&$260911.930\pm7550.637$&$4245.814\pm74.430$\\
SBM-2&$0.7$&$0.465\pm0.002$&$0.465\pm0.002$&$3.774\pm0.063$&$305824.653\pm9984.800$&$305862.869\pm10206.895$&$4814.185\pm84.265$\\
SBM-2&$0.9$&$0.410\pm0.002$&$0.410\pm0.002$&$3.329\pm0.055$&$261423.833\pm7225.399$&$261561.489\pm7270.949$&$4245.654\pm74.354$\\
SBM-3&$0.7$&$0.465\pm0.002$&$0.465\pm0.002$&$3.774\pm0.063$&$104673.709\pm2587.052$&$104675.531\pm2549.178$&$4816.482\pm84.778$\\
SBM-3&$0.9$&$0.410\pm0.002$&$0.410\pm0.002$&$3.329\pm0.055$&$91149.288\pm1448.412$&$91135.775\pm1449.661$&$4247.080\pm74.968$\\
ER-4&$0.7$&$0.465\pm0.002$&$0.465\pm0.002$&$3.774\pm0.063$&$305177.969\pm10269.198$&$305189.119\pm10276.768$&$4814.444\pm84.389$\\
ER-4&$0.9$&$0.410\pm0.002$&$0.410\pm0.002$&$3.329\pm0.055$&$261501.731\pm7297.492$&$261507.970\pm7300.305$&$4245.799\pm74.568$\\
ER-5&$0.7$&$0.465\pm0.002$&$0.465\pm0.002$&$3.600\pm0.137$&$26211.781\pm520.483$&$26209.152\pm522.064$&$4557.881\pm205.714$\\
ER-5&$0.9$&$0.410\pm0.002$&$0.410\pm0.002$&$3.175\pm0.120$&$22835.085\pm242.080$&$22832.954\pm242.523$&$4014.455\pm180.352$\\
PROTEINS&$0.7$&$0.197\pm0.001$&$0.197\pm0.001$&$0.844\pm0.030$&$59715.427\pm7057.934$&$59730.912\pm7058.069$&$2089.489\pm20.112$\\
PROTEINS&$0.9$&$0.174\pm0.001$&$0.174\pm0.001$&$0.747\pm0.026$&$64516.645\pm4037.346$&$64535.982\pm4035.333$&$1848.825\pm18.308$\\
\bottomrule
\end{tabular}}
\end{table*}

\begin{table*}[htb]
\centering
\vspace{-5pt}
\caption{Empirical generalization error bounds via functional derivative for width $h=64$ of the hidden layer for different model types and readout functions, for various supervised ratios of $\beta_\text{sup}$. We report the mean plus/minus one standard deviation over ten runs.}
\label{tab:bounds_full_h64}
\resizebox{\linewidth}{!}{\begin{tabular}{lr|rrr|rrrrrrrrrrrrrr}
\toprule
\multirow{ 2}{*}{Data Set}&\multirow{ 2}{*}{$\beta_\text{sup}$}&\multicolumn{3}{c|}{mean readout}&\multicolumn{3}{c}{sum readout}\\
&&GCN&GCN\_RW&MPGNN&GCN&GCN\_RW&MPGNN\\
\midrule
SBM-1&$0.7$&$0.014\pm0.001$&$0.015\pm0.001$&$0.244\pm0.013$&$2520476.175\pm156887.505$&$2669876.975\pm164650.807$&$4005.391\pm220.699$\\
SBM-1&$0.9$&$0.011\pm0.001$&$0.012\pm0.001$&$0.190\pm0.010$&$1885484.212\pm142336.573$&$1996144.750\pm150151.799$&$3118.617\pm164.941$\\
SBM-2&$0.7$&$0.023\pm0.001$&$0.025\pm0.001$&$0.341\pm0.019$&$4089494.725\pm260354.119$&$4438159.850\pm275790.885$&$5603.943\pm323.262$\\
SBM-2&$0.9$&$0.018\pm0.001$&$0.020\pm0.001$&$0.266\pm0.015$&$3049281.675\pm228393.096$&$3310971.200\pm245013.883$&$4364.249\pm244.524$\\
SBM-3&$0.7$&$0.015\pm0.001$&$0.016\pm0.001$&$0.249\pm0.013$&$624203.919\pm38222.136$&$657592.244\pm40341.332$&$4091.722\pm236.617$\\
SBM-3&$0.9$&$0.011\pm0.001$&$0.012\pm0.001$&$0.193\pm0.010$&$476286.562\pm25613.135$&$501607.622\pm26855.429$&$3182.215\pm179.787$\\
ER-4&$0.7$&$0.005\pm0.000$&$0.005\pm0.000$&$0.122\pm0.005$&$811877.012\pm49965.822$&$815294.519\pm50332.475$&$1993.695\pm90.394$\\
ER-4&$0.9$&$0.004\pm0.000$&$0.004\pm0.000$&$0.095\pm0.004$&$607448.919\pm45159.041$&$609996.025\pm45468.797$&$1552.299\pm68.335$\\
ER-5&$0.7$&$0.008\pm0.000$&$0.009\pm0.001$&$0.168\pm0.008$&$53656.979\pm2753.653$&$58993.716\pm3014.306$&$2752.105\pm142.909$\\
ER-5&$0.9$&$0.006\pm0.000$&$0.007\pm0.000$&$0.131\pm0.006$&$41149.381\pm2515.410$&$45243.126\pm2763.916$&$2137.883\pm111.558$\\
PROTEINS&$0.7$&$0.005\pm0.000$&$0.005\pm0.000$&$0.045\pm0.003$&$363676.153\pm93112.546$&$363892.670\pm93165.563$&$2122.023\pm27.214$\\
PROTEINS&$0.9$&$0.004\pm0.000$&$0.004\pm0.000$&$0.035\pm0.002$&$402014.978\pm50920.169$&$402314.366\pm50931.723$&$1657.479\pm21.872$\\
\bottomrule
\end{tabular}}
\end{table*}

\begin{table*}[htb]
\centering
\vspace{-5pt}
\caption{Empirical generalization error bounds via Rademacher Complexities for width $h=64$ of the hidden layer for different model types and readout functions, for various supervised ratios of $\beta_\text{sup}$. We report the mean plus/minus one standard deviation over ten runs.}
\label{tab:bounds_full_Rademacher_h64}
\resizebox{\linewidth}{!}{\begin{tabular}{lr|rrr|rrrrrrrrrrrrrr}
\toprule
\multirow{ 2}{*}{Data Set}&\multirow{ 2}{*}{$\beta_\text{sup}$}&\multicolumn{3}{c|}{mean readout}&\multicolumn{3}{c}{sum readout}\\
&&GCN&GCN\_RW&MPGNN&GCN&GCN\_RW&MPGNN\\
\midrule
SBM-1&$0.7$&$0.465\pm0.002$&$0.465\pm0.002$&$3.774\pm0.063$&$304382.491\pm9500.618$&$304607.622\pm9518.604$&$4814.476\pm84.398$\\
SBM-1&$0.9$&$0.410\pm0.002$&$0.410\pm0.002$&$3.329\pm0.055$&$260826.333\pm7491.220$&$260911.930\pm7550.637$&$4245.814\pm74.430$\\
SBM-2&$0.7$&$0.465\pm0.002$&$0.465\pm0.002$&$3.774\pm0.063$&$305824.653\pm9984.800$&$305862.869\pm10206.895$&$4814.185\pm84.265$\\
SBM-2&$0.9$&$0.410\pm0.002$&$0.410\pm0.002$&$3.329\pm0.055$&$261423.833\pm7225.399$&$261561.489\pm7270.949$&$4245.654\pm74.354$\\
SBM-3&$0.7$&$0.465\pm0.002$&$0.465\pm0.002$&$3.774\pm0.063$&$104673.709\pm2587.052$&$104675.531\pm2549.178$&$4816.482\pm84.778$\\
SBM-3&$0.9$&$0.410\pm0.002$&$0.410\pm0.002$&$3.329\pm0.055$&$91149.288\pm1448.412$&$91135.775\pm1449.661$&$4247.080\pm74.968$\\
ER-4&$0.7$&$0.465\pm0.002$&$0.465\pm0.002$&$3.774\pm0.063$&$305177.969\pm10269.198$&$305189.119\pm10276.768$&$4814.444\pm84.389$\\
ER-4&$0.9$&$0.410\pm0.002$&$0.410\pm0.002$&$3.329\pm0.055$&$261501.731\pm7297.492$&$261507.970\pm7300.305$&$4245.799\pm74.568$\\
ER-5&$0.7$&$0.465\pm0.002$&$0.465\pm0.002$&$3.600\pm0.137$&$26211.781\pm520.483$&$26209.152\pm522.064$&$4557.881\pm205.714$\\
ER-5&$0.9$&$0.410\pm0.002$&$0.410\pm0.002$&$3.175\pm0.120$&$22835.085\pm242.080$&$22832.954\pm242.523$&$4014.455\pm180.352$\\
PROTEINS&$0.7$&$0.197\pm0.001$&$0.197\pm0.001$&$0.844\pm0.030$&$59715.427\pm7057.934$&$59730.912\pm7058.069$&$2089.489\pm20.112$\\
PROTEINS&$0.9$&$0.174\pm0.001$&$0.174\pm0.001$&$0.747\pm0.026$&$64516.645\pm4037.346$&$64535.982\pm4035.333$&$1848.825\pm18.308$\\
\bottomrule
\end{tabular}}
\end{table*}